%% file: mlps_colm_short.tex
\documentclass{article}
\input{shared/preamble}
\input{shared/metadata}

\definecolor{compressorange}{RGB}{176, 94, 16}
\definecolor{appendixblue}{RGB}{0, 91, 150}

\begin{document}
\ifcolmsubmission
\linenumbers
\fi

\maketitle
\hebbianposttitlefootnotes

\begin{abstract}
\input{sections_short/abstract}

\end{abstract}

\section{Introduction}
\label{sec:intro}
\input{sections_short/section_1_intro/main}

\section{Preliminaries}
\label{sec:definitions}
\input{sections_short/section_2_preliminaries/main}

\section{MLPs, Hebbians, and Margins}
\label{sec:basic_setting}
\label{sec:3}
\input{sections_short/section_3_basic_setting/main}

\section{Margin and Storage Capacity Analysis of Hebbian MLPs}
\label{sec:hebbian-kernel-mlps}
\label{sec:4}
\input{sections_short/section_4_hebbian_kernel_mlps/main}

\section{Integrating Fact-Storing MLPs into Transformers}
\label{sec:5}
\label{sec:usability}
\label{sec:transformers}
\input{sections_short/section_5_transformer_integration/main}

\input{sections_short/section_6_conclusion/main}

\newpage
\input{sections_short/impact_statement}
\input{shared/acknowledgements}

\bibliography{colm2026_conference}
\bibliographystyle{colm2026_conference}

\newpage

\appendix
\onecolumn
\input{appendix_short/main}

\end{document}

%% file: shared/preamble.tex
\usepackage[preprint]{colm2026_conference}

\input{shared/math_commands.tex}

\usepackage{microtype}
\usepackage{graphicx}
\usepackage{subcaption}
\usepackage{booktabs}
\usepackage{hyperref}
\usepackage{url}
\usepackage{amsmath}
\usepackage{amssymb}
\usepackage{mathtools}
\usepackage{amsthm}
\usepackage{mdframed}
\usepackage{caption}
\usepackage{wrapfig}
\usepackage{xcolor}
\usepackage{float}
\usepackage{makecell}
\usepackage{enumitem}
\usepackage{bbm}
\usepackage{natbib}
\usepackage{algorithmic}
\usepackage{algorithm}
\usepackage{aliascnt}
\usepackage[capitalize,noabbrev]{cleveref}
\usepackage{lineno}
\usepackage[breakable]{tcolorbox}
\usepackage{titletoc}

\definecolor{darkblue}{rgb}{0, 0, 0.5}
\hypersetup{colorlinks=true, citecolor=darkblue, linkcolor=darkblue, urlcolor=darkblue}

\definecolor{keyblue}{RGB}{0, 91, 150}       %
\definecolor{valred}{RGB}{180, 60, 30}        %
\definecolor{couplepurple}{RGB}{120, 40, 140} %

\theoremstyle{plain}
\newtheorem{theorem}{Theorem}[section]

\newaliascnt{proposition}{theorem}
\newtheorem{proposition}[proposition]{Proposition}
\aliascntresetthe{proposition}

\newaliascnt{lemma}{theorem}
\newtheorem{lemma}[lemma]{Lemma}
\aliascntresetthe{lemma}

\newaliascnt{corollary}{theorem}
\newtheorem{corollary}[corollary]{Corollary}
\aliascntresetthe{corollary}

\theoremstyle{definition}
\newaliascnt{definition}{theorem}
\newtheorem{definition}[definition]{Definition}
\aliascntresetthe{definition}

\newaliascnt{assumption}{theorem}

\aliascntresetthe{assumption}

\theoremstyle{remark}
\newaliascnt{remark}{theorem}
\newtheorem{remark}[remark]{Remark}
\aliascntresetthe{remark}

\crefname{proposition}{Proposition}{Propositions}
\crefname{lemma}{Lemma}{Lemmas}
\crefname{corollary}{Corollary}{Corollaries}
\crefname{construction}{Construction}{Constructions}
\crefname{definition}{Definition}{Definitions}
\crefname{assumption}{Assumption}{Assumptions}
\crefname{remark}{Remark}{Remarks}

\newcommand{\Pbb}{\mathbb{P}}
\newcommand{\defeq}{\vcentcolon=}

\newcommand{\matK}{\mathbf{K}}

\newcommand{\cN}{\mathcal{N}}

\newtheorem*{proposition*}{Proposition}
\newtheorem*{corollary*}{Corollary}

\newtheorem*{lemma*}{Lemma}

%% file: shared/math_commands.tex
\usepackage{amsmath,amsfonts,bm}

\def\eqref#1{equation~\ref{#1}}

\def\1{\bm{1}}

\def\vone{{\mathbf{1}}}

\def\va{{\mathbf{a}}}
\def\vb{{\mathbf{b}}}
\def\vc{{\mathbf{c}}}

\def\vk{{\mathbf{k}}}

\def\vq{{\mathbf{q}}}

\def\vu{{\mathbf{u}}}
\def\vv{{\mathbf{v}}}
\def\vw{{\mathbf{w}}}

\def\vy{{\mathbf{y}}}

\def\mPhi{{\bm{\Phi}}}

\def\mSigma{{\bm{\Sigma}}}

\DeclareMathAlphabet{\mathsfit}{\encodingdefault}{\sfdefault}{m}{sl}
\SetMathAlphabet{\mathsfit}{bold}{\encodingdefault}{\sfdefault}{bx}{n}

\newcommand{\E}{\mathbb{E}}

\newcommand{\R}{\mathbb{R}}

\newcommand{\Var}{\mathrm{Var}}

\newcommand{\norm}[1]{\|#1 \|}

\newcommand{\col}{\begin{pmatrix}}
\newcommand{\con}{\end{pmatrix}}

\newcommand{\poly}{\operatorname{poly}}

\newcommand{\ip}[2]{\langle {#1}, {#2} \rangle}

\newcommand{\Sph}{\mathbb{S}^{d-1}}

%% file: shared/metadata.tex
\title{MLPs are Hebbians: Constructing Efficient Fact-Storing MLPs for Transformers}

\makeatletter
\renewcommand{\thefootnote}{\fnsymbol{footnote}}
\newcommand{\thanksmark}[1]{\textsuperscript{\@fnsymbol{#1}}}
\newcommand{\corrauthormark}{\kern0.35em\textsuperscript{\@fnsymbol{2}}}
\makeatother

\author{
\textbf{Roberto Garcia}\textsuperscript{1}
  \thanks{Equal first authors; order chosen by coin flip}
  \corrauthormark
\quad
\textbf{Jerry Liu}\textsuperscript{1}
  \thanksmark{1}
  \hspace{-0.2cm}
  \corrauthormark
\quad
\textbf{Ronny Junkins}\textsuperscript{2}
  \thanksmark{1}
\\[0.3ex]
\textbf{Sabri Eyuboglu}\textsuperscript{2}
\quad
\textbf{Atri Rudra}\textsuperscript{3}
\quad
\textbf{Chris R{\'e}}\textsuperscript{2}
\\[1.5ex]
\textsuperscript{1}Institute for Computational \& Mathematical Engineering, Stanford University \\
\textsuperscript{2}Department of Computer Science, Stanford University \\
\textsuperscript{3}Department of Computer Science and Engineering, University at Buffalo
}

\newcommand{\hebbianposttitlefootnotes}{%
  \begingroup
  \renewcommand{\thefootnote}{\fnsymbol{footnote}}%
  \footnotetext[2]{Corresponding authors: \href{mailto:robgarct@stanford.edu}{\texttt{robgarct@stanford.edu}}, \href{mailto:jerrywliu@stanford.edu}{\texttt{jerrywliu@stanford.edu}}.}%
  \footnotetext[3]{Code is released at \url{https://github.com/HazyResearch/hebbian-mlps}.}%
  \endgroup
  \setcounter{footnote}{0}%
  \renewcommand{\thefootnote}{\arabic{footnote}}%
}

%% file: sections_short/abstract.tex
Large language models (LLMs) store factual knowledge in their parameters. While recent work has shown that this knowledge resides in MLP layers, existing constructive and mechanistic interpretability models of fact-storage in LLMs fail to explain the surprising empirical phenomenon that they store facts at an information-theoretically optimal rate. In this work, we develop a theoretical account of this phenomenon. We develop the first Transformer-compatible fact-storing MLP closed-form construction that satisfies the following three properties empirically observed in LLMs: it (i) attains optimal fact storage scaling, (ii) handles arbitrary input/output geometries, and (iii) works inside Transformers. Key to our work is to analyze the \emph{decoding margin} of MLPs, whereas prior work only studies MLP fact storage. Under isotropic embeddings, our construction achieves information-theoretically optimal storage capacity scaling and requires $10$-$104\times$ fewer parameters at matched fact count than prior constructions. For arbitrary key and value embeddings, we show that our construction attains the same storage capacity scaling, up to penalization factors depending on the embedding geometries. Moreover, we demonstrate that our constructed MLPs can be used within Transformer blocks for factual recall tasks at optimal capacity scaling, requiring $15$-$63\times$ fewer parameters at matched fact count than prior constructions. Finally, as a proof-of-concept, we show that fact-storing MLPs enable \emph{modular fact editing} by swapping a Transformer's MLP with a new one.

%% file: sections_short/section_1_intro/main.tex
Large language models (LLMs) achieve remarkable performance across domains such as mathematics, science, and law~\citep{AlphaProof2024_DMind,guha2023legalbench,saab2024capabilities}, in part because they can store vast amounts of knowledge in their parameters~\citep{petroni2019language,meng2023locatingeditingfactualassociations}.
Prior work suggests that knowledge in Transformers is stored in Multi-Layer Perceptrons (MLPs) as key-value mappings, or \emph{facts}~\citep{geva2021transformerfeedforwardlayerskeyvalue,dai2022knowledgeneuronspretrainedtransformers}.
However, despite these findings, fact-storing MLPs remain poorly understood.

While prior work has made important progress toward understanding and modeling fact storage in MLPs, existing models fail to capture three empirically observed properties of LLM fact storage: MLPs must (i) attain optimal fact storage scaling, (ii) handle arbitrary input/output geometries, and (iii) work inside Transformers.
Mechanistic interpretability work~\citep{geva2021transformerfeedforwardlayerskeyvalue,dai2022knowledgeneuronspretrainedtransformers} assumes MLPs store facts in individual neurons, but these models lead to suboptimal storage-capacity scaling.
More recently, \citet{nichani2024understandingfactualrecalltransformers} study LLM fact storage by introducing an MLP \emph{weight construction} (NTK MLP) with theoretical capacity guarantees.
However, existing constructions
(i) theoretically and empirically do not attain the empirically observed information-theoretically optimal capacity scaling of LLMs;
(ii) are restricted to isotropic (e.g., uniformly spherical) embedding distributions, whereas LLM embeddings are anisotropic \citep{ethayarajh2019contextualcontextualizedwordrepresentations,razzhigaev2024shapelearninganisotropyintrinsic}; and 
(iii) cannot be used by Transformer blocks for factual recall tasks, such as answering ``What is the capital of France?''.

Our core insight is to study MLP decoding margin scaling (\Cref{def:margin}). Where prior works focus on fact-storage under noiseless key queries, our decoding margin study allows us to develop the first closed-form MLP construction that is usable by Transformers for factual recall. Consequently, we theoretically show that \emph{Transformer blocks are capable of information-theoretically optimal fact storage} -- providing the first theoretical explanation that aligns with the empirical capacity scaling observed in pretrained LLMs.

Our construction is the first to capture all three empirically observed properties of LLM fact storage:
\input{sections/section_1_intro/figs/banner}

\begin{itemize}[leftmargin=1em, itemsep=0.5em]
\item \textbf{Optimal fact-storage capacity and margin (Sections \ref{sec:basic_setting}, \ref{sec:hebbian-kernel-mlps}).}
Our first step toward matching the empirically optimal fact-storage scaling of LLMs is to demonstrate that 1) MLPs are Hebbian kernel memories and 2) that Hebbian kernel memories achieve asymptotically optimal fact-storage scaling. We show this by developing a closed-form MLP construction, equivalent to a Hebbian memory with sketched quadratic kernel, which provably attains a fact storage capacity of $F = \Theta(W / \log W)$ for $F$ facts using $W$ parameters (\cref{cor:optimal-fact-storage-capacity}).
Theoretically, our construction closes the optimality gap over prior constructions by a factor of $\log^{11} F$ under isotropic embeddings.
Moreover, we show that, at matched fact count, the NTK baseline requires $10$-$104\times$ more parameters than our best data-dependent kernel construction (\cref{fig:sketched_k2_sweeps}c).

\item \textbf{Handling arbitrary embedding geometries (\Cref{sec:hebbian-kernel-mlps}).}
We next study how arbitrary embedding geometries affect MLP margin and storage capacity scaling. We show that generalizing the MLP margin and capacity scaling to arbitrary embedding geometries introduces four embedding-geometric statistics multiplicatively into the information-theoretically optimal scaling derived for isotropic embeddings (\cref{thm:bilinear-scaling}). Intuitively, these statistics penalize the margin and capacity scaling by how clustered the key and value embeddings are. Empirically, we show that our generalized bounds characterize the empirical decoding margin scaling precisely ($R^2 \geq 0.95$; \cref{fig:app:beta-sweeps}c). Furthermore, we find that the capacity gap between our construction and trained MLPs is preserved even under anisotropic embeddings.

\item \textbf{MLPs usable within Transformers for factual recall (\Cref{sec:usability}).}
Towards understanding MLP usage within LLMs, we find that non-trivial decoding margin is needed by MLPs to be used for factual recall within Transformer blocks. Attention layers produce imperfect, noisy queries, so the fact-storing MLPs they query should be robust to noise. Building on this insight, we demonstrate theoretically and empirically, for the first time, that Transformer blocks can retrieve facts from MLPs with optimal fact-storage capacity to solve factual recall tasks (\cref{thm:noise-robust-capacity}, \cref{fig:ssfr_main_fig}c). Moreover, in these Transformer experiments, at matched fact count, the NTK baseline requires roughly $15$-$63\times$ more parameters than Transformer blocks using our data-dependent construction.
\end{itemize}

Finally, as a proof-of-concept, we show that fact-storing MLPs enable \emph{modular fact editing} (\Cref{sec:fact-editing}) by replacing a Transformer's MLP with one storing new facts.
Our method, \textsc{MLP Swapping}, achieves near-perfect \emph{fact-editing score}---correctly editing target facts while avoiding off-target effects---whereas prior state-of-the-art methods degrade to as low as $\sim30$\% score when editing 10\% of the fact-set.

In summary, our work takes a constructive step toward understanding MLPs in Transformers. We present a fact-storing MLP construction that achieves optimal margin and fact-storage capacity, provides provable decoding-margin guarantees under arbitrary embeddings, and is usable within Transformer blocks for factual recall at optimal capacity. We also demonstrate an application to modular fact editing, illustrating a path toward robust and modular knowledge manipulation in LLMs.

%% file: sections/section_1_intro/figs/banner.tex
\begin{figure}[t]
    \centering
    \includegraphics[
        width=\linewidth,
    ]{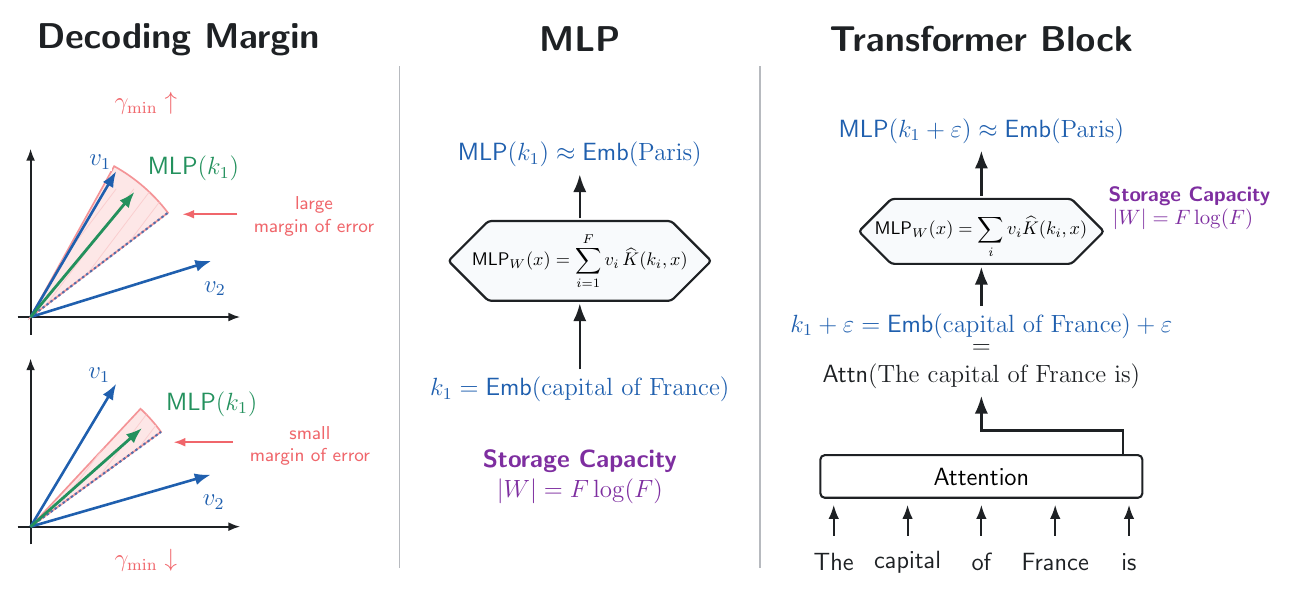}
    \caption{
    \textbf{(A) Decoding margin illustration:} Top plot shows an MLP with large decoding margin, bottom plot with low decoding margin. The MLP with larger decoding margin has a larger ``margin of error'', making it more robust when queried with perturbed versions of $\mathbf{k}_1$. \textbf{(B) MLPs are information-theoretically optimal:} we show that MLPs are Hebbian memories in kernel space and that the storage capacity of MLPs scales at the information-theoretically optimal rate. \textbf{(C) Transformer blocks are information-theoretically optimal:} we show that the storage capacity of Transformer blocks scales at the information-theoretically optimal rate, provided the attention noise remains bounded.
    }
    \label{fig:main-figure}
\end{figure}

%% file: sections_short/section_2_preliminaries/main.tex
\subsection{Formalizing Factual Knowledge}
\label{subsec:definitions_maintext}

\paragraph{Fact sets and storage.}
Given key embeddings $\mathbf{K}\in \mathbb{R}^{|K| \times d}$ and value embeddings $\mathbf{V}\in\mathbb{R}^{|V| \times d}$, a \textit{fact set} is a map $f: [|\mathbf{K}|]\to [|\mathbf{V}|]$. We write $\mathbf{k}_{i}$ and $\mathbf{v}_{i}$ for the $i$th key and value embedding, respectively.

\begin{definition}[Fact storage]
\label{def:fact-storage}
A model $\mathbf{g}_\theta: \mathbb{R}^d\to \mathbb{R}^d$ \textit{stores a fact set} $f: [|\mathbf{K}|]\to [|\mathbf{V}|]$ given embeddings $\mathbf{K}$ and $\mathbf{V}$ if, for all $i \in [|\mathbf{K}|]$ and all $j \neq f(i) \in [|\mathbf{V}|]$,
\begin{equation}
    \label{eq:decoding-criterion-maintext}
    \langle \mathbf{g}_\theta(\mathbf{k}_{i}), \mathbf{v}_{f(i)} \rangle > \langle \mathbf{g}_\theta(\mathbf{k}_{i}), \mathbf{v}_{j} \rangle.
\end{equation}
\end{definition}

Notably, this definition is equivalent to correct softmax decoding in language modeling.

\begin{definition}[Margin]
\label{def:margin}
The margin of $\mathbf{g}_\theta$ on fact $i$ against competitor $j \neq f(i)$ is
\begin{equation}
    \label{eq:margin-maintext}
    \gamma_{i,j} := \langle \mathbf{g}_\theta(\mathbf{k}_{i}), \mathbf{v}_{f(i)} \rangle - \langle \mathbf{g}_\theta(\mathbf{k}_{i}), \mathbf{v}_{j} \rangle,
\end{equation}
and the minimum margin is $\gamma_{\min} := \min_{i, j \neq f(i)} \gamma_{i,j}$.
Note that storing a fact set (in the sense of~\Cref{def:fact-storage}) is equivalent to $\gamma_{\min} > 0$.
\end{definition}

\paragraph{Fact-storage cost and capacity.}
To measure parameter efficiency, we define the smallest parameter budget needed for a model class to store \emph{every} fact set on fixed embeddings.

\begin{definition}[Fact-storage cost and capacity]
\label{def:fact-storage-cost}
The \emph{fact-storage cost} of a model class $\mathbf{g}$ on embeddings $\mathbf{K}$ and $\mathbf{V}$ is the minimum parameter count needed to represent \emph{all} possible fact sets:
\begin{equation}
    W(\mathbf{g}; \mathbf{K}, \mathbf{V}) =
    \min \left\{\#(\theta)\Bigg|\;
    \begin{aligned}
        &\forall f : [|\mathbf{K}|] \to [|\mathbf{V}|], \\
        &\exists\, \theta \; \text{s.t.} \; \mathbf{g}_\theta \text{ stores } f
    \end{aligned}
    \right\}.
    \label{def:complexity}
\end{equation}
The corresponding \emph{fact-storage capacity} is the maximum number of facts storable with a fixed parameter budget.
\end{definition}

\begin{theorem}[Information-theoretic lower bound]
\label{thm: info_bounds_capacity-const}
Assuming a constant number of bits per parameter, the fact-storage cost of embeddings $\mathbf{K}$ and $\mathbf{V}$ for \emph{any} model class $\mathbf{g}$ satisfies
\[
W(\mathbf{g}; \mathbf{K}, \mathbf{V}) = \Omega(|\mathbf{K}|\log [|\mathbf{V}|]).
\]
\end{theorem}
See Appendix~\ref{sec:info_theory_bound} for proof.

\subsection{Model Classes}
\label{subsec:model_classes}
In this work, we study two model classes: gated one-hidden-layer MLPs and Hebbian memories~\citep{10.1109/TC.1972.5008975, doi:10.1073/pnas.79.8.2554, bubeck2020networksizeweightssize, cabannes2024scalinglawsassociativememories, nichani2024understandingfactualrecalltransformers}.

\textbf{MLPs.} We consider models $\mathbf{g}_\theta:\mathbb{R}^d\to\mathbb{R}^{d_v}$ of the form
\begin{equation}
\label{eq:gated-mlp-class}
\text{MLP}(\mathbf{x}) =
\mathbf{g}_\theta(\mathbf{x})=\mathbf{B}\!\left((\mathbf{A}\mathbf{x})\odot \sigma(\mathbf{G}\mathbf{x})\right),
\end{equation}
where $\mathbf{x}\in\mathbb{R}^d$, $\mathbf{A},\mathbf{G}\in\mathbb{R}^{m\times d}$, and $\mathbf{B}\in\mathbb{R}^{d_{v}\times m}$.
This family includes SwiGLU-style MLPs~\citep{shazeer2020gluvariantsimprovetransformer} used in modern language models~\citep{yang2025qwen3technicalreport, deepseekai2025deepseekv3technicalreport, dubey2024llama}. Our explicit construction (\Cref{sec:construction}) uses $\sigma=\mathrm{id}$.

\textbf{Hebbian memories.} These linear models are maps $\mathbf{g}_{\mathbf{W}}:\mathbb{R}^d\to\mathbb{R}^{d_v}$ of the form $\mathbf{g}_{\mathbf{W}}(\mathbf{x})=\mathbf{W}\mathbf{x}$, where $\mathbf{x}\in\mathbb{R}^d$ and $\mathbf{W}\in\mathbb{R}^{d_{v}\times d}$. A fact set is stored by taking
\begin{equation}
\label{eq:hebbian-memory-class}
\mathbf{W}=\sum_{j=1}^{|\mathbf{K}|}\mathbf{v}_{f(j)}\mathbf{k}_{j}^\top.
\end{equation}
Section~\ref{sec:3} shows that MLPs can be recast as Hebbian memories in a kernel feature space.

\subsection{Related Work}
\label{sec:related_work}

The two closest works to ours are \citet{nichani2024understandingfactualrecalltransformers} and \citet{zhong2025understandingtransformerperspectiveassociative}. \citet{nichani2024understandingfactualrecalltransformers} gave the first explicit construction of fact-storing MLPs and showed near-optimal fact-storage capacity, up to a polylogarithmic factor, but their analysis is restricted to isotropic embeddings and studies only the separability condition $\gamma_{\min}>0$. \citet{zhong2025understandingtransformerperspectiveassociative} developed a unified associative-memory view of attention and MLPs
, but focused on average retrieval fidelity rather than worst-case margins. In contrast, we derive explicit margin bounds beyond isotropy, identify margin bounded away from zero as the condition for Transformer usability, and show that constructed MLPs can be integrated into Transformer blocks for factual recall.

Additional discussion of probing, editing, and scaling studies of factual knowledge in language models appears in Appendix~\ref{app:related_work}.

%% file: sections_short/section_3_basic_setting/main.tex
We first establish two observations that let us analyze MLP fact storage within Transformers. First, MLPs are equivalent to \emph{Hebbian kernel memories} after whitening the empirical feature covariance. Second, MLPs in Transformers need \emph{margins bounded away from zero} (not just positive separability) because attention layers pass noisy queries to MLPs.

\subsection{MLPs Are Hebbian Kernel Memories}
\label{sec:mlps-are-hebbians}

Our first observation is that any MLP can be rewritten as a \emph{Hebbian kernel memory} on its stored examples. The only gap between the plain Hebbian predictor and the original MLP is the empirical feature covariance $\hat{\mSigma}$, so whitening converts any MLP into a (kernel) Hebbian.

\begin{tcolorbox}[breakable,colback=white,colframe=blue!50!black,boxrule=1.5pt,arc=0mm,left=8pt,right=8pt,top=8pt,bottom=8pt,title={\small\bfseries Key Result: MLPs are Hebbian Memories},colbacktitle=blue!50!black,coltitle=white,titlerule=0.5pt]
\begin{theorem}[MLPs as kernel Hebbians, informal]\label{thm:mlp-hebb-whiten}
For stored examples $(\mathbf{x}_{i},\mathbf{y}_{i})$ with
$\mathbf{y}_{i}=\mathrm{MLP}(\mathbf{x}_{i})$ and
$\mathrm{MLP}(\mathbf{x})=\mathbf{B}\phi(\mathbf{x})$
(where for gated MLPs
$\phi(\mathbf{x})=(\mathbf{A}\mathbf{x})\odot \sigma(\mathbf{G}\mathbf{x})$),
define the empirical feature covariance
\[
\hat{\mSigma}:=\frac1F\sum_{i=1}^F
\phi(\mathbf{x}_{i})\phi(\mathbf{x}_{i})^\top .
\]
Assuming $\hat{\mSigma}$ is invertible, define the whitened Hebbian memory
\[
H_{\mathrm{white}}(\mathbf{z})
:=
\frac1F\sum_{i=1}^F
\mathbf{y}_{i}\,K(\mathbf{x}_{i},\mathbf{z})
\]
induced by the kernel
\[
K(\mathbf{x},\mathbf{z})
:=
\phi(\mathbf{x})^\top\hat{\mSigma}^{-1}\phi(\mathbf{z}).
\]
Then
\[
H_{\mathrm{white}}(\mathbf{z})=\mathrm{MLP}(\mathbf{z})
\qquad \text{for all } \mathbf{z}.
\]
Thus, after feature whitening, the MLP is exactly a Hebbian kernel memory.
\end{theorem}
See Appendix~\ref{app:theory:hebbian_mlps:setup} for formal statement and proof. 
\end{tcolorbox}

This reduction motivates the rest of the paper:
constructing an MLP with desirable margin and storage capacity scaling amounts to designing an effective kernel for Hebbian memories.

\subsection{Margins Govern Transformer Usability}
\label{sec:margins-transformer-usability}

Our second observation is that positive margin at stored keys is not enough for a Transformer block to use a fact-storing MLP reliably.
Intuitively, because the attention mechanism perturbs the query before it reaches the MLP, end-to-end usability requires the decoding margin to be bounded away from zero.

\begin{definition}[Synthetic Sequential Factual Recall (SSFR)]
\label{def:ssfr}
Fix a junk-token vocabulary of size $V_{J}$, a model dimension $d$, and a set of $F$
key-value pairs $\{(\mathbf{k}_{i}, \mathbf{v}_{f(i)})\}_{i=1}^{F}$ where
$f : [F] \to [F]$ is a fact set. Then \emph{SSFR} inputs have the form:
$$
  \underbrace{j_{1}, \ldots, j_{J/2}}_{\text{junk prefix}},\;
  \underbrace{k_i}_{\text{key}},\;
  \underbrace{j_{J/2+1}, \ldots, j_{J}}_{\text{junk suffix}},\;
  \underbrace{q}_{\text{query}} \rightarrow \underbrace{v_{f(i)}}_{\text{value}},
$$
where $j_{t} \overset{\mathrm{i.i.d.}}{\sim} \mathrm{Unif}([V_{J}])$ and $q$ is a fixed query token.
The junk length $J$ controls how difficult it is for attention to isolate the relevant key.
\end{definition}

\paragraph{Empirical verification.}
We pretrain an attention-only Transformer block, freeze it, insert a frozen GD-trained fact-storing MLP, and sweep hidden width (see Appendix~\ref{appendix:ssfr-task} for details).
Crucially, although the MLP stores the fact-set as soon as $\gamma_{\min}$ becomes positive, end-to-end SSFR accuracy lags until the margin is bounded away from zero (Figure~\ref{fig:sketched_k2_sweeps}a).
This observation motivates our theoretical study of Hebbian margin bounds in~\Cref{sec:4}.

%% file: sections_short/section_4_hebbian_kernel_mlps/main.tex
In \cref{sec:3}, we identified decoding margin as a property of interest in MLPs. In this section we turn to study the margin and fact storage capacity scaling of MLPs. We begin by proposing a simple bilinear MLP construction, which we term the \emph{Hebbian MLP}, allowing us to characterize how decoding margin and fact storage capacity scales with the number of facts and MLP parameters. Equipped with our simple construction, we demonstrate that MLPs realize optimal margin scaling (\cref{thm:bilinear-scaling}) and that their fact storage capacity scales at the information-theoretically optimal rate (\cref{cor:optimal-fact-storage-capacity}). Finally, we develop kernel-whitened and data-dependent variants of our MLP construction that realize optimal capacity scaling empirically; at matched fact count, the NTK baseline requires roughly $10$-$104\times$ more parameters than our data-dependent construction.

\paragraph{Decoding margin.} We study the MLP decoding margin when viewed as a Hebbian kernel memory with kernel $K$ (\Cref{thm:mlp-hebb-whiten}). Intuitively, the decoding margin is the ``slack'' a model is allowed in its outputs so that it still decodes to the right values (see \cref{fig:main-figure}a). Our margin analysis follows from decomposing the margin into \emph{signal} and \emph{cross-talk} terms (Appendix~\ref{app:theory:margin_bounds:akav}):
\begin{equation}
\gamma_{i,j}
=
\underbrace{
\langle \mathbf{v}_{f(i)} - \mathbf{v}_{j},\; \mathbf{v}_{f(i)} \rangle \,K(\mathbf{k}_{i}, \mathbf{k}_{i})
}_{\text{signal}}
+
\underbrace{
\sum_{t \neq i}
\langle \mathbf{v}_{f(i)} - \mathbf{v}_{j},\; \mathbf{v}_{f(t)} \rangle \,K(\mathbf{k}_{t}, \mathbf{k}_{i})
}_{\text{cross-talk}}.
\label{eqn:margin-decomposition}
\end{equation}
Intuitively, we wish to develop a kernel $K$ that 1) maximizes the signal-to-cross talk ratio while 2) remaining implementable as a gated MLP.

\subsection{A Bilinear Hebbian MLP Construction}
\label{sec:isotropic-case}
\label{sec:isotropic-embeddings}
\label{sec:sketched-k2}
\label{sec:construction}
Our first step toward understanding the decoding margin and storage-capacity
scaling of MLPs is to develop a simple closed-form gated MLP construction capable of storing a fact set. Our \emph{Hebbian MLP} construction is defined as
\begin{equation}\label{eq:mlp-construction}
\mathrm{MLP}(\mathbf{x}) = \mathbf{B}\bigl(\mathbf{A}\mathbf{x} \odot \mathbf{G}\mathbf{x}\bigr)
\end{equation}
with $\mathbf{A}_{i,j}\sim N(0, \frac{1}{m})$, $\mathbf{G}_{i,j}\sim N(0, \frac{1}{m})$, and $\mathbf{B}=\mathbf{V}^T\mPhi$, where $\mPhi_{i}=\mathbf{A}\mathbf{k}_{i} \odot \mathbf{G}\mathbf{k}_{i}$.
The key component of this construction is that it is equivalent to a Hebbian kernel memory (\Cref{app:theory:hebbian_mlps:mlp_kernels}), using an $m$-dimensional feature map to sketch the exact quadratic kernel $K_{2}(\mathbf{x},\mathbf{z}) = \langle \mathbf{x},\mathbf{z} \rangle^2$, which we can analyze:
\begin{equation}\label{eq:approx-quad-kernel}
H(\mathbf{z})= \sum_{i=1}^{F}\mathbf{v}_i\hat{K}_{2}(\mathbf{k}_i, \mathbf{z}), \quad \hat{K}_{2}(\mathbf{k},\mathbf{z})
=
\sum_{r=1}^m
(\mathbf{A}_{r}^\top \mathbf{k})(\mathbf{A}_{r}^\top \mathbf{z})\,
(\mathbf{G}_{r}^\top \mathbf{k})(\mathbf{G}_{r}^\top \mathbf{z}).
\end{equation}
We provide a pseudocode implementation of our construction in~\Cref{alg:sketched_k2_hebbian_whitening}.

\subsection{Margin Scaling}

\subsubsection{Isotropic Embeddings}

We first show that in the isotropic keys and values setting, our bilinear MLP construction's margin scales at an asymptotically optimal rate:
\begin{tcolorbox}[breakable,colback=white,colframe=blue!50!black,boxrule=1.5pt,arc=0mm,left=8pt,right=8pt,top=8pt,bottom=8pt,title={\small\bfseries Key Result: MLP margin scales at optimal rate},colbacktitle=blue!50!black,coltitle=white,titlerule=0.5pt]
\begin{theorem}[MLP Margin Scaling (Isotropic Embeddings Setting) - Informal]\label{thm:informal-iso-iso-margin}
Under isotropic key and value embeddings, the decoding margin of our bilinear MLP construction (\cref{eq:mlp-construction}) scales as:
\begin{equation}
\gamma_{\min}
\ \geq\
\underbrace{1}_{\text{signal}}
\;-\;
\underbrace{C\sqrt{\frac{F\log(F)}{md}}}_{\text{cross-talk}}.
\label{eq:bilinear-scaling-2term}
\end{equation}
\end{theorem}

See \Cref{cor:informal-iso-iso-margin} for formal statement and proof.
\end{tcolorbox}

\begin{remark}[Asymptotic optimality in decoding margin]\label{rem:welch-benchmark}
Our construction implicitly uses a kernel with feature dimension $m$. A rank-limited Welch-style upper bound shows that under isotropic keys and values, no PSD rank-$m$ kernel with near-unit diagonal can asymptotically improve upon the \(\sqrt{F\log\!(F)/(md)}\) term in~\Cref{thm:bilinear-scaling}, up to logarithmic factors. As such, our construction attains the asymptotically optimal decoding margin bound (Appendix~\ref{app:welch-upper-bound-iso}).
\end{remark}

\paragraph{Empirical verification.}
We evaluate the empirical margin scaling against the theoretical scaling as we sweep facts $F$ and MLP hidden-dimension $m$ under isotropic embeddings. Figure~\ref{fig:sketched_k2_sweeps}b validates that our margin bounds closely match the empirical minimum margins ($R^2\geq 0.97$).

\begin{figure*}[t]
\centering
\includegraphics[width=\linewidth]{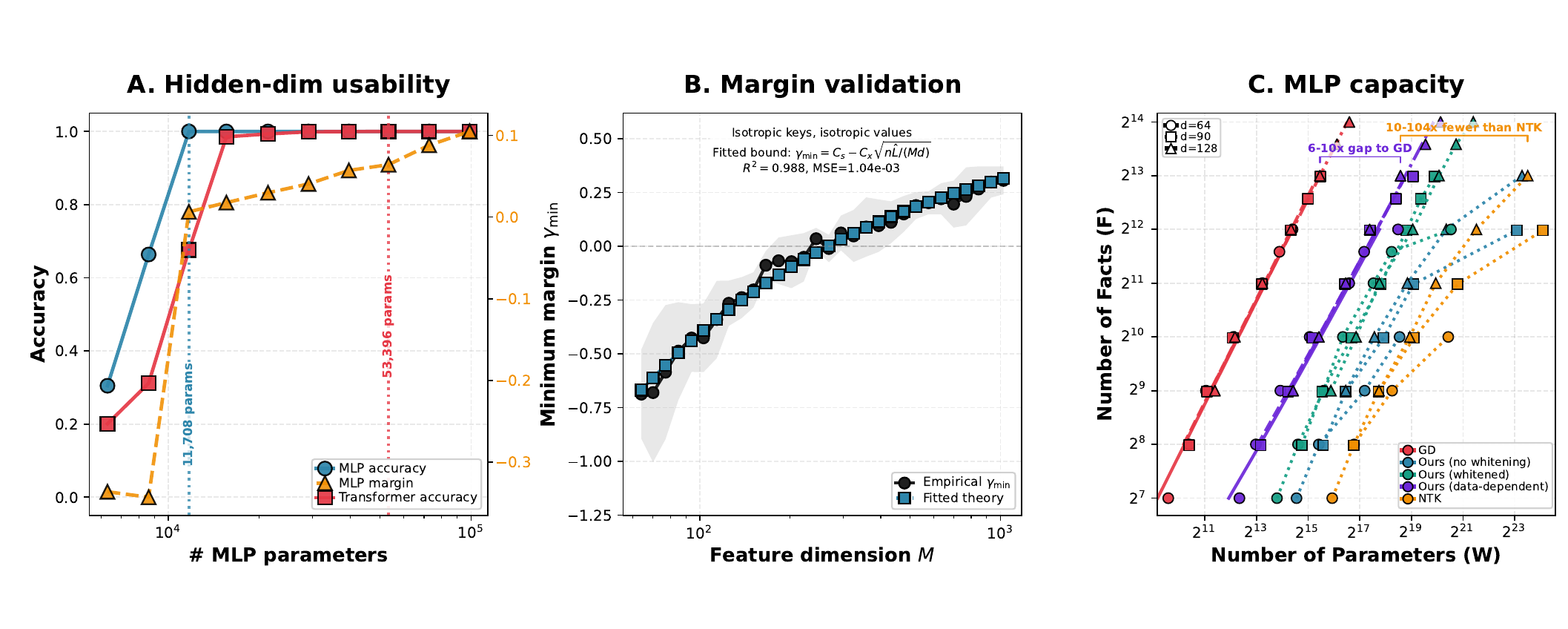}
\caption{\textbf{Decoding margins govern Transformer usability and yield provable capacity bounds.}
\textit{(A)} In a Transformer, inserted fact-storing MLPs become usable only once their decoding margin is bounded away from zero.
\textit{(B)} Under isotropic keys and values, the empirical margin follows our predicted scaling as the MLP hidden dimension $m$ increases.
\textit{(C)} Our construction achieves asymptotically optimal fact-storage capacity scaling under isotropic keys and values.
}
\label{fig:sketched_k2_sweeps}
\end{figure*}

\subsubsection{Beyond isotropic embeddings}

\phantomsection\label{sec:arbk-isov}
\phantomsection\label{sec:isok-arbv}
\phantomsection\label{sec:arbk-arbv}
Our margin decomposition (\cref{eqn:margin-decomposition}) analysis extends beyond isotropic key/value embeddings to arbitrary embedding geometries, like those found in LLMs~\citep{ethayarajh2019contextualcontextualizedwordrepresentations,razzhigaev2024shapelearninganisotropyintrinsic}. To this end, we present the most general margin bound for arbitrary embedding geometries, illustrating how embedding-geometric statistics enter the margin scaling. Furthermore, we provide a full ladder of bounds across the key/value geometry regimes (\cref{tab:2x2}).

\begin{theorem}[MLP Margin Scaling (Arbitrary Embeddings Setting) -- Informal]\label{thm:bilinear-scaling}
Under \emph{arbitrary} key and value embeddings, and in the regime $d\gtrsim\log F$, the decoding margin of our bilinear MLP construction (\cref{eq:mlp-construction}) scales as:
\begin{equation}
\gamma_{\min}
\ \geq\
\underbrace{C \ S_{\mathrm{sig}}}_{\text{signal}}
\;-\;
\underbrace{\sqrt{\frac{F\log(F)}{md}} \ P_{\mathrm{key}}
\ P_{\mathrm{val}}
\ P_{\mathrm{align}} }_{\text{cross-talk}}.
\label{eq:maintext:akav_margin_bound}
\end{equation}
\end{theorem}
See \Cref{cor:bilinear-scaling-simplified} for formal statement and proof, including the formal definitions of the four embedding-geometric statistics $P_{\mathrm{key}}$, $P_{\mathrm{val}}$, $P_{\mathrm{align}}$, and $S_{\mathrm{sig}}$.

\newcommand{\keyact}{\boldsymbol{K}}
\newcommand{\valint}{\boldsymbol{V}}

For a fact set mapping $\vk_i \to \vv_{i}$, define the kernel vectors $\keyact_{i}\in\mathbb R^{F-1}$ and the value interference vectors $\valint_{i,j}\in\mathbb R^{F-1}$ as:
\[
(\keyact_{i})_t \defeq \hat{K}(\vk_i, \vk_t),
\qquad
(\valint_{i,j})_t
\defeq
\langle \vv_i-\vv_j,\vv_t\rangle,
\qquad t\neq i.
\]
From our signal and cross talk decomposition (\cref{eqn:margin-decomposition}), $\keyact_i$ captures which other facts are activated when retrieving fact $i$, while $\valint_{i,j}$ measures which values interfere with distinguishing fact $i$ from competitor $j$. Using these vectors as building blocks, four embedding-geometric statistics enter the isotropic margin bound multiplicatively when generalizing it to arbitrary embedding geometries, with each statistic isolating a distinct source of signal or cross-talk:
\begin{itemize}[leftmargin=0.5cm]
    \item \textbf{Key crowding penalty.}
    \[
    P_{\mathrm{key}}
    \defeq
    \frac{\sqrt{E_K}}{\sqrt{F/m}}
    \qquad
    \left(
    E_K
    \defeq
    \max_{i\in[F]}\|\keyact_{i}\|_2^2
    \right).
    \]
    This measures how much the stored keys overlap under the bilinear featurization,
    relative to the random/isotropic baseline scale $\sqrt{F/m}$.
    Smaller $P_{\mathrm{key}}$ (and thus smaller cross-talk) means the key features are more separated,
    so fewer irrelevant facts are activated by a query.

    \item \textbf{Value crowding penalty.}
    \[
    P_{\mathrm{val}}
    \defeq
    \frac{\sqrt{E_v}}{\sqrt{F/d}}
    \qquad
    \left(
    E_v
    \defeq
    \max_{j\neq i \in [F]}
    \|\valint_{i,j}\|_2^2
    \right).
    \]
    This measures how much the stored value directions overlap,
    relative to the random/isotropic baseline scale $\sqrt{F/d}$.
    Smaller $P_{\mathrm{val}}$ (and thus smaller cross-talk) means the incorrect values are less aligned
    with the correct value margin direction.

    \item \textbf{Key--value alignment penalty.}
    \[
    P_{\mathrm{align}}
    \defeq
    \frac{\kappa}{\sqrt{\log(F)/F}}
    \qquad
    \left(
    \kappa
    \defeq
    \max_{j\neq i \in [F]}\left|\cos\angle\!\left(\keyact_{i},\valint_{i,j}\right)\right|
    \right).
    \]
    This measures alignment between which facts a query activates (kernel column) and which values are most confusable (value interference), relative to the isotropic baseline $\sqrt{\log(F)/F}$.
    Smaller $P_{\mathrm{align}}$ (and thus smaller cross-talk) means key and value errors are orthogonal rather than compounding each other.

    \item \textbf{Signal strength.}
    \[
    S_{\mathrm{sig}}
    \defeq
    \frac{K_{\min}^{\mathrm{diag}}V_{\min}}{(1-\sqrt{\log(F)/d})}
    \qquad
    \left(
    K_{\min}^{\mathrm{diag}}
    \defeq
    \min_{i\in[F]}\hat{K}(\vk_i, \vk_i),
    \quad
    V_{\min}
    \defeq
    \min_{i\neq j}
    \langle \vv_i-\vv_j,\vv_i\rangle
    \right).
    \]
    This measures the strength of the signal relative to the $1-\log(F)/d$ baseline.
    Larger $S_{\mathrm{sig}}$ (and thus stronger decoding signal) means featurized keys have larger norm while value embeddings remain nearly-orthogonal.
\end{itemize}

\paragraph{Empirical verification.} We evaluate our general margin bound (and those from \cref{tab:2x2}) in \cref{fig:app:beta-sweeps}, finding they closely track the empirically observed margin as we vary the key and value crowding ($R^2 \geq 0.95$).

\subsection{Fact-Storage Capacity}
\label{sec:fact-storage-capacity}

Equipped with our margin scaling bounds (\cref{eq:bilinear-scaling-2term,eq:maintext:akav_margin_bound}), we develop storage capacity scaling laws for our bilinear gated MLPs by simply solving for the parameters necessary to make their margin positive.
We find that the storage capacity of our simple MLP construction scales at the information-theoretically optimal rate, under the isotropic embeddings setting, while for non-isotropic embeddings it does so up to penalization factors.

\begin{tcolorbox}[breakable,colback=white,colframe=blue!50!black,boxrule=1.5pt,arc=0mm,left=8pt,right=8pt,top=8pt,bottom=8pt,title={\small\bfseries Key Result: MLPs store facts at an info-theoretically optimal rate.},colbacktitle=blue!50!black,coltitle=white,titlerule=0.5pt]
\begin{corollary}[MLP Fact-storage Capacity (Isotropic Embeddings Setting) - Informal]\label{cor:optimal-fact-storage-capacity}
For isotropic key and value embeddings, our bilinear MLP construction stores $F$ facts using
\[
W = \Theta(md) = \Theta\!\left(F\log(F)\right)
\]
parameters.
Thus MLPs achieve \textbf{information-theoretically optimal} fact-storage capacity.
\end{corollary}
See \Cref{cor:optimal-fact-storage-capacity-formal} for formal statement and proof.
\end{tcolorbox}

\begin{corollary}[Optimal Fact-storage Capacity (Arbitrary Embeddings Setting) - Informal]\label{cor:optimal-fact-storage-capacity-arb}
For arbitrary key and value embeddings, our bilinear MLP construction stores $F$ facts using 
\[
W = \Theta(md) = \Theta\!\left(F\log(F)\left(\frac{ P_{\mathrm{key}} P_{\mathrm{val}} P_{\mathrm{align}} }{S_{\mathrm{sig}}}\right)^2\right)
\]
parameters.
Thus MLPs achieve \emph{information-theoretically optimal} fact-storage capacity up to penalization factors.
\end{corollary}
See \Cref{cor:optimal-fact-storage-capacity-arb-formal} for formal statement and proof.

\Cref{cor:optimal-fact-storage-capacity} counts real-valued parameters; Appendix~\ref{app:theory:margin_bounds:bit_complexity} shows that, under bounded precision, our construction incurs only an extra logarithmic factor in total bit complexity.

\paragraph{Empirical verification.}
Figure~\ref{fig:sketched_k2_sweeps}c compares the storage scaling of our construction to gradient descent-trained (GD) MLPs and the NTK construction from~\citet{nichani2024understandingfactualrecalltransformers}. Notably, we include two more closed-form variants of our construction in our empirical analysis:
\begin{itemize}[leftmargin=0.5cm]
    \phantomsection\label{kernel-whitened-construction}
    \item \textbf{Kernel-whitened construction.} Our kernel-whitened MLP construction whitens the sketched quadratic kernel $\hat{K}$ in our vanilla MLP construction with its empirical covariance. This construction is motivated by the observations that 1) the key crowding penalty $P_{\mathrm{key}}$ governs the cross-talk magnitude scale in our signal and cross talk decomposition (\cref{eq:maintext:akav_margin_bound}) and 2) whitening the sketched quadratic kernel $\hat{K}$ with its empirical covariance (Appendix~\ref{app:theory:margin_bounds:kernel_whitening_crowding}) reduces $P_{\mathrm{key}}$ (\Cref{lem:whiten_min_frob}). To our knowledge, our \emph{kernel-whitened} construction is the first closed-form MLP construction to empirically achieve the optimal fact-storage capacity scaling $W \,= \Theta\Big(F\log\!\Big(\frac{F}{\delta}\Big)\Big)$.
    \phantomsection\label{data-dependent-construction}
    \item \textbf{Data-dependent construction.} Our data-dependent MLP closed-form construction solves for each weight matrix in our vanilla MLP construction via a least-squares objective, as opposed to initializing them randomly. Intuitively, the least squares objective we solve for each matrix leverages the key and value geometry to maximize the construction's margin (Appendix~\ref{app:theory:margin_bounds:data_dependent_construction}). This approach improves storage capacity without gradient descent: at matched fact count, the NTK baseline requires roughly $10$-$104\times$ more parameters than our data-dependent construction, while our data-dependent construction only requires about $6$-$10\times$ more parameters than GD.
\end{itemize}

%% file: sections_short/section_5_transformer_integration/main.tex
In \Cref{sec:hebbian-kernel-mlps}, we characterized the margin scaling of \emph{Hebbian MLPs}. In this section we leverage our margin scaling bounds to study the storage capacity scaling of Transformer blocks, where MLP inputs are no longer exact keys but instead are noisy attention outputs. Crucially, we demonstrate that the fact storage capacity of Transformer blocks scales at the information-theoretically optimal rate, provided the attention noise remains bounded. Further, we empirically show that our construction remains usable within Transformer blocks for factual recall tasks; at matched fact count, the NTK baseline requires roughly $15$-$63\times$ more parameters than Transformer blocks using our data-dependent construction. Finally, we show that fact-storing MLPs unlock a new capability: modular, zero-shot fact editing within a Transformer by \emph{swapping out its MLP}.

\begin{figure*}[t]
\centering
\includegraphics[width=\linewidth]{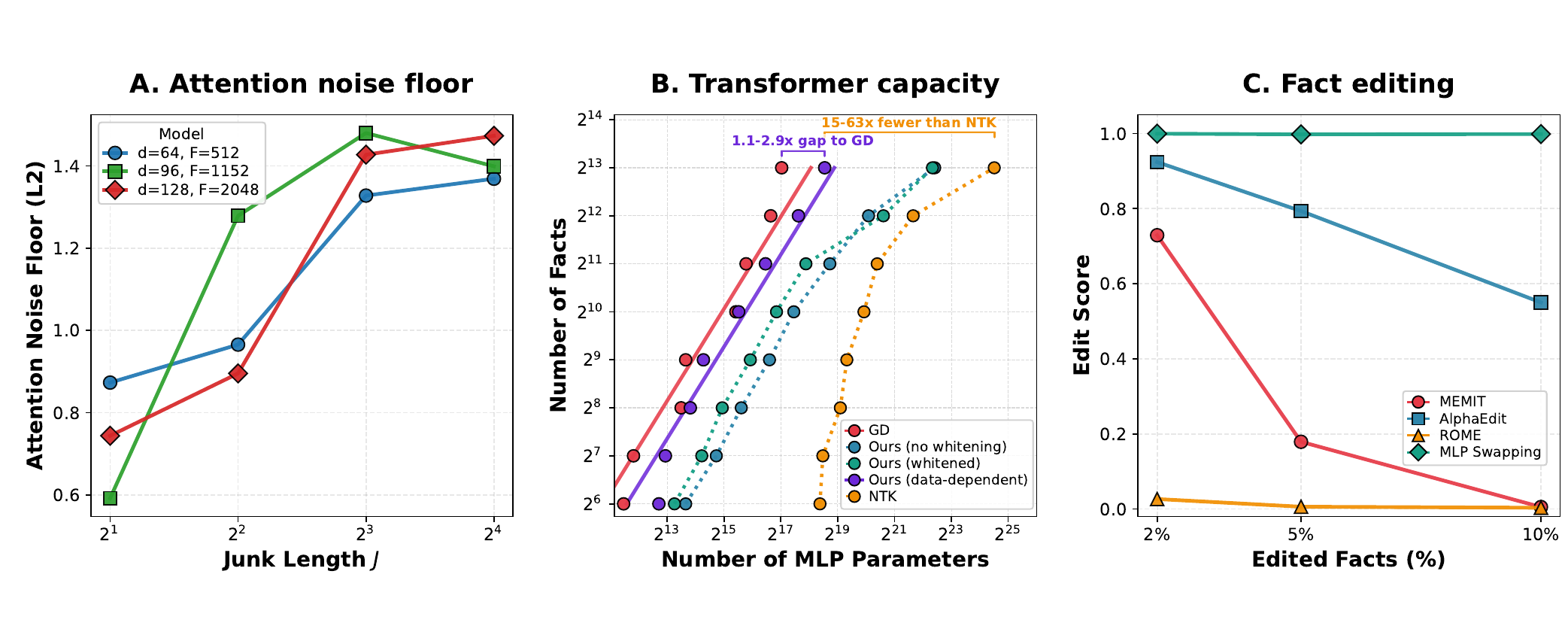}
\caption{\textbf{Transformer blocks achieve information-theoretic optimal fact-storage capacity and enable modular fact editing.}
\textit{(A)} Attention noise ceiling $\varepsilon_{\mathrm{attn}}$ scales with junk context length $J$ in SSFR.
\textit{(B)} Under bounded attention noise, Transformer blocks with our MLP construction achieve information-theoretic optimal fact-storage capacity scaling.
\textit{(C)} \textsc{MLP Swapping} achieves near-perfect fact-editing score ($>\!0.99$) at up to $10\%$ edited facts, more than 40\% better than existing fact-editing baselines.
}
\label{fig:ssfr_main_fig}
\end{figure*}

\subsection{MLPs in Transformers Achieve Optimal Fact-Storage Capacity}
\label{subsec:transformer-capacity}
\label{subsec:attention-noise}

We start by investigating why MLPs need positive margin in Transformers.
Unlike in the standalone setting, attention does not query the MLP with the exact stored key.  We quantify the worst-case deviation formally with an \emph{attention noise ceiling}, which we define as the maximum noise an attention layer can produce when querying an MLP for a fact:

\begin{definition}[Attention noise ceiling -- informal]
\label{def:attn-noise}
Let $Q_i \subset \mathbb{R}^d$ be the set of all possible queries an attention layer can produce when querying the MLP for the fact corresponding to the key $\mathbf{k}_i$. The \emph{attention noise ceiling} is defined as
\[
  \varepsilon_{\mathrm{attn}}
  := \max_{i \in [F]}\; \max_{\mathbf{q}\in Q_i} \|\mathbf{q} - \mathbf{k}_{i}\|_{2}.
\]
\end{definition}

Intuitively, we find that the attention noise ceiling increases with the number of distractor tokens in the sequence being processed by a Transformer block (Figure~\ref{fig:ssfr_main_fig}). We next show that Transformer usability reduces to whether the MLP margin can handle perturbations at this scale.

\paragraph{MLPs achieve optimal fact-storage capacity in Transformers.}
We now present our main result. Equipped with our margin scaling analysis in MLPs, we demonstrate that Transformer blocks can store facts at an information-theoretic optimal rate, provided bounded attention noise ceiling:

\begin{tcolorbox}[breakable,colback=white,colframe=blue!50!black,boxrule=1.5pt,arc=0mm,left=8pt,right=8pt,top=8pt,bottom=8pt,title={\small\bfseries Key Result: Transformer blocks can store facts at an info-theory optimal rate.},colbacktitle=blue!50!black,coltitle=white,titlerule=0.5pt]
\begin{theorem}[Transformer Block Fact-storage Capacity (Isotropic Embeddings) - Informal]
\label{thm:noise-robust-capacity}
A Transformer block equipped with a fact-storing bilinear MLP, with non-trivial margin $\gamma_{\min}>c_0>0$, for constant $c_0$, can store $F$ facts using
\[
  W = \Theta(md)
    = \Theta\!\left(F\log (F)\right)
\]
MLP parameters, provided the attention layer in the block satisfies the \textit{attention noise ceiling}
\begin{equation}
  \varepsilon_{\mathrm{attn}}
  \lesssim
  \frac{c_0}
  {L_{\mathrm{bil}}\!\left(
    \sqrt{\frac{F\log(F)}{d}}
  \right)} .
  \label{eq:eps-cap}
\end{equation}
where $L_{\mathrm{bil}}$ is the Lipschitz constant of the MLP. 
\end{theorem} See Appendix~\ref{app:theory:transformers_perturbation} for formal statement and proof.
\end{tcolorbox}
\Cref{cor:capacity-new} presents the formal statement and proof. We note that this result can be easily extended to arbitrary embedding geometries, incurring the same penalization factors from \cref{eq:maintext:akav_margin_bound}.

\paragraph{Empirical verification.}
We produce GD, NTK, and our constructed MLPs, freeze their parameters, then insert them into a 1-layer Transformer and train on the SSFR task.
Figure~\ref{fig:ssfr_main_fig}b validates the predicted asymptotically optimal capacity scaling for Transformer blocks using our construction.
Appendix~\ref{app:theory:transformers_perturbation} reports a complementary per-key margin diagnostic; in particular, \Cref{fig:margin-violin} shows that the usable-key fraction inferred from per-key margins closely tracks end-to-end Transformer accuracy.
Among the constructions we evaluate, our data-dependent construction only requires at most $3\times$ more parameters than GD MLPs at matched fact count. Relative to the NTK baseline, our data-dependent construction requires $15$-$63\times$ fewer parameters at matched fact count. See Appendix~\ref{app:theory:transformers_perturbation} for further diagnostics.

\subsection{Fact Editing via \textsc{MLP Swapping}}
\label{sec:fact-editing}

Having demonstrated that fact-storing MLPs are usable inside Transformers, we now use GD-trained fact-storing MLPs to show a simple proof-of-concept method for \emph{zero-shot fact editing}. We call this procedure \textsc{MLP Swapping}: to edit the model's facts, we construct a new MLP storing the revised fact set and swap it into the Transformer, with \emph{no further tuning of the Transformer's parameters}.

We evaluate on a synthetic author-book language-modeling task (Appendix~\ref{appendix:fact-editing}) using a one-layer Transformer trained to store book--author facts through a frozen fact-storing MLP.
After training, we edit a subset of stored facts and evaluate two metrics. The first is the standard fact-editing score~\citep{memit}, which jointly captures edit efficacy (edited facts predict the new values), specificity (unedited facts stay correct), and paraphrase generalization (edits transfer to paraphrased prompts). The second is the non-fact PPL ratio, measuring post-edit versus pre-edit perplexity on non-fact tokens.
\textsc{MLP Swapping} achieves near-perfect fact-editing score across the edit fractions we test.
At $10\%$ edited facts, \textsc{MLP Swapping} achieves score $0.999$---a $44.9$-percentage-point gain over the strongest baseline, AlphaEdit~\citep{fang2025alphaeditnullspaceconstrainedknowledge} (score $0.550$)---while achieving a non-fact PPL ratio of only $1.02$ (Figure~\ref{fig:ssfr_main_fig}c).

Appendix~\ref{appendix:fact-editing} further shows that \textsc{MLP Swapping} also works for a data-dependent constructed Hebbian MLP.
\textsc{MLP Swapping} with Hebbian MLPs keeps the edit score above \(0.98\) through up to \(10\%\) edited facts, while the strongest baseline reaches only \(0.847\)---a 13.3-percentage-point gain (Figure~\ref{fig:fact-editing-constructed-data-dependent}).

%% file: sections_short/section_6_conclusion/main.tex
\section{Discussion}

Our work presents a stepping stone toward understanding MLPs within Transformers from a constructive lens.
We present an MLP construction that achieves optimal margin and fact-storage capacity, provides provable margin guarantees for arbitrary embeddings, and remains usable within Transformer blocks for factual recall.
We also show an application of modular fact-storing MLPs in fact editing, illustrating a path toward robust, modular knowledge manipulation in LLMs.

Our analysis currently applies to constructed MLPs in a single-layer Transformer setting.
Extending to MLPs in pretrained LLMs---for example, by understanding how the geometries of real LLM embeddings affect the kernels MLPs learn---would provide a principled lens for investigating how trained MLPs store knowledge.
Furthermore, moving beyond the single-layer setting would let us study multi-hop recall and more realistic editing scenarios.

%% file: sections_short/impact_statement.tex
\section*{Impact Statement}
This paper presents work whose goal is to advance the field of machine learning. There are many potential societal consequences of our work, none of which we feel must be specifically highlighted here.

%% file: shared/acknowledgements.tex
\section*{Acknowledgements}
\label{sec:acknowledgements}

The authors thank Yasa Baig, Kelly Buchanan, Mayee Chen, Vivien Cheng, Catherine Deng, Owen Dugan, Rajat Vadiraj Dwaraknath, Neel Guha, Junmiao Hu, Ishan Khare, Hermann Kumbong, Eshaan Nichani, Jon Saad-Falcon, Thanawat Sornwanee, Stuart Sul, Alex Waitz, John Winnicki, Morris Yau, Michael Zhang, and Dylan Zinsley for their helpful feedback and discussion.

The authors gratefully acknowledge the support of NIH under No. U54EB020405 (Mobilize), NSF under Nos. CCF2247015 (Hardware-Aware), CCF1763315 (Beyond Sparsity), CCF1563078 (Volume to Velocity), and 1937301 (RTML); US DEVCOM ARL under Nos. W911NF-23-2-0184 (Long-context) and W911NF-21-2-0251 (Interactive Human-AI Teaming); ONR under Nos. N000142312633 (Deep Signal Processing); Stanford HAI under No. 247183; NXP, Xilinx, LETI-CEA, Intel, IBM, Microsoft, NEC, Toshiba, TSMC, ARM, Hitachi, BASF, Accenture, Ericsson, Qualcomm, Analog Devices, Google Cloud, Salesforce, Total, the HAI-GCP Cloud Credits for Research program, the Stanford Data Science Initiative (SDSI), and members of the Stanford DAWN project: Meta, Google, and VMWare. The U.S. Government is authorized to reproduce and distribute reprints for Governmental purposes notwithstanding any copyright notation thereon. Any opinions, findings, and conclusions or recommendations expressed in this material are those of the authors and do not necessarily reflect the views, policies, or endorsements, either expressed or implied, of NIH, ONR, or the U.S. Government.
JL is supported by the Department of Energy Computational Science Graduate Fellowship under Award Number DE-SC0023112.
AR's research is supported by NSF grant CCF\#2247014.

%% file: appendix_short/main.tex
\input{appendix/main}

\newpage

\input{appendix_short/related_work}

%% file: appendix/main.tex
\section*{Appendix}

\startcontents[appendix]
\printcontents[appendix]{}{1}{\setcounter{tocdepth}{3}}
\newpage

\section{Experiments}
\label{app:expt}
\input{appendix/experiments}

\newpage

\section{Theory}
\label{app:extended_theory}
\input{appendix/theory}

\stopcontents[appendix]

%% file: appendix/experiments.tex
\subsection{Model Definitions}
\label{app:exp:mlp-variants}
\input{appendix/experiments/mlp_variants}

\subsection{Experimental details for~\Cref{sec:3}}
\label{app:exp:section3}
\input{appendix/experiments/section_3_basic_setting}

\subsection{Experimental details for~\Cref{sec:4}}
\label{app:exp:section4}
\input{appendix/experiments/section_4_hebbian_kernel_mlps}

\subsection{Experimental details for~\Cref{sec:5}}
\label{app:exp:section5}
\input{appendix/experiments/section_5_transformer_integration}

\subsection{Additional Empirical Results}
\label{app:exp:additional-results}
\input{appendix/experiments/additional_empirical_results}

%% file: appendix/experiments/mlp_variants.tex
\subsubsection{MLP Variants}
\label{app:exp:mlp-variant-definitions}

Across the standalone fact-storage-capacity experiments (\Cref{sec:fact-storage-capacity,app:exp:section4:capacity}) and the Transformer-block capacity experiments (\Cref{appendix:lm-scaling}), we compare the following MLP variants at matched hidden width $m$ on the same synthetic fact sets:

\begin{description}
    \item[\textbf{GD (gradient-descent trained).}] Our default GD baseline is a bilinear gated-identity MLP, i.e.\ the bilinear specialization of the gated MLP family in \Cref{eq:gated-mlp-class}.
    Unless stated otherwise, we train these models for
    $10{,}000$ epochs using Adam with initial learning rate $10^{-3}$ and a
    cosine-annealing schedule down to $10^{-6}$.

    \item[\textbf{NTK.}] Our NTK baseline is the degree-$1$ Hermite weight construction of~\citet{nichani2024understandingfactualrecalltransformers}, which uses a gated ReLU MLP architecture.
    A self-contained description is given in \Cref{app:theory:hebbian_mlps:ntk_construction}.

    \item[\textbf{Our construction.}] Our closed-form construction is a Hebbian MLP with
    sketched-$K_{2}$ kernel, i.e. the bilinear random-feature construction described in
    \Cref{sec:sketched-k2,alg:sketched_k2_hebbian_whitening}.
    We study three variants:
    \begin{itemize}
        \item \textbf{Unwhitened.} This is the raw sketched-$K_{2}$ Hebbian construction,
        with the bilinear random-feature map used directly in the final Hebbian readout.
        \item \textbf{Whitened.} This uses the same bilinear feature map, but applies the
        full readout whitening procedure from
        \Cref{app:theory:margin_bounds:kernel_whitening} with ridge parameter $10^{-6}$.
        \item \textbf{Data-dependent.} This keeps the same bilinear architecture, but performs least squares solves to obtain the random feature vectors within the Hebbian kernel, as described in
        \Cref{app:theory:margin_bounds:data_dependent_construction}, then applies whitening with ridge parameter $10^{-6}$.
    \end{itemize}
\end{description}

\subsubsection{Transformer Setups}
\label{app:exp:transformer-setups}
\label{app:ssfr-training-setup}
We describe details about the modified 1-layer, 1-head GPT-style Transformer architecture used across all of our SSFR experiments unless otherwise specified.

We freeze the input and output embeddings, and when an experiment includes a fact-storing MLP, that MLP uses the same embeddings as the Transformer.
Positional encodings are disabled throughout.
We use RMSNorm before the attention layer, and use unit-RMSNorm (which projects to the unit sphere) before the MLP and before the final Transformer output.
We train with AdamW using learning rate $2\times 10^{-4}$, batch size $1{,}280$, and $4{,}000$ iterations.

Our experiments include two different training protocols:
\begin{itemize}
    \item In the \emph{pretrained attention} setting, we first train attention alone on a dummy SSFR task using the identity map, so that it learns to query the MLP when faced with distractor junk tokens in its context.
    After training, we freeze the attention layer, then train a frozen fact-storing MLP on a fresh random-permutation fact set using the same embeddings, insert it into the Transformer, and evaluate the combined model on SSFR with that fact set.
    \item In the \emph{inserted MLP} setting, we first construct or train a frozen fact-storing MLP on a random-permutation fact set, then insert it into the Transformer block. We then train the attention layer around that frozen MLP on SSFR using the same fact set.
\end{itemize}

%% file: appendix/experiments/section_3_basic_setting.tex
\subsubsection{Synthetic Sequential Factual Recall (SSFR)}
\label{appendix:ssfr-task}

We instantiate SSFR using the contextual single-token recall task
\[
\mathcal{S}_{\mathrm{SSFR}}[f]
:=
\left\{
\operatorname{concat}(j^{\mathrm{pre}},\, k,\, j^{\mathrm{suf}},\, q,\, f(k))
\;:\;
k\in\mathcal{S}_{k},\;
j^{\mathrm{pre}}\in\mathcal{J}_{\mathrm{pre}},\;
j^{\mathrm{suf}}\in\mathcal{J}_{\mathrm{suf}}
\right\},
\]
where $f:\mathcal{S}_{k}\to\mathcal{S}_{v}$ is a bijection defining the fact set, $q$ is a dedicated query token, and $j^{\mathrm{pre}},j^{\mathrm{suf}}$ are junk-token strings that do not belong to the fact vocabulary. Keys and values are treated as single tokens.

Unless otherwise specified, the default SSFR configuration used in our experiments uses a junk vocabulary of size $9$ and junk-prefixes and junk-suffixes of length $9$.
Transformer architecture and training details are summarized in Appendix~\ref{app:exp:transformer-setups}.

\subsubsection{Margins Govern Transformer Usability Sweep}
\label{app:exp:section3:hidden-width}
The hidden-width sweep in Figure~\ref{fig:sketched_k2_sweeps}a uses the \emph{pretrained attention} setting from Appendix~\ref{app:exp:transformer-setups} and uses GD-trained fact-storing MLPs.
We use $d_{\mathrm{model}}=128$,
$F=2048$, $J=V_J=9$, and sweep across $16$ logarithmically spaced hidden widths in $[16,256]$, with four seeds per width.
For each width we report the standalone MLP margin and accuracy together with the accuracy of the combined Transformer block on the random fact set.

%% file: appendix/experiments/section_4_hebbian_kernel_mlps.tex
\subsubsection{MLP Fact-Storage Capacity}
\label{app:exp:section4:capacity}
We estimate the MLP fact-storage capacity in Figure~\ref{fig:sketched_k2_sweeps}c by binary-searching over hidden width $m$ at fixed embedding dimension $d$ and fact count $F$.
We sweep over
\[
d\in\{64,90,128\},
\qquad
\alpha:=F/d^2\in\{1/32,1/16,1/8,1/4,1/2,3/4,1\}.
\]
For each $(\text{method},d,\alpha)$ tuple, we generate a random-permutation fact set and sample key and value embeddings from the unit sphere.
We then binary-search for the smallest hidden width $m$ that attains 100\% fact storage on the sampled fact set and report the corresponding parameter count $W$.
We compare GD-trained bilinear MLPs, our sketched-$K_2$ Hebbian construction (with and without whitening, and our data-dependent kernel variant), and the NTK baseline from~\citet{nichani2024understandingfactualrecalltransformers}, as described in Appendix~\ref{app:exp:mlp-variants}.

\subsubsection{Margin Sweeps}
\label{app:exp:section4:margins}
We use the same bilinear sketched-$K_2$ construction for the isotropic margin sweeps in Figure~\ref{fig:sketched_k2_sweeps}b, as well as for the arbitrary-geometry comparisons summarized in Table~\ref{tab:2x2}.

\paragraph{Isotropic keys and values.}
\label{app:exp:section4:iso-iso}
For the isotropic keys and isotropic values sweeps supporting Appendix~\ref{app:theory:margin_bounds:rkrv}, we sample unit-norm spherical keys and values from a random-permutation fact set, vary either $F$ or $m$, and construct the MLP at each point.
We then compare the measured minimum margin with the theoretical prediction (\Cref{eq:app:rkrv:sketchedk2_full}).
The $F$-sweep (exact $K_2$) uses $d=32$, $m=512$, $F\in[32,2048]$ (30 log-spaced points, 5 seeds).
The $F$-sweep (bilinear sketched-$K_2$) uses $d=64$, $m=512$, $F\in[32,512]$ (30 log-spaced points, 5 seeds).
The $m$-sweep (bilinear bilinear sketched-$K_2$) uses $d=64$, $F=256$, $m\in[64,1024]$ (30 log-spaced points, 5 seeds).

\begin{figure}[h]
\centering
\includegraphics[width=0.60\linewidth]{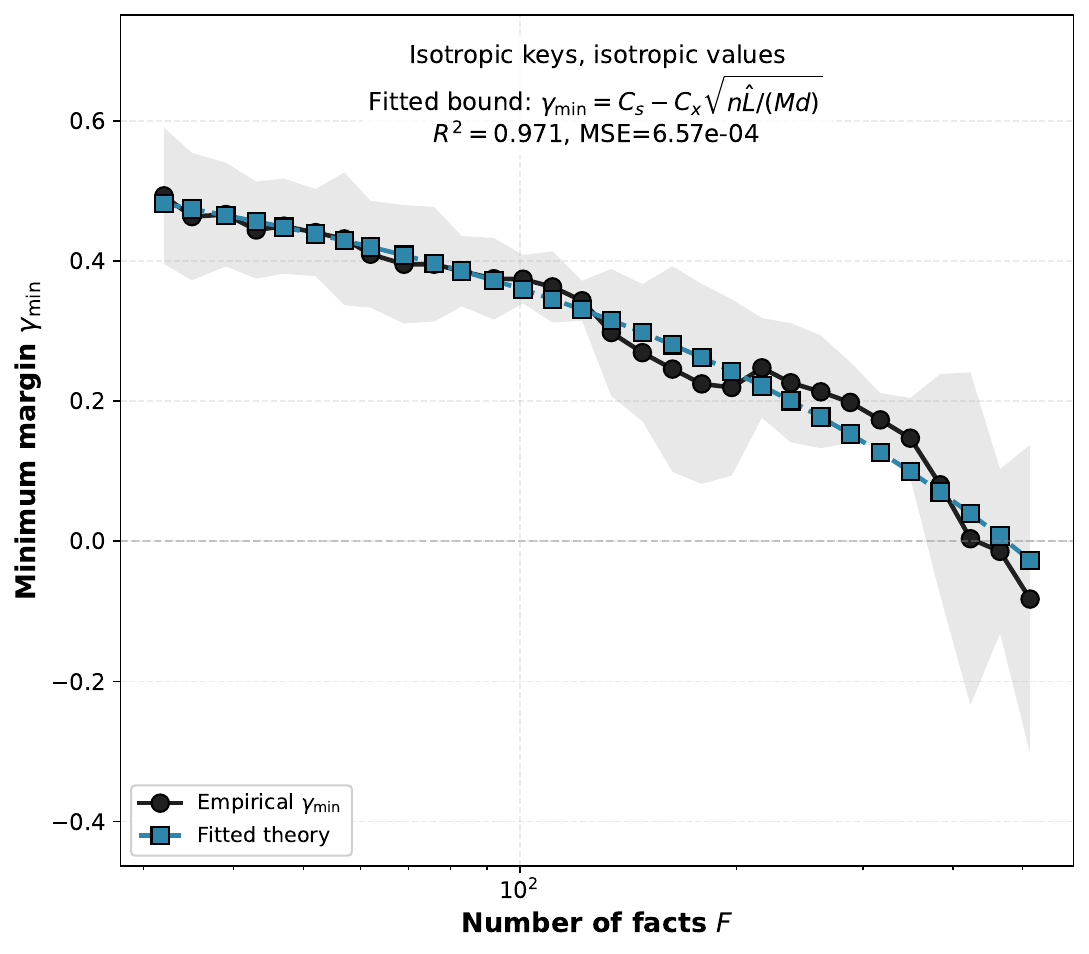}
\caption{\textbf{Margin scaling with number of facts (isotropic setting).}
Minimum margin of bilinear sketched-$K_2$ MLPs decreases with the number of stored facts $F$, following the predicted scaling from \Cref{eq:app:rkrv:sketchedk2_full}. The empirical margin closely tracks the theoretical bound across the $F$-sweep with $d=64$, $m=512$.}
\label{fig:app:rf-f-sweep}
\end{figure}

\paragraph{Arbitrary keys, isotropic values.}
\label{app:exp:section4:anisotropic-keys}
For the arbitrary-key sweep supporting Appendix~\ref{app:theory:margin_bounds:akrv}, we start from spherical keys and apply a rank-1 spike transform
\[
\mathbf{k}_{i}'
=
\frac{\mathbf{k}_{i} + \beta \langle \mathbf{k}_{i}, \mathbf{u}\rangle \mathbf{u}}
{\|\mathbf{k}_{i} + \beta \langle \mathbf{k}_{i}, \mathbf{u}\rangle \mathbf{u}\|_2},
\]
where $\mathbf{u}\in\mathbb{S}^{d-1}$ is fixed within a sweep and $\beta\ge 0$ controls the anisotropy strength.
As $\beta$ increases, the keys crowd along $\mathbf{u}$, increasing the key-side quantities entering the bound, especially $E_{K}$.
Values remain isotropic, sampled uniformly from the unit sphere.
At each $\beta$, we construct the MLP, measure $\gamma_{\min}$, and compare it against the theorem, plug-in, and heuristic predictions for the resulting key geometry (\Cref{eq:app:akrv:thm}).
We use $d=64$, $F=128$, $m=512$, $\beta\in[0,5]$.

\paragraph{Isotropic keys, arbitrary values.}
\label{app:exp:section4:anisotropic-values}
For the arbitrary-value sweep supporting Appendix~\ref{app:theory:margin_bounds:rkav}, we keep the keys isotropic and apply the same rank-1 spike transform to the values:
\[
\mathbf{v}_{i}'
=
\frac{\mathbf{v}_{i} + \beta \langle \mathbf{v}_{i}, \mathbf{u}\rangle \mathbf{u}}
{\|\mathbf{v}_{i} + \beta \langle \mathbf{v}_{i}, \mathbf{u}\rangle \mathbf{u}\|_2}.
\]
Again $\mathbf{u}$ is fixed within a sweep and $\beta$ is the control parameter. This changes the value-side quantities $V_{\min}$, $B_{Y}$, and $E_{v}$ while preserving unit norm.
Keys remain isotropic, sampled uniformly from the unit sphere.
At each $\beta$, we rebuild the MLP, measure the empirical decoding margin, and compare it against the corresponding theorem, plug-in, and heuristic predictions (\Cref{eq:app:rkav:thm}).
We use $d=64$, $F=128$, $m=512$, $\beta\in[0,10]$.

\paragraph{Arbitrary keys and values.}
\label{app:exp:section4:anisotropic-keys-values}
For the fully structured sweep supporting Appendix~\ref{app:theory:margin_bounds:akav}, we apply the same rank-1 spike model to both keys and values, using the same spike strength $\beta$ and fixed direction $\mathbf{u}$:
\[
\mathbf{e}_{i}'
=
\frac{\mathbf{e}_{i} + \beta \langle \mathbf{e}_{i}, \mathbf{u}\rangle \mathbf{u}}
{\|\mathbf{e}_{i} + \beta \langle \mathbf{e}_{i}, \mathbf{u}\rangle \mathbf{u}\|_2},
\qquad
\mathbf{e}_{i}\in\{\mathbf{k}_{i},\mathbf{v}_{i}\}.
\]
This setting makes key crowding and value interference vary coherently, which is the regime where the coupling factor $\kappa$ emerges.
For each $\beta$, we recompute the geometric summary statistics entering the deterministic bound and compare the measured margin against the theorem, plug-in, and heuristic predictions, focusing on the composite cross-talk scale $\sqrt{E_K}\sqrt{E_v}\,\kappa$ (\Cref{eq:app:akav:thm}).
We use $d=64$, $F=128$, $m=512$, $\beta\in[0,3.35]$.

\begin{figure*}[t]
\centering
\begin{subfigure}[t]{0.32\linewidth}
\centering
\includegraphics[width=\linewidth]{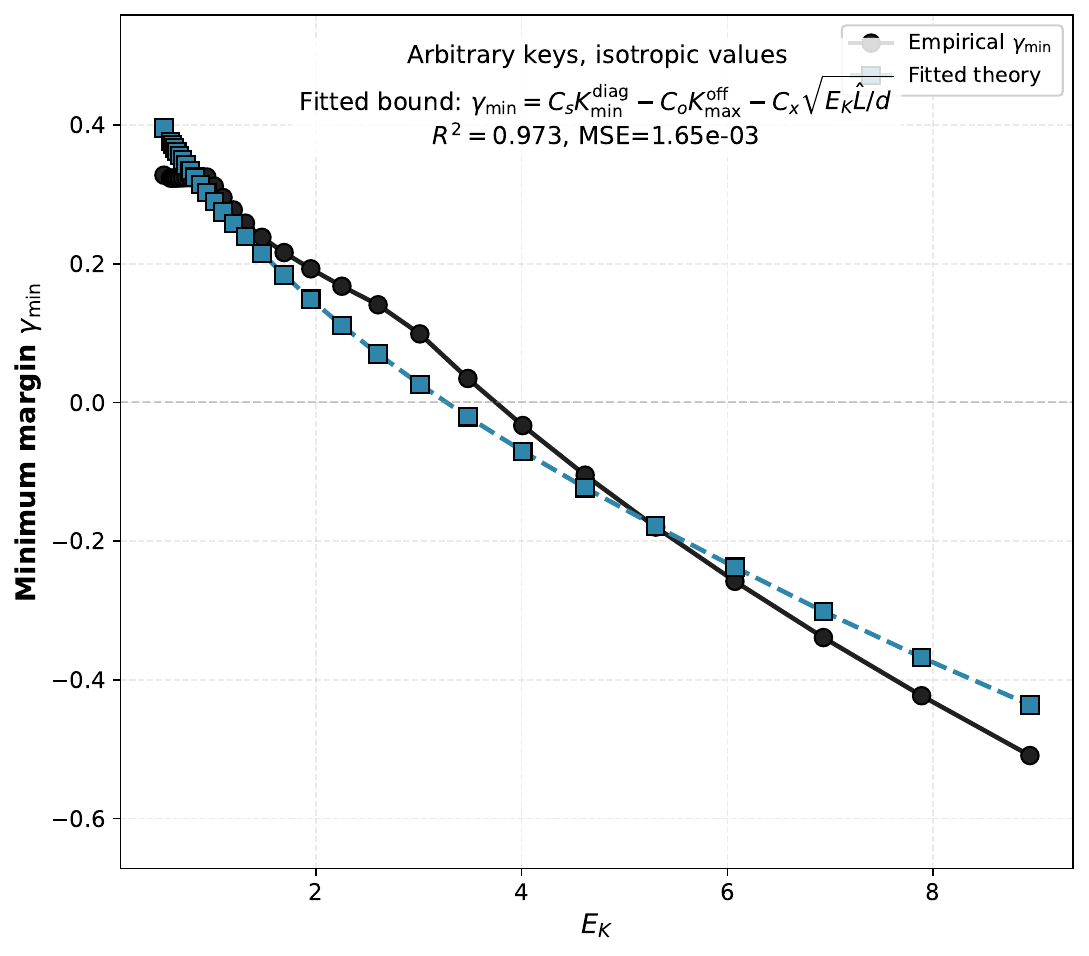}
\caption{Arbitrary keys, isotropic values ($\beta$-sweep). As $\beta$ increases, key crowding grows and the measured margin tracks the theoretical bound via key-geometry terms $K_{\min}^{\mathrm{diag}}$, $K_{\max}^{\mathrm{off}}$, and $E_K$.}
\label{fig:app:beta-sweep-akrv}
\end{subfigure}\hfill
\begin{subfigure}[t]{0.32\linewidth}
\centering
\includegraphics[width=\linewidth]{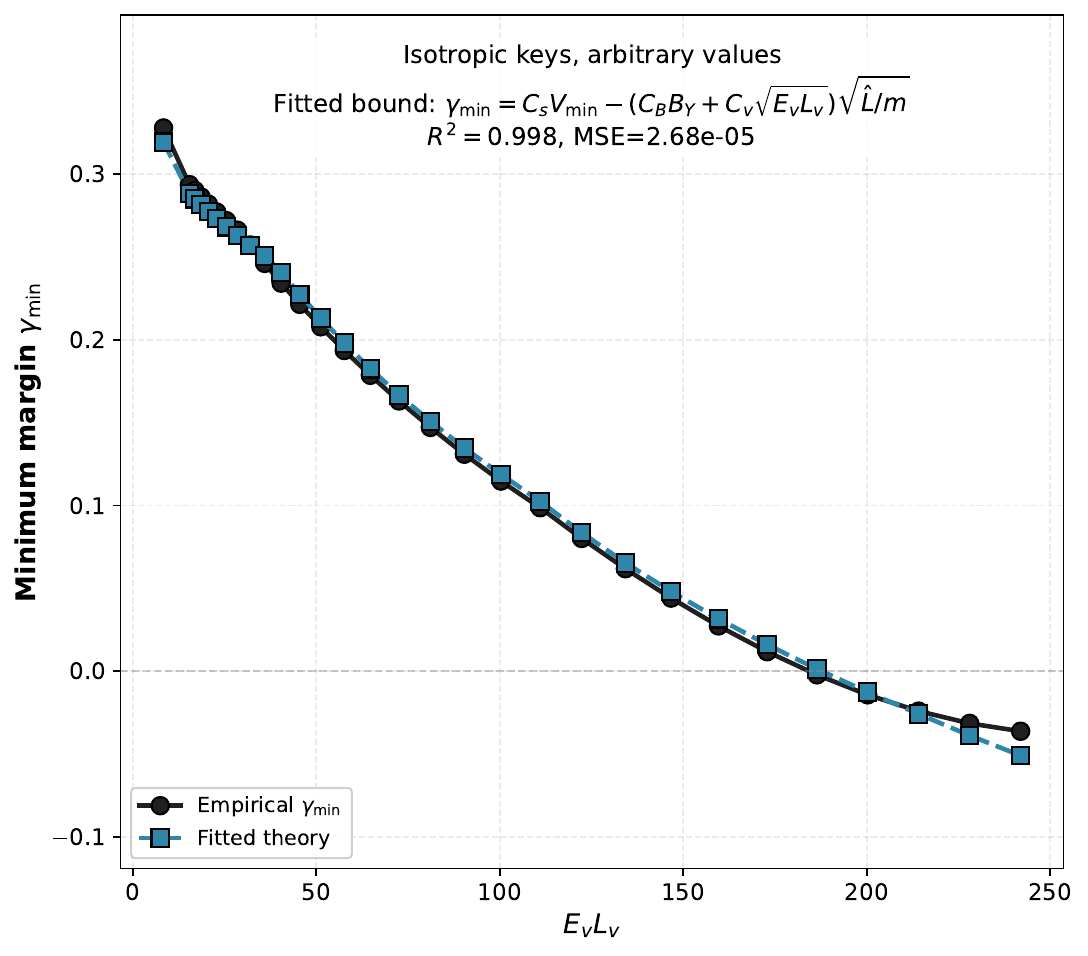}
\caption{Isotropic keys, arbitrary values ($\beta$-sweep). Increasing $\beta$ concentrates values and the margin closely follows the bound through value-geometry terms $V_{\min}$, $B_Y$, and $E_v$.}
\label{fig:app:beta-sweep-rkav}
\end{subfigure}\hfill
\begin{subfigure}[t]{0.32\linewidth}
\centering
\includegraphics[width=\linewidth]{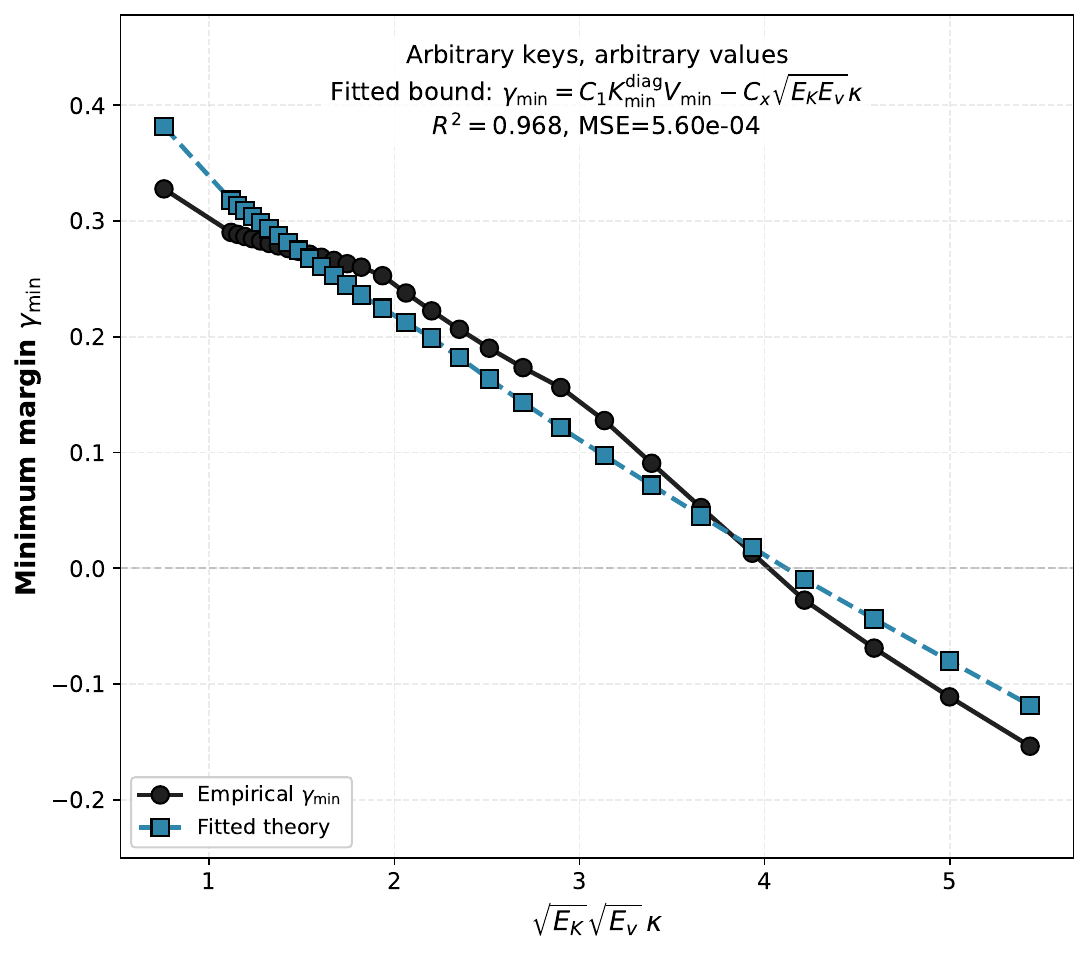}
\caption{Arbitrary keys and values ($\beta$-sweep). Both geometries are spiked simultaneously; the margin degradation is governed by the composite cross-talk scale $\sqrt{E_K}\sqrt{E_v}\,\kappa$, matching the deterministic bound.}
\label{fig:app:beta-sweep-akav}
\end{subfigure}
\caption{\textbf{Margin $\beta$-sweeps across arbitrary key/value geometry regimes.}
Each panel applies a rank-1 spike transform with strength $\beta$ to the keys (left), values (middle), or both (right), using $F=128$ facts.
In all three cases the theoretical bound tracks the empirically measured minimum margin ($R^2\geq0.95$), validating the margin decomposition of \Cref{eqn:margin-decomposition} and the geometric summary statistics of \Cref{tab:2x2}.
The arbitrary keys and values plot uses the raw summary-statistic form of \Cref{thm:combined-deterministic}, which is equivalent to the penalty-statistic general margin/capacity bound in \Cref{cor:optimal-fact-storage-capacity-arb-formal} of the main text.}
\label{fig:app:beta-sweeps}
\end{figure*}

%% file: appendix/experiments/section_5_transformer_integration.tex
\subsubsection{Attention-Noise Sweep}
\label{app:exp:section5:attn-noise}
The attention-noise experiment in Figure~\ref{fig:ssfr_main_fig}a uses the \emph{pretrained attention} Transformer training setup (Appendix~\ref{app:exp:transformer-setups}) and isolates the attention module during the attention-only pretraining phase.
At evaluation, for each stored key and junk context sampled, we measure the $\ell_2$ deviation between the attention output at the query position and the corresponding stored key embedding. We then aggregate these deviations to estimate the attention noise floor.

We vary junk length
\(
J\in\{2,4,8,16\}
\)
and couple the junk vocabulary size to the length ($V_J=J$).
We run our sweeps using model sizes
\[
(d_{\mathrm{model}},F)\in\{(64,512),(96,1152),(128,2048)\}.
\]
and average over four seeds.

\subsubsection{Noisy-Margin Diagnostic}
\label{app:exp:section5:noisy-margin}
The noisy-margin curve in Figure~\ref{fig:ssfr_main_fig}b
measures how the minimum margin of our bilinear construction changes under perturbed key queries.
We use isotropic keys and values with $d=64$, $F=256$, and hidden width $m=512$.
For each stored key, we perturb it via
\[
\widetilde{\mathbf{k}}_i=\mathbf{k}_i+\epsilon \mathbf{u}_i,
\qquad
\mathbf{u}_i\sim\mathrm{Unif}(S^{d-1}),
\]
and sweep $\epsilon$ logarithmically from $0.01$ to $2.0$ over $15$ values.
For each sweep point, we report the empirical noisy minimum margin and the predicted linear degradation term, averaged over three seeds.

\begin{figure}[h]
\centering
\includegraphics[width=0.60\linewidth]{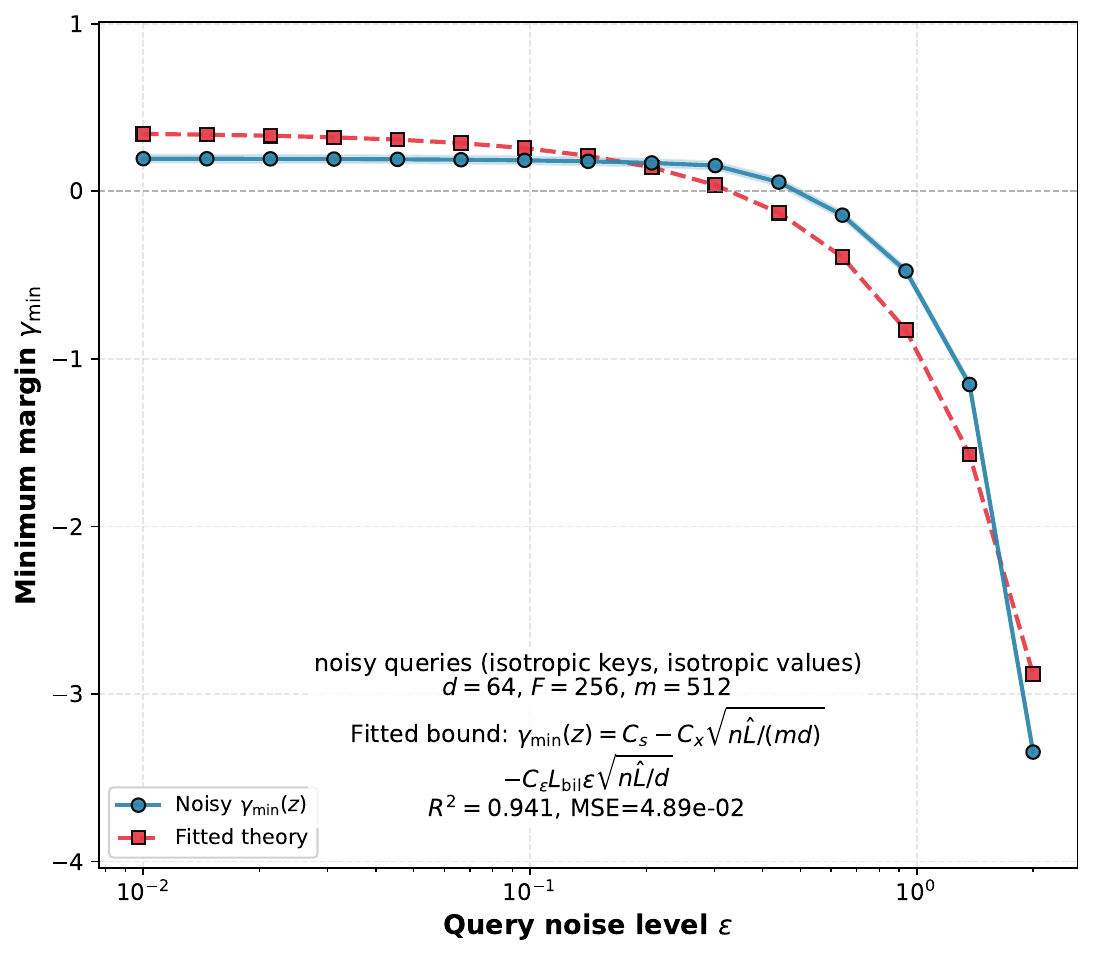}
\caption{\textbf{Margin degradation under noisy queries.}
Bilinear MLP minimum margin decreases linearly with the magnitude of key perturbations $\epsilon$, validating the noisy-margin degradation bound. The empirical margin (solid line) closely follows the predicted linear degradation term (dashed line) with $d=64$, $F=256$, $m=512$.}
\label{fig:app:noisy-margin-sweep}
\end{figure}

\subsubsection{Transformer Capacity Sweep}
\label{appendix:lm-scaling}

Our Transformer capacity plots use the SSFR task from Appendix~\ref{appendix:ssfr-task} and the \emph{inserted MLP} training setup from Appendix~\ref{app:exp:transformer-setups}. In our experiments, we compare GD, our unwhitened construction, our whitened construction, our data-dependent construction, and the NTK baseline. Unless otherwise stated, we use embedding dimension \(d=128\), junk length \(J=9\), junk vocabulary size \(V_J=9\), \(4{,}000\) training epochs, and one seed.

We evaluate Transformer capacity using two complementary success criteria:
\begin{itemize}
    \item \textbf{Training accuracy.} In this variant, we evaluate a Transformer's accuracy on SSFR using the MLP and the fact set it is trained with. This criterion assesses whether the inserted MLP is well-behaved enough that it can be used inside a Transformer at all. Crucially, this criterion is unable to distinguish between when a Transformer is learning to query its fact-storing MLP to perform factual recall and when it is simply using its attention parameters to memorize the fact set.
    
    \item \textbf{Evaluation (fact-adaptive) accuracy.} In this variant, we train a Transformer using an MLP $MLP_A$ storing a given fact set $A$, but during evaluation, we generate a new fact set $B$ and fact-storing MLP, $MLP_B$, that stores it. We then insert $MLP_B$ into the pretrained Transformer, with no additional training, and evaluate the Transformer's end-to-end accuracy on SSFR \emph {using fact set $B$}.
    This criterion assesses whether the Transformer is performing factual recall by learning to query its fact-storing MLP, while being prevented from introducing additional fact-set-dependent computation.
\end{itemize}

In practice, we find that we need to make two changes to the standard GPT-style architecture to ensure that Transformers trained on a given fact set will learn a fact-adaptive solution: (i) we disable the residual connection after the attention layer to prevent extra signal propagation, and (ii) we freeze the value and output projections to identity matrices to hinder memorization.

For the main-text Transformer-capacity plot in Figure~\ref{fig:main-figure}c, we define Transformer capacity as the smallest model that achieves 100\% fact-adaptive accuracy.
We sweep
\[
F\in\{2^6,2^7,\dots,2^{13}\}=\{64,128,256,512,1024,2048,4096,8192\}
\]
and binary-search over hidden width \(m\in[1,65536]\) with precision \(16\) for each method.

\subsubsection{Per-Key Margin Diagnostic}
\label{app:exp:section5:margin-violin}
The violin plot in \Cref{fig:margin-violin} visualizes the full distribution of per-key margins rather than only the minimum margin, as in the other margin sweeps.
For this experiment, we use the \emph{pretrained attention} setting from Appendix~\ref{app:exp:transformer-setups}.
We use $d=128$, $F=2048$, $J=V_J=9$, $16$ logarithmically spaced hidden widths in $[16,256]$, and four seeds.
For each width, we pool the per-key margins across seeds and compare their distribution against the corresponding combined Transformer accuracy and standalone MLP accuracy.

\begin{figure}[h]
\centering
\includegraphics[width=0.62\linewidth]{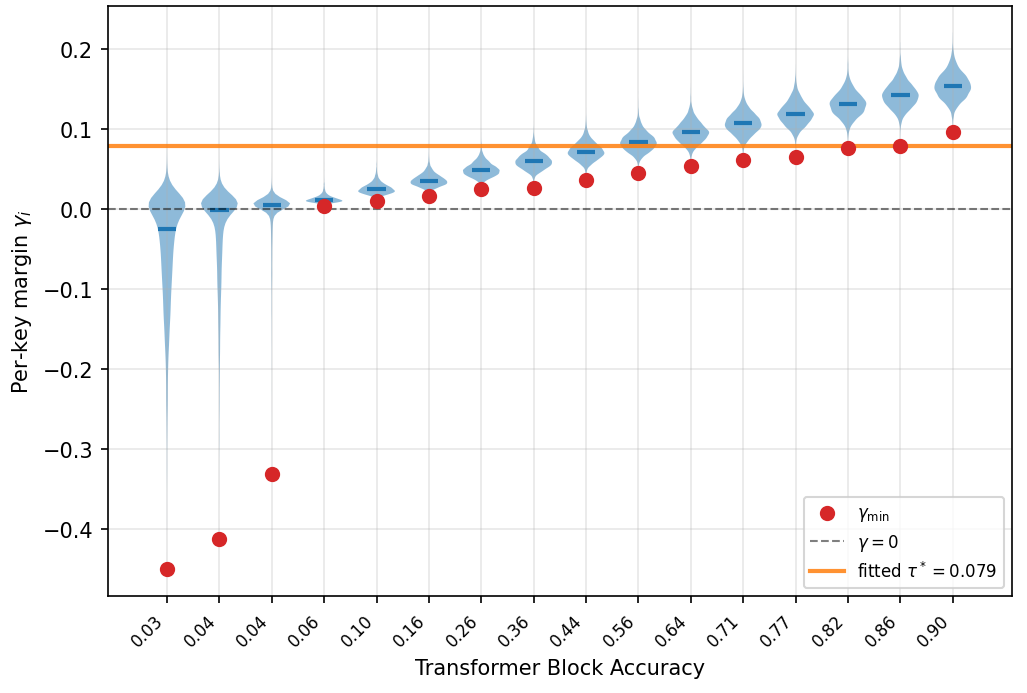}
\caption{\textbf{Per-key margin distributions is predictable of end-to-end Transformer accuracy.}
The usable-key fraction inferred from per-key margins is predictable of the combined
attention+MLP Transformer accuracy across the tested models.}
\label{fig:margin-violin}
\end{figure}

\subsubsection{Fact Editing}
\label{appendix:lm-task}
\label{appendix:fact-editing}

\paragraph{Language Modeling Task.}
We introduce a simple language modeling (LM) task to evaluate a Transformer's ability to perform next-token prediction while recalling factual information. In this task, the model is presented with a natural-language sentence expressing a \((\textit{book}, \textit{author})\) relation and is required to predict each subsequent token in the sequence. We curate this dataset using author-book relations from the Goodreads Book Graph Dataset \citep{authors_dataset}.

Formally, let \(f : S_k \to S_v\) be the authors \textit{fact set}, where
\(S_k = \{\text{``It''},\ \text{``1984''},\ \text{``And Then There Were None''},\ \ldots\}\) is the set of book titles (keys) and
\(S_v = \{\text{``Stephen King''},\ \text{``George Orwell''},\ \text{``Agatha Christie''},\ \ldots\}\) is the set of corresponding authors (values).
To simplify analysis, we select exactly one book per author.
Let \(J = \{(\text{``The author of''},\ \text{``is''}),\ (\text{``Who is the author of''},\ \text{``? It is''}),\ \ldots\}\) denote the set of natural-language template prefix--suffix pairs.
The LM task given \(f\) can then be defined as:
\[
    \mathcal{S}_{LM}[f] = \{\text{concat}(t_{\text{prefix}},\ k,\ t_{\text{suffix}}, f(k)) \ |\ (t_{\text{prefix}}, t_{\text{suffix}}) \in J,\ k \in S_k \}.
\]

For example, given the sequence:
\[
\underbrace{\text{The author of}}_{\text{template prefix}}\ 
\underbrace{1984}_{\text{key}}\ 
\underbrace{\text{is}}_{\text{template suffix}}\ 
\underbrace{\text{George Orwell}}_{\text{value}}
\]
from \(\mathcal{S}_{LM}[f]\), the model’s task is to perform next-token prediction \textit{at every position} in the sentence. This LM task allows us to study factual recall in a more natural language modeling setting, complementing the SSFR setup.

\paragraph{Model Parameterization.}
For this experiment, we use a one-layer Transformer architecture with one head and the following non-standard design choices.
We use this parameterization because it was the most
realistic and simplest Transformer variant in our sweep that was able to achieve near-perfect accuracy on the book--author task when trained with an inserted fact-storing MLP.

The model uses frozen tied input/output embeddings, no attention or
MLP residual connection, RoPE positional encoding, freezes the RMSNorm before the MLP, uses an RMSNorm before attention and the language-modeling head, and sets value and output projections of the attention layer frozen to identity matrices. The model hidden dimension is \(d_{\mathrm{model}}=256\).

The attention layer keeps standard causal softmax attention, but uses learned nonlinear query and key projections. If \(z_t\) is the attention-normalized residual stream, then, omitting biases,
\[
  q_t = W_{Q,2}\,\mathrm{GELU}(W_{Q,1}\,\mathrm{LN}(z_t)),
  \qquad
  k_t = W_{K,2}\,\mathrm{GELU}(W_{K,1}\,\mathrm{LN}(z_t)),
  \qquad
  v_t = z_t,
\]
and the attention output is
\[
  a_t=\sum_{s\le t}
  \mathrm{softmax}_{s}\!\left(\frac{q_t^\top k_s}{\sqrt{d_{\mathrm{model}}}}\right)
  v_s,
\]
with the output projection fixed to the identity.

The feedforward layer is a two-expert module. Let \(x_t\) denote the residual stream entering the feedforward block at token position \(t\), and let \(\tilde{x}_t\) be the normalized MLP input after the block's pre-MLP RMSNorm.
The feedforward output is
\[
  \mathrm{FFN}(x_t,\tilde{x}_t)
  =
  \alpha_t f_{\mathrm{fact}}(\tilde{x}_t)
  +
  (1-\alpha_t) f_{\mathrm{aux}}(x_t),
  \qquad
  \alpha_t=\sigma(g_2(\mathrm{GELU}(g_1(x_t)))).
\]
Here \(f_{\mathrm{fact}}\) is the inserted fact-storing MLP, held fixed during Transformer training, and \(f_{\mathrm{aux}}(x)=xUV\) is a trainable rank-\(8\) low-rank linear expert. The router is a two-layer scalar sigmoid MLP that forms a linear combination of the fact expert and auxiliary expert outputs. Intuitively, our goal is for the model to learn to use the fact-storing MLP to store the book--author relations and the auxiliary expert to learn the natural-language sentence formats.

Book titles and author names are added to the model tokenizer as atomic tokens before training. For the fact MLP, the key embedding representing each book is the normalized version of the corresponding atomic book token.

\paragraph{GD MLP Setup.}
In our fact-editing experiment, we use GD-trained fact-storing MLPs (see Appendix~\ref{app:exp:mlp-variant-definitions} for the bilinear gated architecture and optimizer settings), but we train the MLP with a MSE objective under arg-max decoding:
\[
L_{MLP}(\mathbf{K}, \mathbf{V}, f)
\propto
\sum_{i=1}^{|\mathbf{K}|}
\left\lVert
MLP(\mathbf{k}_i) - \mathbf{v}_{f(i)}
\right\rVert_2^2.
\]
The inserted fact expert is a gated MLP with hidden width \(h=512\),
trained by GD for up to \(10{,}000\) epochs with learning rate \(10^{-3}\),
minimum learning rate \(10^{-6}\), and early stopping once the objective falls
below \(10^{-7}\).
Crucially, we train this MLP with MSE rather than cross entropy because MSE
matches the full author-value embedding vectors, not only the nearest-token
classifier. We find that a parameter-matched cross-entropy GD MLP reaches \(100\%\) fact
classification accuracy when evaluated standalone, but when inserted into the Transformer the model reaches only about \(90\%\) accuracy.

\paragraph{Training Setup.}
We train on \(16{,}384\) book-author facts with \(16\) rephrases per fact. We
initialize embeddings with Kaiming-uniform initialization and train the
Transformer with AdamW using learning rate \(2\times10^{-4}\), weight decay
\(0.1\), batch size \(32\), \(18\) epochs, and \(8192\) optimizer steps per epoch.
The trained base reaches near-perfect standalone MLP accuracy \(0.9998\) and Transformer value-token accuracy \(0.9984\).

\paragraph{Evaluation.}
We divide the \(16{,}384\) stored facts into a \emph{preserved} set whose answers should remain unchanged and an \emph{altered} set whose answers are replaced by a fresh permutation of author values.
We edit \(328\), \(819\), or \(1638\) facts, corresponding to \(2\%\),
\(5\%\), and \(10\%\) of the fact set. For altered facts, we evaluate the
new answer on both the edited training template and held-out rephrases; for
preserved facts, we evaluate the original answer.

We report four fact-editing quantities.
\emph{Efficacy} is author-token accuracy on altered facts under the edited labels.
\emph{Paraphrase} is accuracy on held-out rephrasings of the altered facts under the edited labels.
\emph{Specificity} is accuracy on preserved facts under the original labels.
\emph{Score} is the harmonic mean of efficacy, paraphrase, and specificity.
We additionally report the \emph{non-fact PPL ratio}:
\[
  r_{\mathrm{NF}}
  =
  \frac{\mathrm{PPL}_{\mathrm{post,NF}}}{\mathrm{PPL}_{\mathrm{pre,NF}}}
  =
  \exp\!\left(
    \mathrm{CE}_{\mathrm{post,NF}}-\mathrm{CE}_{\mathrm{pre,NF}}
  \right),
\]
where the CE is averaged over all next-token positions except the author-value tokens, across preserved and altered prompts before and after editing.
For example, \(r_{\mathrm{NF}}=1\) means the edit leaves non-fact-token
language-modeling loss unchanged.

\paragraph{Baselines.}
We compare four editing methods:
\begin{itemize}
    \item Our method, \textsc{MLP Swapping}, constructs a replacement GD fact expert for the complete post-edit fact set and swaps that expert into the frozen Transformer, with no update to the surrounding Transformer.
    \item MEMIT~\citep{memit} applies a multi-edit linear residual update to the fact expert's output projection, using altered-fact keys and a regularized solve.
    \item ROME~\citep{rome} applies a rank-one residual update for each edit, adapted to our one-layer setting and applied at the final prompt token.
    \item AlphaEdit~\citep{fang2025alphaeditnullspaceconstrainedknowledge} first estimates a preserve-key subspace and then projects the residual edit into the approximate nullspace of that preserve subspace before updating the fact expert.
\end{itemize}
For the weight-update editors, residuals are computed from a single templated prompt per fact and applied to the inserted MLP output immediately upstream of the logits. We omit random prefix contexts and the ROME KL term because this synthetic dataset has a single relation and a unique author per book.

For each weight-update baseline and edit fraction, we sweep method-specific
hyperparameters and report the score-best configuration. AlphaEdit uses the best targeted grid setting \(\texttt{train\_steps}=100\), \(\texttt{lr}=0.05\), \(\texttt{clip\_norm}=\text{None}\), and \(\texttt{singular\_value\_tolerance}=10\) at all three edit fractions.
MEMIT uses \(\texttt{lr}=0.05\), \(\lambda=150\), and \(\texttt{clip\_norm}=0.75\), with \(100\) steps at \(2\%\) and \(5\%\) edits and \(25\) steps at \(10\%\) edits.
ROME uses \(\texttt{lr}=0.05\), \(\texttt{wd}=1.5\times10^{-3}\), and \(\texttt{early\_stopping\_loss}=5\times10^{-2}\), with \(10\) steps at \(2\%\) edits and \(100\) steps at \(5\%\) and \(10\%\) edits.

\begin{table}[h]
\centering
\scriptsize
\setlength{\tabcolsep}{5pt}
\begin{tabular}{lccccc}
\toprule
Method & Efficacy & Paraphrase & Specificity & Score & \(r_{\mathrm{NF}}\) \\
\midrule
\multicolumn{6}{l}{\(2.00\%\) edited facts} \\
\quad \textsc{MLP Swapping} & \textbf{1.000} & \textbf{1.000} & \textbf{0.999} & \textbf{1.000} & \emph{1.016} \\
\quad AlphaEdit & \emph{0.896} & \emph{0.899} & 0.981 & \emph{0.924} & \textbf{1.000} \\
\quad MEMIT & 0.643 & 0.645 & \emph{0.990} & 0.729 & \textbf{1.000} \\
\quad ROME & \textbf{1.000} & \textbf{1.000} & 0.009 & 0.026 & 1.238 \\
\addlinespace
\multicolumn{6}{l}{\(5.00\%\) edited facts} \\
\quad \textsc{MLP Swapping} & \textbf{0.998} & \textbf{0.998} & \textbf{0.998} & \textbf{0.998} & \emph{1.057} \\
\quad AlphaEdit & 0.747 & 0.745 & 0.909 & \emph{0.794} & \textbf{1.001} \\
\quad MEMIT & 0.128 & 0.126 & \emph{0.986} & 0.179 & \textbf{1.001} \\
\quad ROME & \emph{0.991} & \emph{0.991} & 0.002 & 0.006 & 1.642 \\
\addlinespace
\multicolumn{6}{l}{\(10.00\%\) edited facts} \\
\quad \textsc{MLP Swapping} & \textbf{0.998} & \textbf{0.999} & \textbf{0.999} & \textbf{0.999} & 1.021 \\
\quad AlphaEdit & 0.482 & 0.482 & 0.766 & \emph{0.550} & \emph{1.010} \\
\quad MEMIT & 0.004 & 0.003 & \emph{0.995} & 0.005 & \textbf{1.002} \\
\quad ROME & \emph{0.834} & \emph{0.832} & 0.001 & 0.003 & 2.405 \\
\bottomrule
\end{tabular}
\caption{Fact-editing metrics for baselines and \textsc{MLP Swapping} on the GD MLP base model.
\(r_{\mathrm{NF}}\) is the non-fact-token perplexity ratio.
Bold marks the best value and italics mark the second-best value within each edit fraction and metric; lower is better
for \(r_{\mathrm{NF}}\), and higher is better for all other metrics.
Ties at the displayed precision share the same marking.}
\label{tab:fact-editing-updated-h512}
\end{table}

\begin{figure}[h]
\centering
\includegraphics[width=0.68\linewidth]{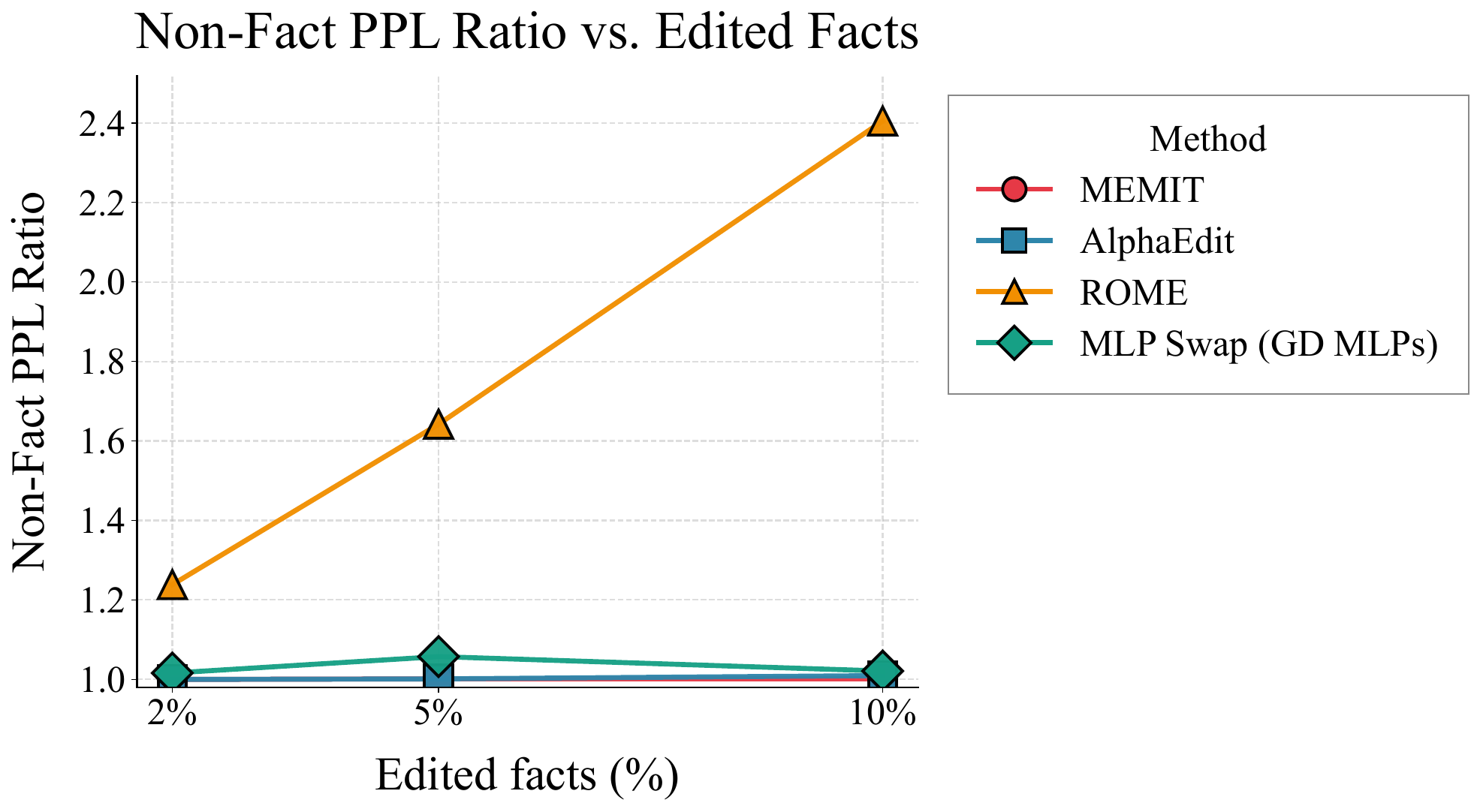}
\caption{Non-fact-token perplexity ratio for the
fact-editing setup of Figure~\ref{fig:ssfr_main_fig}c. AlphaEdit and MEMIT have
nearly unchanged non-fact loss but weaker edit scores at larger edit fractions;
ROME edits target prompts while severely damaging non-fact loss and specificity.
\textsc{MLP Swapping} keeps near-perfect edit scores with a small non-fact PPL
increase, at most \(1.06\times\).}
\label{fig:fact-editing-nonfact-ppl-ratio}
\end{figure}

\paragraph{Fact editing with constructed MLPs.}
We also repeat the fact-editing experiment from Figure~\ref{fig:ssfr_main_fig}c using the \emph{data-dependent Hebbian MLP construction}. Because the data-dependent construction has slightly worse storage capacity scaling than GD MLPs, we find we need to increase the hidden dimension to \(h=1024\) to achieve near-perfect MLP and Transformer accuracy. Other than this change, our experimental setup is the same as described above for GD MLPs.
We retune the method-specific hyperparameters for each of the fact-editing baselines.
We show that \textsc{MLP Swapping} remains effective even with constructed Hebbian MLPs in Figure~\ref{fig:fact-editing-constructed-data-dependent}.
Swapping out a \(h=1024\) data-dependent constructed MLP keeps the edit score above \(0.98\) through up to \(10\%\) edited facts, while the strongest tuned local-editing baseline reaches only \(0.847\) at \(10\%\). The \textsc{MLP Swapping} non-fact PPL ratio remains below \(1.11\times\).

\begin{figure}[h]
\centering
\includegraphics[width=0.68\linewidth]{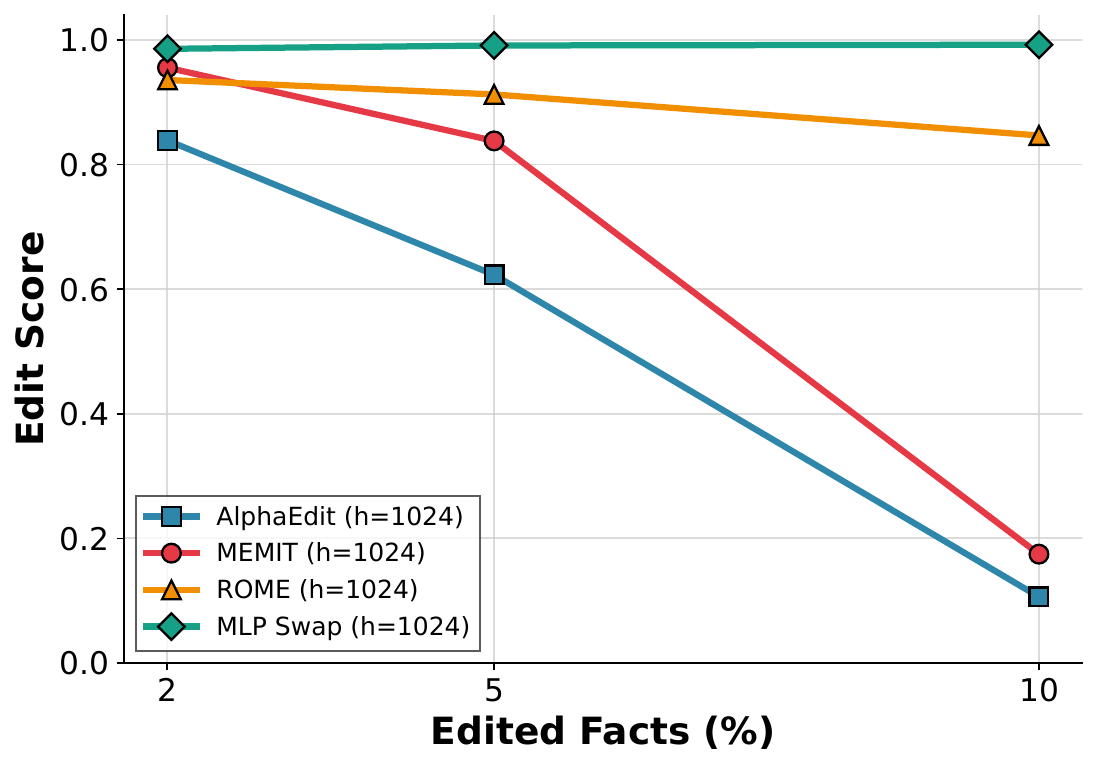}\\[0.5em]
\includegraphics[width=0.68\linewidth]{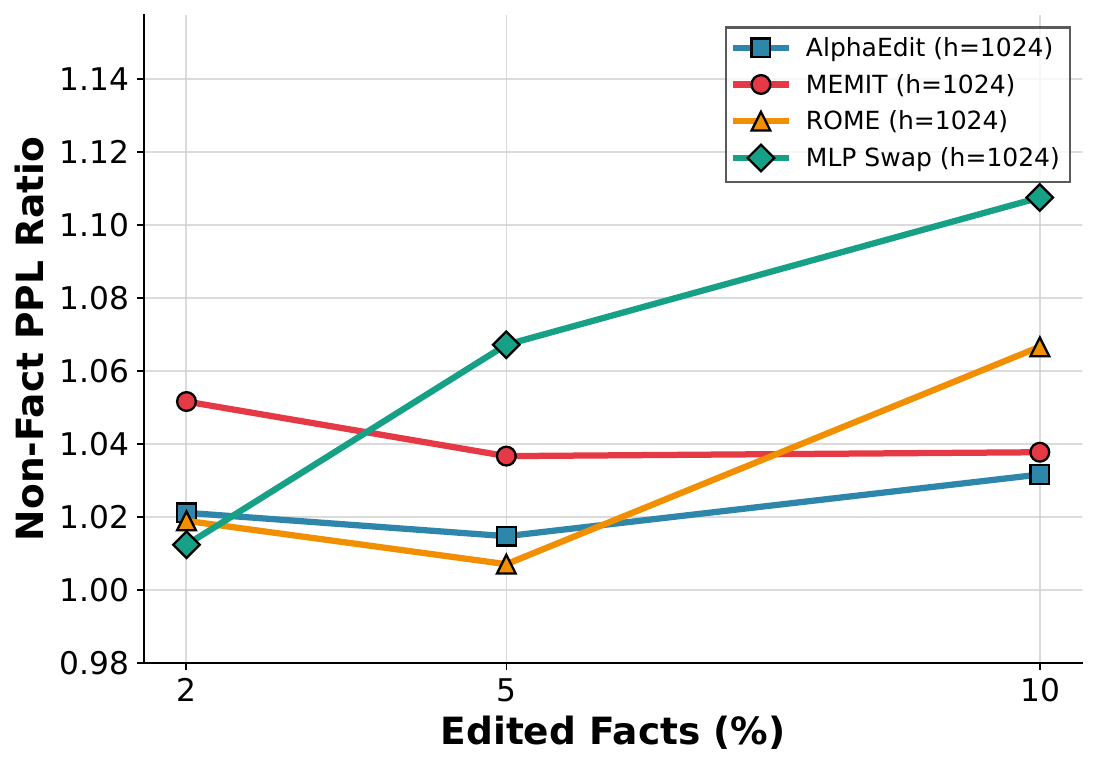}
\caption{Fact-editing score (top) and non-fact-token perplexity ratio (bottom) for the \(h=1024\) data-dependent constructed-MLP setting. \textsc{MLP Swapping} achieves an edit score above \(0.98\) through up to \(10\%\) edited facts---13 percentage points higher than the next highest baseline---all while non-fact PPL ratio remains below \(1.11\times\).}
\label{fig:fact-editing-constructed-data-dependent}
\end{figure}

%% file: appendix/experiments/additional_empirical_results.tex
\subsubsection{Training-Accuracy Transformer Capacity}
\label{app:exp:additional:train99-capacity}

In~\Cref{fig:appendix:transformer-capacity-train99}, we rerun the main-text Transformer capacity experiment from Figure~\ref{fig:main-figure}c, but with a 99\% \emph{training accuracy} criterion instead of fact-adaptive accuracy (as defined in \Cref{appendix:lm-scaling}).
Unlike the evaluation accuracy plot, we fix the hidden dimension \(m\in\{44,88,176,352,704,1408\}\) and binary-search for the maximum number of facts \(F \in [1, 65536] \) that the Transformer is able to store.

\begin{figure}[t]
\centering
\includegraphics[width=0.78\linewidth]{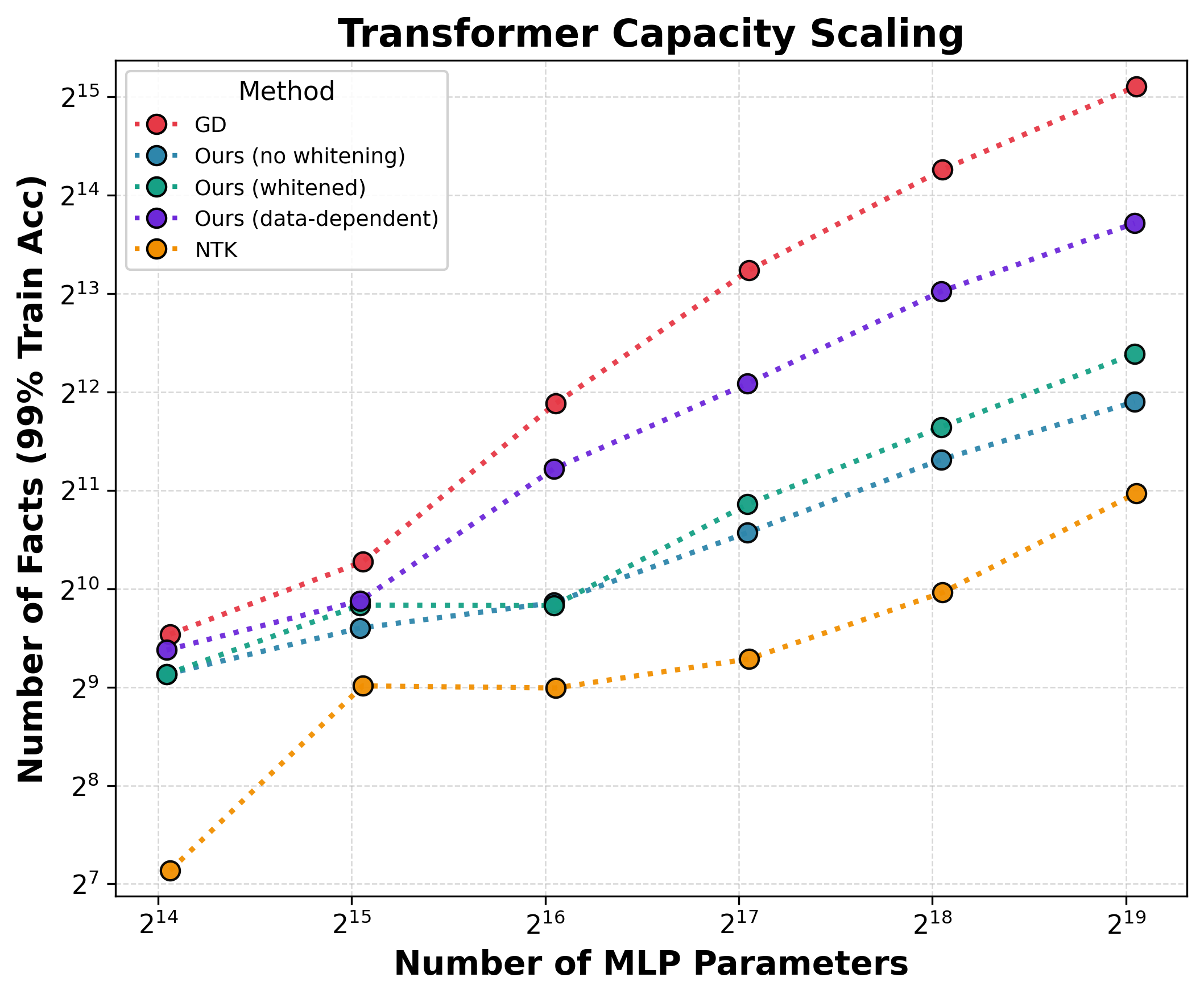}
\caption{\textbf{Transformer fact-storage capacity under a 99\% training accuracy criterion.}
Transformer train-accuracy capacity is $2$-$8\times$ higher than fact-adaptive capacity for our constructions and fixes the asymptotics for NTK -- this suggests that attention parameters are learning to participate in storing the fact set.}
\label{fig:appendix:transformer-capacity-train99}
\end{figure}

Relative to the fact-adaptive frontier in Figure~\ref{fig:main-figure}c, the train-side plot stores roughly \(2\)–\(8\times\) more facts at a fixed parameter budget for GD MLPs and for all of our construction variants. The gaps close the most for the weakest constructions, especially NTK and our unwhitened method; in particular, the suboptimal fixed-\(d\) asymptotics of the NTK construction are avoided under the train-accuracy criterion.
This suggests that without the stricter fact-adaptive criterion, the surrounding attention layer uses its parameters to help store the fact set, rather than solely relying upon the inserted fact-storing MLP.

\subsubsection{Anisotropic MLP Capacity}
\label{app:exp:additional:anisotropic-mlp-capacity}

\Cref{fig:appendix:anisotropic-mlp-capacity} uses the same MLP capacity protocol as Appendix~\ref{app:exp:section4:capacity}, but evaluates the different MLP families on \emph{anisotropic} keys and values.
We apply the rank-1 spike model from Appendix~\ref{app:exp:section4:anisotropic-keys-values} with \(\beta=1.5\).
We then estimate the resulting fact-storage frontier for \(d\in\{64,90,128\}\).

\begin{figure}[t]
\centering
\includegraphics[width=0.86\linewidth]{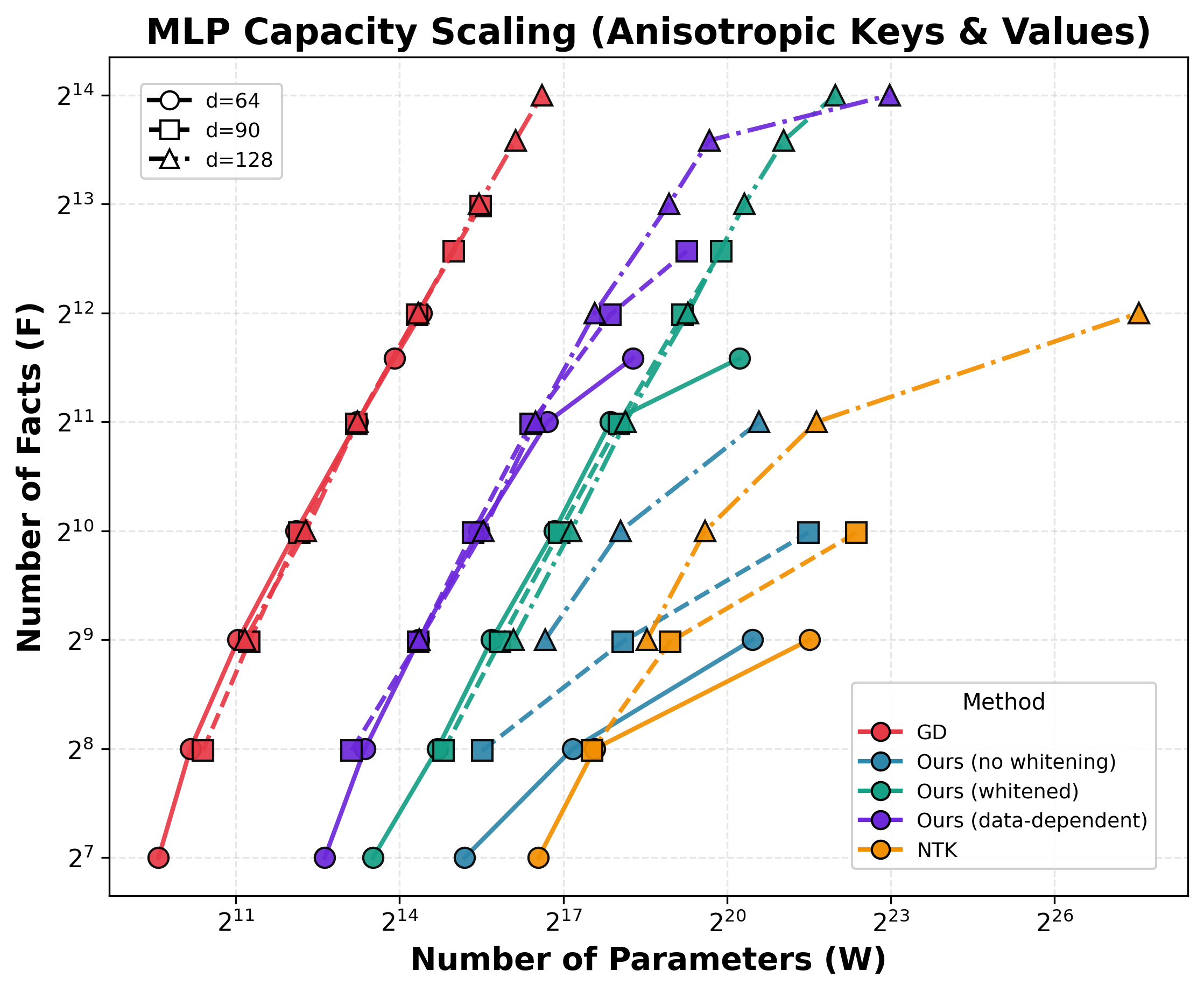}
\caption{\textbf{MLP fact-storage capacity under anisotropic keys and values.}
Our data-dependent construction achieves asymptotically optimal fact-storage capacity (only $4$-$8\times$ worse than GD), while NTK fails to achieve the desired \(W=\Theta(F\log F)\) scaling under anisotropy.}
\label{fig:appendix:anisotropic-mlp-capacity}
\end{figure}

Although anisotropic embeddings degrade the attainable fact-storage capacity for all methods relative to the isotropic setting (Figure~\ref{fig:sketched_k2_sweeps}c), our data-dependent construction remains only \(4\)–\(8\times\) worse than GD, as in the isotropic case.
On the other hand, the NTK and unwhitened constructions struggle to store facts at the \(W=\Theta(F\log F)\) scaling once both keys and values are anisotropic. We find that whitening improves fact-storage capacity substantially more in the anisotropic case.

\subsubsection{Margin and Capacity Scaling Under LLM Embeddings}
\label{app:exp:additional:llm-embeddings}

In this section, we explore the margin scaling and MLP and Transformer Block capacity scaling when using our construction on embeddings sampled from an intermediate LM layer. Concretely, we replace the synthetic spherical key/value embeddings used in Sections~\ref{sec:4}~and~\ref{sec:5} with paired \((\mathbf{x}, \mathbf{y})\) embeddings captured from a real language model. The intent is to check that the margin bounds and the predicted storage-capacity scaling remain meaningful when the keys and values come from those that an intermediate layer of an LM presents to its MLP blocks.

\paragraph{LM embeddings.} We stream the train split of WikiText through Qwen3-0.6B-Base and, at the middle decoder block (layer 14 of 28, the ``mid layer''), record per-token MLP inputs \(\mathbf{x}\) (post-RMSNorm hidden state) and outputs \(\mathbf{y}\) (pre-residual), keeping \(N=500{,}000\) pairs. A factset of size \(F\) is then built by uniformly sampling \(F\) of these \((\mathbf{x}_i, \mathbf{y}_i)\) pairs, taking \(\mathbf{x}_i\) as the input and \(\mathbf{y}_i\) as the output embedding under the identity mapping.

\paragraph{Margin scaling.} \Cref{fig:appendix:llm-margin} shows the Section~\ref{sec:4} margin scaling sweep on the captured Qwen3 mid-layer embeddings. We construct our bilinear random-feature Hebbian variant with \(d=M=1024\), sweeping \(F\) from \(4\) to \(128\) over five seeds per point. In both the arbitrary-keys/values and the isotropic regimes, the fitted bound tracks the empirical minimum margin \(\gamma_{\min}\) closely (\(R^2\geq0.95\)), suggesting that our margin scaling laws continue to hold on real, anisotropic LM embeddings.

\begin{figure}[t]
\centering
\includegraphics[width=0.48\linewidth]{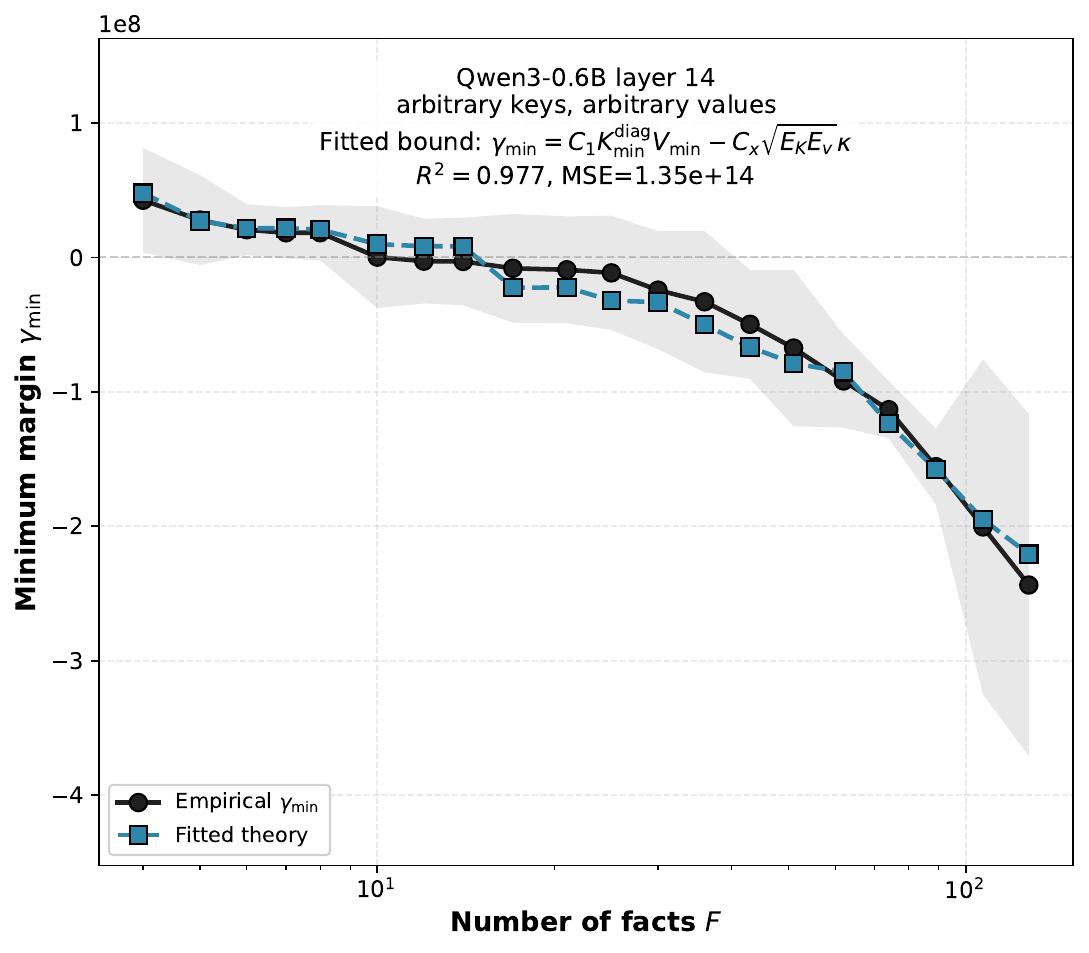}\hfill
\includegraphics[width=0.48\linewidth]{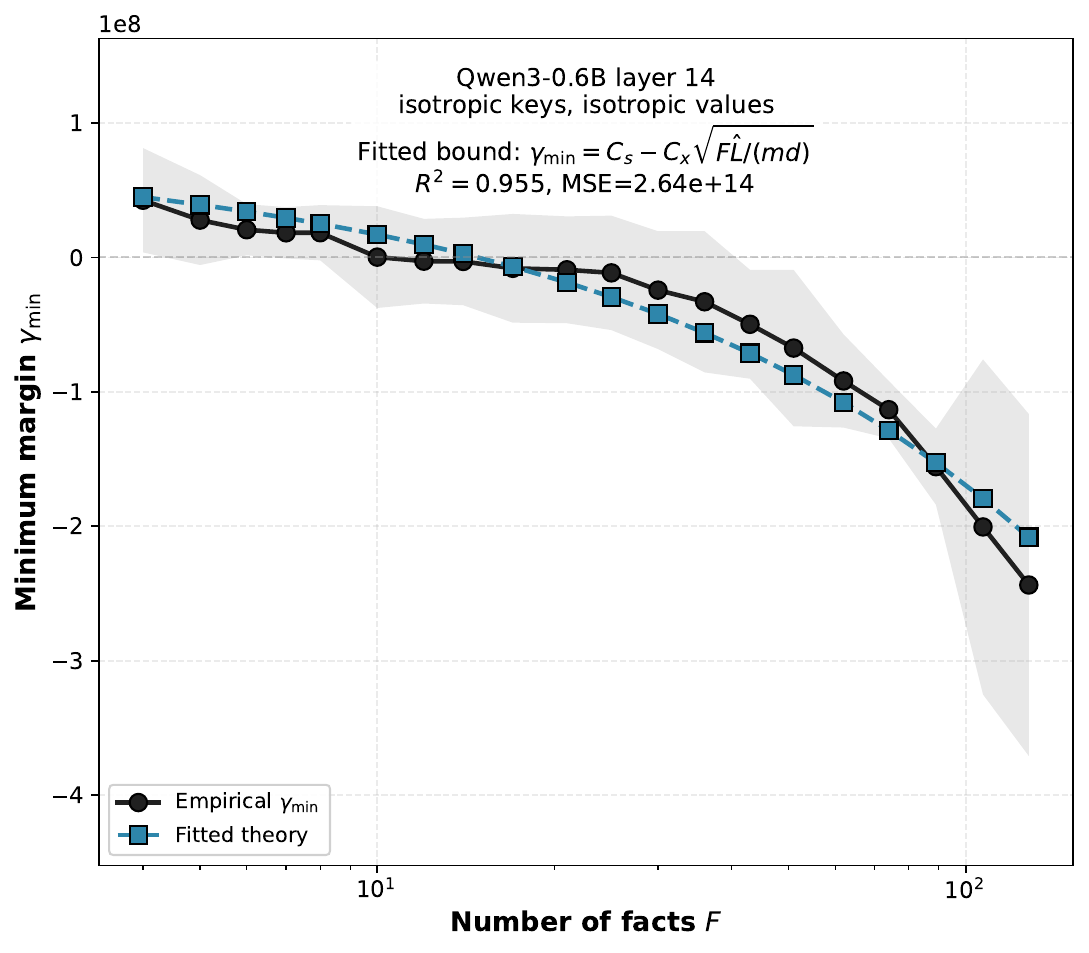}
\caption{\textbf{Bilinear-RF margin scaling on Qwen3-0.6B mid-layer embeddings.}
Left: arbitrary keys and values, where the measured margin tracks the deterministic bound through the key/value-geometry terms $K_{\min}^{\mathrm{diag}}$, $E_K$, $E_v$ and the composite cross-talk scale $\sqrt{E_K}\sqrt{E_v}\,\kappa$. Right: the isotropic-key / isotropic-value bound, where the margin follows the single cross-talk term $\sqrt{FL/(md)}$. In both regimes the bound tracks the empirically measured minimum margin $\gamma_{\min}$ ($R^2\geq0.95$), suggesting that our margin scaling laws hold for LM embeddings.}
\label{fig:appendix:llm-margin}
\end{figure}

\paragraph{MLP storage capacity.} \Cref{fig:appendix:llm-mlp-capacity} reproduces the standalone-MLP fact-storage capacity scaling of Section~\ref{sec:4}, but with the random key/value embeddings replaced by the captured Qwen3 mid-layer embeddings. Following the protocol in Appendix~\ref{app:exp:section4:capacity}, we binary-search over MLP hidden width \(W\) for the minimum width that achieves a \(98\%\) per-fact recall threshold using our whitened bilinear-RF Hebbian construction for \(F \in \{2^9, 2^{10}, \dots, 2^{14}\}\). Notably, the facts \(F\) vs.\ parameters \(W\) scaling follows the predicted \(W \approx \Theta(F\log F)\) capacity scaling on LM embeddings.

\begin{figure}[t]
\centering
\includegraphics[width=0.62\linewidth]{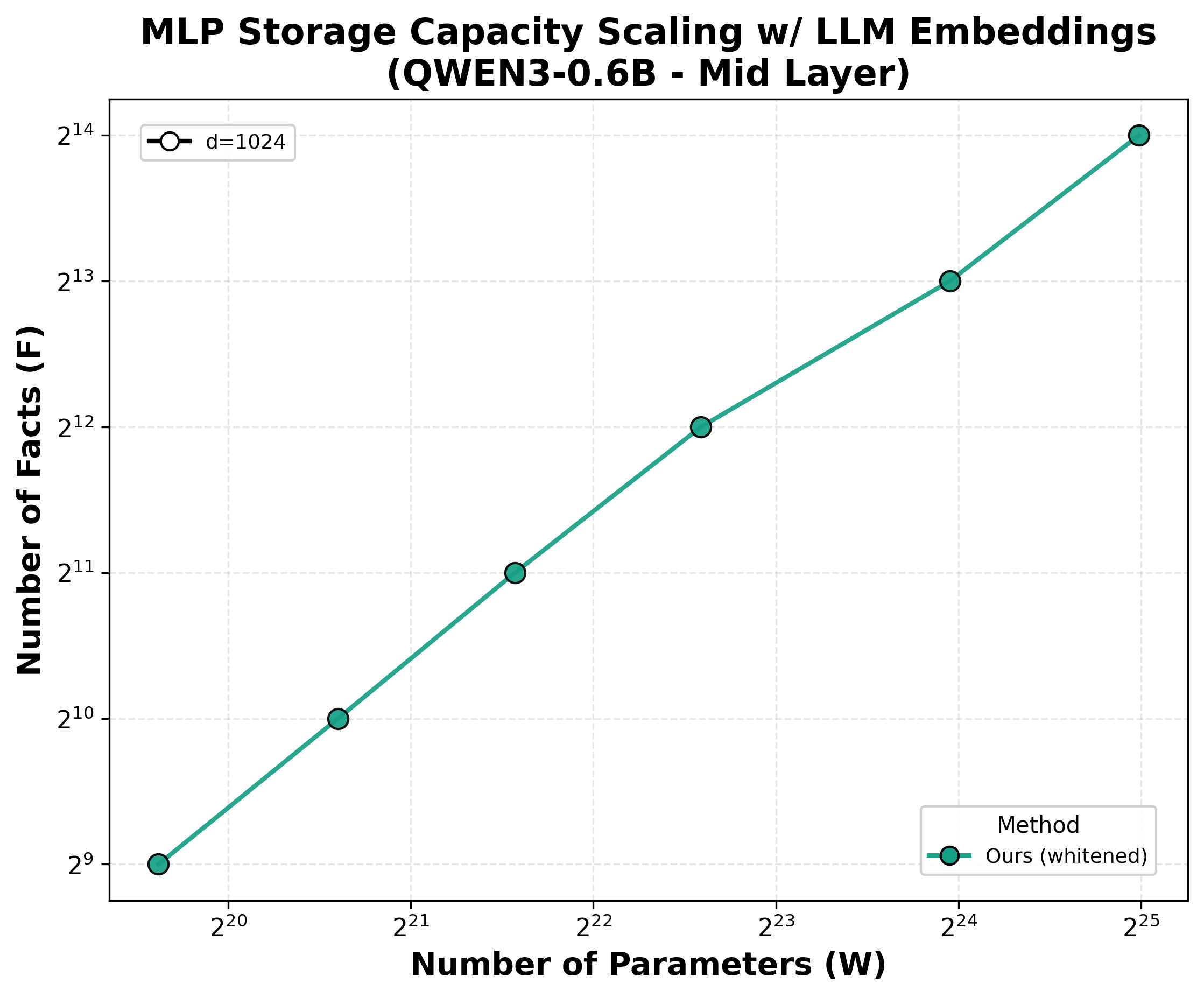}
\caption{\textbf{Standalone-MLP fact-storage capacity on Qwen3-0.6B mid-layer embeddings.}
Our whitened bilinear-RF construction achieves the expected \(W \approx \Theta(F\log F)\) capacity scaling on LM embeddings.}
\label{fig:appendix:llm-mlp-capacity}
\end{figure}

\paragraph{Transformer block storage capacity.} For the Transformer block (\Cref{fig:appendix:llm-transformer-capacity}), we use the same insert-then-train-attention hidden-width capacity sweep as the main-text Figure~\ref{fig:main-figure}c: a frozen whitened-construction MLP is inserted into a single Transformer block whose input and output embeddings are the captured Qwen3 \(\mathbf{x}\) and \(\mathbf{y}\), and we binary-search for the smallest hidden width reaching the success threshold. Here, we require \(100\%\) SSFR \emph{training} accuracy on the trained fact set rather than fact-adaptive evaluation accuracy. Moreover, we keep the post-attention residual (rather than disabling it), leave the value/output projections trainable (rather than freezing them to identity), and evaluate on the same inserted MLP (rather than swapping in an eval-MLP for a held-out fact set). Notably, as in the MLP case, the facts \(F\) vs.\ parameters \(W\) scaling follows the predicted \(W \approx \Theta(F\log F)\) capacity scaling on LM embeddings.

\begin{figure}[t]
\centering
\includegraphics[width=0.62\linewidth]{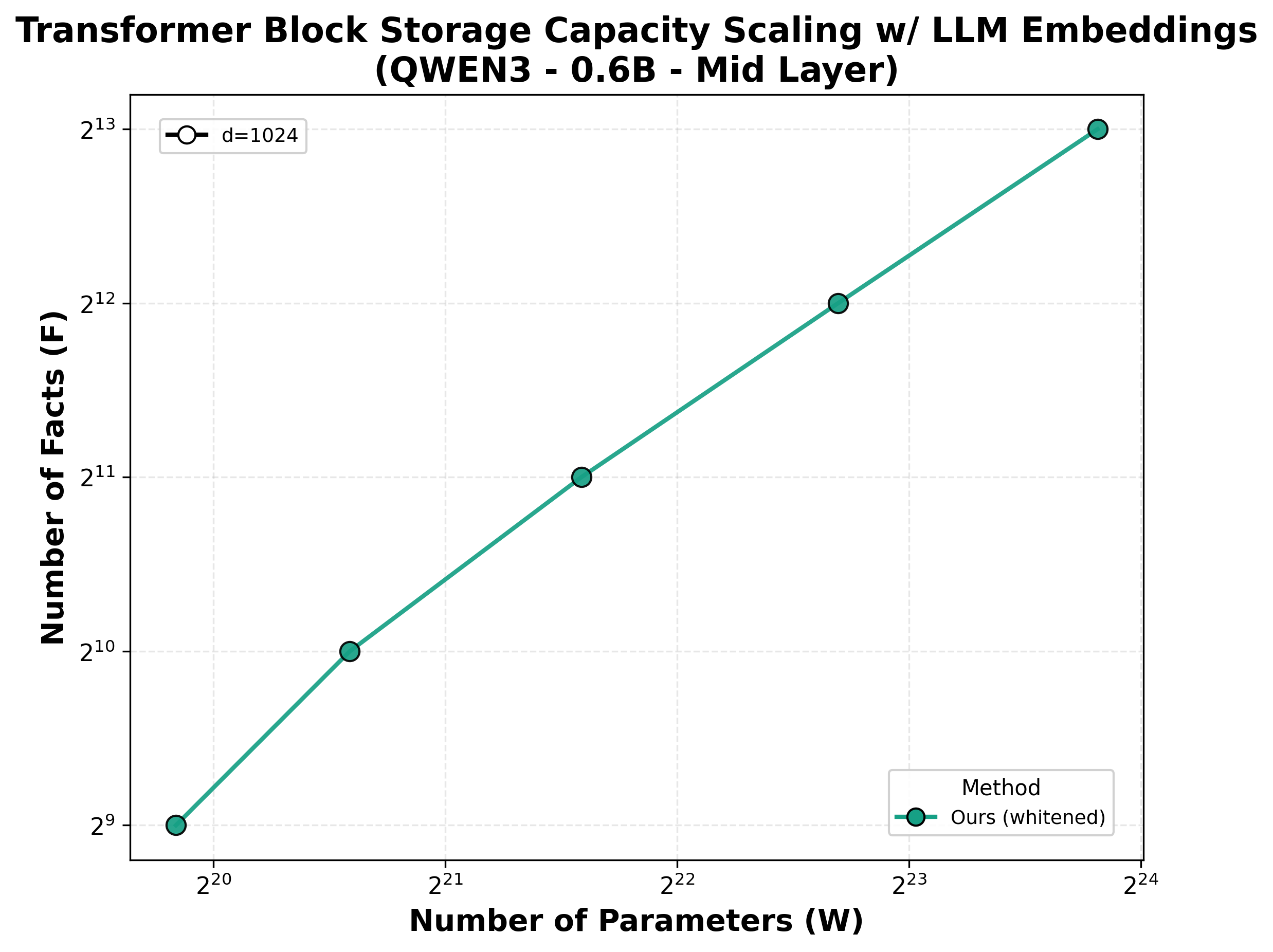}
\caption{\textbf{Transformer-block fact-storage capacity on Qwen3-0.6B mid-layer embeddings (training-accuracy criterion).}
Our whitened bilinear-RF construction, inserted into a full Transformer block, retains the expected \(W \approx \Theta(F\log F)\) scaling on LM embeddings.}
\label{fig:appendix:llm-transformer-capacity}
\end{figure}

%% file: appendix/theory.tex
\subsection{Information-Theoretic Lower Bound on Fact Storage}
\label{sec:info_theory_bound}
\label{app:theory:preliminaries}
\input{appendix/theory/00_fact_storage_lower_bound}

\subsection{Hebbian MLP Construction}
\label{app:theory:hebbian_mlps}
\input{appendix/theory/02_hebbian_mlps}

\subsection{NTK Baseline}
\label{app:theory:hebbian_mlps:ntk_construction}
\input{appendix/theory/02_hebbian_mlps/04_ntk_construction}

\input{appendix/theory/03_margin_bounds_full}

\newpage

\subsection{Hebbian MLPs within Transformers}
\label{app:theory:transformers_perturbation}
\input{appendix/theory/04_transformers_perturbation}

%% file: appendix/theory/00_fact_storage_lower_bound.tex
We prove the counting lower bound used in
\Cref{thm: info_bounds_capacity-const}.

\begin{proof}[Proof of \Cref{thm: info_bounds_capacity-const}]
Let $b$ be the constant number of bits used to store each trainable parameter.
A model with $W$ trainable parameters has at most $2^{bW}$ distinct parameter
settings, and therefore can realize at most $2^{bW}$ distinct input-output
behaviors on the fixed key set $\mathbf{K}$.

On the other hand, the number of possible fact sets on $\mathbf{K}$ and
$\mathbf{V}$ is
\[
    \left|\{f : [|\mathbf{K}|] \to [|\mathbf{V}|]\}\right|
    =
    |\mathbf{V}|^{|\mathbf{K}|},
\]
since each key can be assigned any value independently.

If a model class $\mathbf{g}$ stores every fact set using $W$ parameters, then
the number of realizable behaviors must be at least the number of fact sets:
\[
    2^{bW}
    \ge
    |\mathbf{V}|^{|\mathbf{K}|}.
\]
Taking logarithms gives
\[
    bW
    \ge
    |\mathbf{K}|\log_{2}|\mathbf{V}|.
\]
Because $b$ is constant,
\[
    W
    \ge
    \frac{|\mathbf{K}|\log_{2}|\mathbf{V}|}{b}
    =
    \Omega\!\left(|\mathbf{K}|\log |\mathbf{V}|\right).
\]
\end{proof}

%% file: appendix/theory/02_hebbian_mlps.tex
\subsubsection{MLPs Are Hebbians}
\label{app:theory:hebbian_mlps:setup}
\input{appendix/theory/02_hebbian_mlps/01_setup}

\subsubsection{Bilinear MLP Featurization Induces the $K_{2}$ Kernel}
\label{app:theory:hebbian_mlps:mlp_kernels}
\input{appendix/theory/02_hebbian_mlps/02_mlp_kernels}

\subsubsection{Sketched-$K_{2}$ Construction}
\label{app:hebbian-mlps-construction}
\input{appendix/theory/02_hebbian_mlps/03_construction}

\subsubsection{Kernel Whitening}
\label{app:theory:margin_bounds:kernel_whitening}
\input{appendix/theory/03_margin_bounds/06_kernel_whitening}

\subsubsection{Data-Dependent Construction}
\label{app:theory:margin_bounds:data_dependent_construction}
\input{appendix/theory/03_margin_bounds/07_data_dependent_construction}

\subsubsection{Bit Complexity}
\label{app:theory:margin_bounds:bit_complexity}
\input{appendix/theory/03_margin_bounds/07_bit_complexity}

%% file: appendix/theory/02_hebbian_mlps/01_setup.tex
\phantomsection\label{app:theory:proof:mlp-hebb-whiten}
We formally restate and prove~\Cref{thm:mlp-hebb-whiten}.

\begin{theorem}[MLPs are Hebbians]
Fix a feature map $\phi:\R^d\to\R^m$ and $\mathbf{B}\in\R^{d_{v}\times m}$, and define
the MLP
\[
\mathrm{MLP}(\mathbf{x}):=\mathbf{B}\phi(\mathbf{x}).
\]
For the gated architecture used in this paper, the feature map is
\[
\phi(\mathbf{x})=(\mathbf{A}\mathbf{x})\odot \sigma(\mathbf{G}\mathbf{x}).
\]
Given stored inputs $\mathbf{x}_{1},\dots,\mathbf{x}_{F}$, let $\mathbf{y}_{i}:=\mathrm{MLP}(\mathbf{x}_{i})$ and define the
empirical feature covariance
\[
\hat\Sigma:=\frac1F\sum_{i=1}^F \phi(\mathbf{x}_{i})\phi(\mathbf{x}_{i})^\top.
\]
Assume $\hat\Sigma$ is invertible, and define the whitened kernel
\[
K(\mathbf{x},\mathbf{z}):=\phi(\mathbf{x})^\top\hat\Sigma^{-1}\phi(\mathbf{z})
\]
and the corresponding whitened kernel Hebbian memory
\[
H_{\mathrm{white}}(\mathbf{z}):=\frac1F\sum_{i=1}^F \mathbf{y}_{i}\,K(\mathbf{x}_{i},\mathbf{z}).
\]
Then
\[
H_{\mathrm{white}}(\mathbf{z})=\mathrm{MLP}(\mathbf{z})\qquad\text{for all } \mathbf{z}\in\R^d.
\]
In other words, the MLP is exactly a kernel Hebbian memory
with whitened kernel $K$.
\end{theorem}

\begin{proof}
Let
\[
\hat{\mathbf{W}}:=\frac1F\sum_{i=1}^F \mathbf{y}_{i}\phi(\mathbf{x}_{i})^\top
\]
denote the Hebbian weight matrix on the stored examples. Since
$\mathbf{y}_{i}=\mathrm{MLP}(\mathbf{x}_{i})=\mathbf{B}\phi(\mathbf{x}_{i})$, we have
\[
\hat{\mathbf{W}}
=
\frac1F\sum_{i=1}^F \mathbf{B}\phi(\mathbf{x}_{i})\phi(\mathbf{x}_{i})^\top
=
\mathbf{B}\left(\frac1F\sum_{i=1}^F \phi(\mathbf{x}_{i})\phi(\mathbf{x}_{i})^\top\right)
=
\mathbf{B}\hat\Sigma.
\]

Now expand the whitened kernel Hebbian memory:
\begin{align*}
H_{\mathrm{white}}(\mathbf{z})
&=
\frac1F\sum_{i=1}^F \mathbf{y}_{i}\,K(\mathbf{x}_{i},\mathbf{z}) \\
&=
\frac1F\sum_{i=1}^F \mathbf{y}_{i}\,\phi(\mathbf{x}_{i})^\top\hat\Sigma^{-1}\phi(\mathbf{z}) \\
&=
\left(\frac1F\sum_{i=1}^F \mathbf{y}_{i}\phi(\mathbf{x}_{i})^\top\right)\hat\Sigma^{-1}\phi(\mathbf{z}) \\
&=
\hat{\mathbf{W}}\hat\Sigma^{-1}\phi(\mathbf{z}) \\
&=
\mathbf{B}\hat\Sigma\hat\Sigma^{-1}\phi(\mathbf{z}) \\
&=
\mathbf{B}\phi(\mathbf{z}) \\
&=
\mathrm{MLP}(\mathbf{z}).
\end{align*}
Therefore $H_{\mathrm{white}}(\mathbf{z})=\mathrm{MLP}(\mathbf{z})$ for all $\mathbf{z}\in\R^d$, so the
MLP is exactly a kernel Hebbian memory with whitened kernel $K$.
\end{proof}

%% file: appendix/theory/02_hebbian_mlps/02_mlp_kernels.tex
We first prove that our bilinear MLP construction converges to a Hebbian MLP with quadratic ($K_{2}$) kernel when the number of random features $m$ scales as $\Theta(d^2)$:

\begin{lemma}[Bilinear MLP featurization induces sketched $K_{2}$ kernel.] \label{lem:bilinear-sketched-k2}
Let rows $(\mathbf{A}_{r},\mathbf{G}_{r})_{r=1}^m$ of matrices $\mathbf{A},\mathbf{G}$ be i.i.d.\ standard Gaussian vectors in $\R^d$, and define the
bilinear feature map
\[
g(\mathbf{x})
\defeq
\frac{1}{\sqrt m}\big((\mathbf{A}_{r}^\top \mathbf{x})(\mathbf{G}_{r}^\top \mathbf{x})\big)_{r=1}^m
\in \R^m.
\]
Then the kernel induced by this feature map is
\[
\hat K(\mathbf{x},\mathbf{z})
\defeq
\ip{g(\mathbf{x})}{g(\mathbf{z})}
=
\frac{1}{m}\sum_{r=1}^m
(\mathbf{A}_{r}^\top \mathbf{x})(\mathbf{A}_{r}^\top \mathbf{z})\,
(\mathbf{G}_{r}^\top \mathbf{x})(\mathbf{G}_{r}^\top \mathbf{z}).
\]
and satisfies
\[
\E[\hat K(\mathbf{x},\mathbf{z})] = \ip{\mathbf{x}}{\mathbf{z}}^2 = K_{2}(\mathbf{x},\mathbf{z}),
\]
so $\hat K$ is an unbiased random-feature sketch of the exact quadratic kernel $K_{2}$.
\end{lemma}

\begin{proof}
The kernel identity follows by expanding the inner product of the feature vectors:
\[
\ip{g(\mathbf{x})}{g(\mathbf{z})}
=
\frac{1}{m}\sum_{r=1}^m
(\mathbf{A}_{r}^\top \mathbf{x})(\mathbf{G}_{r}^\top \mathbf{x})(\mathbf{A}_{r}^\top \mathbf{z})(\mathbf{G}_{r}^\top \mathbf{z})
=
\frac{1}{m}\sum_{r=1}^m
(\mathbf{A}_{r}^\top \mathbf{x})(\mathbf{A}_{r}^\top \mathbf{z})\,
(\mathbf{G}_{r}^\top \mathbf{x})(\mathbf{G}_{r}^\top \mathbf{z}).
\]

For the expectation, the pairs $(\mathbf{A}_{r},\mathbf{G}_{r})$ are i.i.d., so it suffices to compute one summand.
Using independence of $\mathbf{A}_{r}$ and $\mathbf{G}_{r}$ and the Gaussian covariance identity,
\[
\E[(\mathbf{A}_{r}^\top \mathbf{x})(\mathbf{A}_{r}^\top \mathbf{z})] = \ip{\mathbf{x}}{\mathbf{z}},
\qquad
\E[(\mathbf{G}_{r}^\top \mathbf{x})(\mathbf{G}_{r}^\top \mathbf{z})] = \ip{\mathbf{x}}{\mathbf{z}}.
\]
Therefore
\[
\E[\hat K(\mathbf{x},\mathbf{z})]
=
\E[(\va_{1}^\top \mathbf{x})(\va_{1}^\top \mathbf{z})]\,
\E[(\vb_{1}^\top \mathbf{x})(\vb_{1}^\top \mathbf{z})]
=
\ip{\mathbf{x}}{\mathbf{z}}^2.
\]
\end{proof}
For fixed $\mathbf{x},\mathbf{z}$, standard concentration bounds for Gaussian chaos terms~\citep{vershynin2018high} imply that the sketching error scales as $\bigl|\hat K(\mathbf{x},\mathbf{z})-K_{2}(\mathbf{x},\mathbf{z})\bigr| = O(m^{-1/2})$ with high probability.
In Appendix~\ref{app:theory:margin_bounds:rkrv}, we show that exact off-diagonal $K_{2}$ entries have size $\Theta(1/d)$ with high probability in the isotropic keys and values setting; these off-diagonal terms are precisely what drive the cross-talk contribution in the margin decomposition. Consequently, taking $m\gtrsim d^2$ suffices to make the sketching noise smaller than the natural off-diagonal $K_{2}$ scale, so that the sketching makes the margin no worse than the margin with the exact-$K_{2}$ kernel.

%% file: appendix/theory/02_hebbian_mlps/03_construction.tex
\Cref{alg:sketched_k2_hebbian_whitening} provides a pseudocode implementation of our the full bilinear MLP construction used in our experiments.

\begin{algorithm}[t]
\caption{Sketched-$K_{2}$ Hebbian Construction with Optional Whitening}
\label{alg:sketched_k2_hebbian_whitening}
\begin{algorithmic}[1]
\REQUIRE Keys $\{\mathbf{k}_{i}\}_{i=1}^F\subset\mathbb{R}^d$, value embeddings $\{\mathbf{v}_{j}\}_{j=1}^F\subset\mathbb{R}^d$, mapping $f:[F]\to[F]$
\REQUIRE Codes $\{\mathbf{c}_{j}\}_{j=1}^F\subset\mathbb{R}^{d_c}$ (default $\mathbf{c}_{j}=\mathbf{v}_{j}$), feature width $m$
\REQUIRE Optional whitening: mode $\in\{\texttt{none},\texttt{diag},\texttt{full}\}$ (default \texttt{none}), ridge $\lambda>0$
\ENSURE Feature map $g$, readout $\mathbf{B}\in\mathbb{R}^{d_{c}\times m}$, predictor $\hat{\mathbf{c}}(\mathbf{x})=\mathbf{B}\,g(\mathbf{x})$
\STATE Sample rows $\mathbf{A}_{r},\mathbf{G}_{r}\stackrel{i.i.d.}{\sim}\mathcal{N}(0,I_{d})$ for $r=1,\dots,m$
\STATE Define $g(\mathbf{x})=\frac{1}{\sqrt{m}}\big((\mathbf{A}_{r}^\top \mathbf{x})(\mathbf{G}_{r}^\top \mathbf{x})\big)_{r=1}^m$
\STATE Form $\mathbf{C}_{f}\in\mathbb{R}^{F\times d_c}$ with row $i$ equal to $\mathbf{c}_{f(i)}^\top$
\STATE Form $\mPhi\in\mathbb{R}^{F\times m}$ with row $i$ equal to $g(\mathbf{k}_{i})^\top$
\STATE Compute raw Hebbian readout $\mathbf{B}_{0}\gets \frac{1}{F}\mathbf{C}_{f}^\top\mPhi$
\IF{mode is \texttt{none}}
  \STATE $\mathbf{B}\gets \mathbf{B}_{0}$
\ELSIF{mode is \texttt{diag}}
  \STATE $\Sigma\gets\frac{1}{F}\mPhi^\top\mPhi$
  \FOR{$r=1,\dots,m$}
    \STATE $\mathbf{B}_{:,r}\gets \mathbf{B}_{0,:,r}/(\Sigma_{rr}+\lambda)$
  \ENDFOR
\ELSIF{mode is \texttt{full}}
  \IF{$m\le F$}
    \STATE $\Sigma\gets\frac{1}{F}\mPhi^\top\mPhi$
    \STATE $\mathbf{B}\gets \mathbf{B}_{0}(\Sigma+\lambda I_{m})^{-1}$
  \ELSE
    \STATE $\mathbf{B}\gets \mathbf{C}_{f}^\top(\mPhi\mPhi^\top+\lambda I_{F})^{-1}\mPhi$ \COMMENT{dual branch}
  \ENDIF
\ENDIF
\STATE Return $\hat{\mathbf{c}}(\mathbf{x})=\mathbf{B}\,g(\mathbf{x})$ and retrieval scores $s_{j}(\mathbf{x})=\langle \mathbf{v}_{j},\hat{\mathbf{c}}(\mathbf{x})\rangle$
\end{algorithmic}
\end{algorithm}

%% file: appendix/theory/03_margin_bounds/06_kernel_whitening.tex
Here, we describe the kernel whitening procedure used in~\Cref{sec:4} to improve our Hebbian MLP's fact storage capacity. For the sketched-$K_{2}$ feature map, let
\[
\mPhi=
\begin{bmatrix}
g(\vk_{1})^\top\\
\vdots\\
g(\vk_{F})^\top
\end{bmatrix}\in\mathbb{R}^{F\times m},
\qquad
\mathbf{C}_{f}=
\begin{bmatrix}
\vc_{f(1)}^\top\\
\vdots\\
\vc_{f(F)}^\top
\end{bmatrix}\in\mathbb{R}^{F\times d_c},
\]
with $g(\mathbf{x})=\frac{1}{\sqrt m}\big((\mathbf{A}_{r}^\top \mathbf{x})(\mathbf{G}_{r}^\top \mathbf{x})\big)_{r=1}^m$ as in
\Cref{thm:bilinear-scaling}. The raw Hebbian readout is
\begin{equation}
\mathbf{B}_{0}=\frac{1}{F}\mathbf{C}_{f}^\top\mPhi,
\qquad
\hat{\mSigma}=\frac{1}{F}\mPhi^\top\mPhi.
\label{eq:app:whitening:raw}
\end{equation}

The whitened construction replaces the raw Hebbian readout by the full ridge-whitened readout indexed by $\lambda\ge 0$,
\begin{equation}
\mathbf{B}_{\lambda}
\;=\;
\mathbf{B}_{0}(\hat{\mSigma}+\lambda I_{m})^{-1}.
\label{eq:app:whitening:full_primal}
\end{equation}
This rescales the bilinear features according to their empirical covariance, mitigating the feature imbalance that appears at finite width. Unless stated otherwise, our construction uses $\lambda=10^{-6}$ by default.

When $m>n$, we instead perform the corresponding dual solve for numerical stability:
\begin{equation}
\mathbf{B}_{\lambda}
=
\mathbf{C}_{f}^\top(\mPhi\mPhi^\top+\lambda I_{n})^{-1}\mPhi.
\label{eq:app:whitening:full_dual}
\end{equation}
The primal and dual forms are equivalent up to scaling.

\paragraph{Whitening to reduce key crowding $E_{K}$.} \label{app:theory:margin_bounds:kernel_whitening_crowding}
Let $g:\mathbb{R}^d\to\mathbb{R}^m$, $\mPhi\in\mathbb{R}^{n\times m}$ and $\hat{\mSigma}$ be as defined above.
The unwhitened Gram matrix on the stored keys is $\mathbf{K}_{\mathrm{raw}} := \mPhi\mPhi^\top$.
Full whitening replaces this by the \emph{preconditioned} Gram matrix
\[
\mathbf{K}_{\mathrm{white}} \;:=\; \mPhi \hat{\mSigma}^{-1}\mPhi^\top
\;=\; \big(\mPhi\hat{\mSigma}^{-1/2}\big)\big(\mPhi\hat{\mSigma}^{-1/2}\big)^\top ,
\]
equivalently using whitened features $\tilde g(\cdot)=\hat{\mSigma}^{-1/2}g(\cdot)$.

Recall the key-crowding statistic:
\[
E_{K}(i)\;:=\;\sum_{t\neq i} K_{ti}^2,
\qquad
E_{K}\;:=\;\max_{i} E_{K}(i).
\]
Since $E_{K}(i)\le \sum_{t} K_{ti}^2$ and $\max_{i}(\cdot)\le \sum_{i}(\cdot)$, we always have
\begin{equation}\label{eq:EK_le_frob}
E_{K} \;\le\; \|\mathbf{K}\|_{F}^2 .
\end{equation}

\begin{lemma}[Whitening minimizes upper bound on $E_{K}$]
\label{lem:whiten_min_frob}
Assume $\hat{\mSigma}\succ 0$. For any PSD preconditioner $\mathbf{M}\succeq 0$, define
$\mathbf{K}_{\mathbf{M}} := \mPhi \mathbf{M} \mPhi^\top$ (i.e.\ using features $\mathbf{M}^{1/2}g(\cdot)$).
Among all $\mathbf{M}\succeq 0$ with fixed average self-kernel
\begin{equation}\label{eq:trace_norm}
\frac1F \mathrm{tr}(\mathbf{K}_{\mathbf{M}}) \;=\; \mathrm{tr}(\mathbf{M}\hat{\mSigma}) \;=\; m ,
\end{equation}
the choice $\mathbf{M}^\star=\hat{\mSigma}^{-1}$ (full whitening) minimizes $\|\mathbf{K}_{\mathbf{M}}\|_{F}^2$.
Consequently, by \eqref{eq:EK_le_frob}, whitening minimizes an explicit upper bound on $E_{K}$:
\[
E_{K}(\mathbf{K}_{\mathrm{white}})\;\le\;\|\mathbf{K}_{\mathrm{white}}\|_{F}^2
\;=\;\min_{\substack{\mathbf{M}\succeq 0:\\ \mathrm{tr}(\mathbf{M}\hat{\mSigma})=m}}\;\|\mathbf{K}_{\mathbf{M}}\|_{F}^2 .
\]
\end{lemma}

\begin{proof}
Write $\bar{\mPhi}:=\mPhi/\sqrt{F}$ so that $\hat{\mSigma}=\bar{\mPhi}^\top \bar{\mPhi}$ and
$\mathbf{K}_{\mathbf{M}} = F\,\bar{\mPhi} \mathbf{M} \bar{\mPhi}^\top$.
Using $\|\mathbf{A}\mathbf{A}^\top\|_{F}=\|\mathbf{A}^\top \mathbf{A}\|_{F}$ with $\mathbf{A}=\bar{\mPhi} \mathbf{M}^{1/2}$,
\[
\|\mathbf{K}_{\mathbf{M}}\|_{F}^2
= F^2 \|\bar{\mPhi} \mathbf{M} \bar{\mPhi}^\top\|_{F}^2
= F^2 \|\mathbf{M}^{1/2}\hat{\mSigma} \mathbf{M}^{1/2}\|_{F}^2.
\]
Let $\mathbf{A} := \hat{\mSigma}^{1/2} \mathbf{M} \hat{\mSigma}^{1/2}\succeq 0$. Then
$\mathrm{tr}(\mathbf{A})=\mathrm{tr}(\mathbf{M}\hat{\mSigma})=m$ and $\|\mathbf{M}^{1/2}\hat{\mSigma} \mathbf{M}^{1/2}\|_{F}=\|\mathbf{A}\|_{F}$,
so minimizing $\|\mathbf{K}_{\mathbf{M}}\|_{F}^2$ under \eqref{eq:trace_norm} is equivalent to minimizing $\|\mathbf{A}\|_{F}^2$
subject to $\mathbf{A}\succeq 0$ and $\mathrm{tr}(\mathbf{A})=m$.
If $\{\lambda_{r}\}_{r=1}^m$ are the eigenvalues of $\mathbf{A}$, then $\|\mathbf{A}\|_{F}^2=\sum_{r} \lambda_{r}^2$
and $\sum_{r} \lambda_{r}=m$. By Cauchy--Schwarz,
$\sum_{r} \lambda_{r}^2 \ge (\sum_{r} \lambda_{r})^2/m = m$, with equality iff all $\lambda_{r}=1$,
i.e.\ $\mathbf{A}=I_{m}$. Thus $\hat{\mSigma}^{1/2} \mathbf{M} \hat{\mSigma}^{1/2}=I_{m}$, so $\mathbf{M}=\hat{\mSigma}^{-1}$.
\end{proof}

%% file: appendix/theory/03_margin_bounds/07_data_dependent_construction.tex
Our data-dependent construction used in~\Cref{sec:4} keeps the same bilinear architecture as the sketched-$K_{2}$ MLP, but refines the bilinear feature factors using the fact set and key/value embeddings. We describe our procedure here.

Let
\[
\mathbf{K}\in\mathbb{R}^{n\times d},
\qquad
\mathbf{C}_{f}\in\mathbb{R}^{n\times d_c},
\]
where the rows of $\mathbf{K}$ are the stored keys and the rows of $\mathbf{C}_{f}$ are the corresponding target codes.
For random feature matrices $\mathbf{A},\mathbf{G}\in\mathbb{R}^{m\times d}$, define
\[
\mPhi(\mathbf{A},\mathbf{G})\;:=\;(\mathbf{K}\mathbf{A}^\top)\odot(\mathbf{K}\mathbf{G}^\top)\in\mathbb{R}^{n\times m}.
\]
We initialize $\mathbf{A}_{0}$ and $\mathbf{G}_{0}$ with the same random bilinear sketch used in the sketched-$K_{2}$ construction, and form the corresponding Hebbian
\[
\mathbf{B}_{0} \;:=\; \frac{1}{n}\mathbf{C}_{f}^\top \mPhi(\mathbf{A}_{0},\mathbf{G}_{0})\in\mathbb{R}^{d_{c}\times m}.
\]

The data-dependent kernel is obtained by performing two least-squares solves:
\begin{align}
\mathbf{G}_{1}
&\in
\operatorname*{argmin}_{\mathbf{G}\in\mathbb{R}^{m\times d}}
\bigl\|
\mathbf{C}_{f}-\mPhi(\mathbf{A}_{0},\mathbf{G})\mathbf{B}_{0}^\top
\bigr\|_{F}^2,
\\
\mathbf{A}_{1}
&\in
\operatorname*{argmin}_{\mathbf{A}\in\mathbb{R}^{m\times d}}
\bigl\|
\mathbf{C}_{f}-\mPhi(\mathbf{A},\mathbf{G}_{1})\mathbf{B}_{0}^\top
\bigr\|_{F}^2.
\end{align}
Note that each subproblem is linear because $\mPhi(\mathbf{A},\mathbf{G})$ is linear in either factor once the other is held fixed.

After these two updates, we discard the intermediate readout $\mathbf{B}_{0}$ and form the learned feature matrix
\[
\mPhi_{1} \;:=\; \mPhi(\mathbf{A}_{1},\mathbf{G}_{1}).
\]
We then replace $\mathbf{B}_{0}$ with the full ridge-whitened readout from \Cref{eq:app:whitening:full_primal}:
\[
\mathbf{B}_{\lambda}
\;=\;
\frac{1}{n}\mathbf{C}_{f}^\top \mPhi_{1}
\Bigl(\frac{1}{n}\mPhi_{1}^\top\mPhi_{1}+\lambda I_{m}\Bigr)^{-1}.
\]

%% file: appendix/theory/03_margin_bounds/07_bit_complexity.tex
We now extend the real-valued parameter-count statement of \Cref{cor:optimal-fact-storage-capacity} to a bounded-precision \emph{bit complexity} theorem.
The proof has three steps:
\textbf{(i)} positive margin implies robustness to output perturbations,
\textbf{(ii)} the bilinear MLP is Lipschitz in its parameters on the stored keys,
and \textbf{(iii)} sufficiently fine parameter quantization therefore preserves all margins.

\paragraph{Setup.}
Recall from \Cref{eq:mlp-construction} that our bilinear MLP has the form
\begin{equation}
g_{\boldsymbol{\theta}}(\mathbf{x})=\mathbf{B}\bigl((\mathbf{A}\mathbf{x})\odot(\mathbf{G}\mathbf{x})\bigr),
\qquad
\boldsymbol{\theta}=(\mathbf{A},\mathbf{G},\mathbf{B}),
\label{eq:bilinear_mlp_bits}
\end{equation}
with $\mathbf{A},\mathbf{G}\in\mathbb{R}^{m\times d}$ and $\mathbf{B}\in\mathbb{R}^{d\times m}$.
Let $P:=\dim(\boldsymbol{\theta})$ denote the total number of scalar parameters, so
$P\asymp md$ up to an absolute constant factor.

For a stored fact set $f:[F]\to[|V|]$, define the minimum margin
\[
\gamma_{\min}(\boldsymbol{\theta})
:=
\min_{i\in[F]}
\min_{j\neq f(i)}
\bigl\langle
g_{\boldsymbol{\theta}}(\mathbf{k}_{i}),\, \mathbf{v}_{f(i)}-\mathbf{v}_{j}
\bigr\rangle.
\]
We also write
\[
R_{K}:=\max_{i\in[F]}\|\mathbf{k}_{i}\|_{2},
\qquad
R_{V}:=\max_{a\in[|V|]}\|\mathbf{v}_{a}\|_{2}.
\]
Note that in the isotropic and unit-norm key/value regime of \Cref{sec:isotropic-case} and Appendix~\ref{app:theory:margin_bounds:rkrv}, one has $R_{K}\le 1$ and $R_{V}\le 1$.

\begin{lemma}[Margin robustness under output perturbations]
\label{lem:margin_robustness_bits}
Let $\boldsymbol{\theta}^\star$ be any parameter vector such that
$\gamma_{\min}(\boldsymbol{\theta}^\star)\ge \gamma_{0}>0$.
Assume that for some $\widetilde{\boldsymbol{\theta}}$,
\[
\max_{i\in[F]}
\bigl\|
g_{\widetilde{\boldsymbol{\theta}}}(\mathbf{k}_{i})-g_{\boldsymbol{\theta}^\star}(\mathbf{k}_{i})
\bigr\|_{2}
\le
\frac{\gamma_0}{4R_V}.
\]
Then $\widetilde{\boldsymbol{\theta}}$ stores the same fact set in the sense of
\Cref{def:fact-storage}, and in fact
\[
\gamma_{\min}(\widetilde{\boldsymbol{\theta}})\ge \frac{\gamma_0}{2}.
\]
\end{lemma}

\begin{proof}
Fix any stored key $\mathbf{k}_{i}$ and any competitor $j\neq f(i)$.
Then
\begin{align*}
\Bigl\langle
g_{\widetilde{\boldsymbol{\theta}}}(\mathbf{k}_{i}),\, \mathbf{v}_{f(i)}-\mathbf{v}_{j}
\Bigr\rangle
&=
\Bigl\langle
g_{\boldsymbol{\theta}^\star}(\mathbf{k}_{i}),\, \mathbf{v}_{f(i)}-\mathbf{v}_{j}
\Bigr\rangle
+
\Bigl\langle
g_{\widetilde{\boldsymbol{\theta}}}(\mathbf{k}_{i})-g_{\boldsymbol{\theta}^\star}(\mathbf{k}_{i}),\, \mathbf{v}_{f(i)}-\mathbf{v}_{j}
\Bigr\rangle \\
&\ge
\gamma_{0}
-
\bigl\|
g_{\widetilde{\boldsymbol{\theta}}}(\mathbf{k}_{i})-g_{\boldsymbol{\theta}^\star}(\mathbf{k}_{i})
\bigr\|_{2}
\,
\|\mathbf{v}_{f(i)}-\mathbf{v}_{j}\|_{2} \\
&\ge
\gamma_{0}
-
\frac{\gamma_0}{4R_V}
\bigl(\|\mathbf{v}_{f(i)}\|_{2}+\|\mathbf{v}_{j}\|_{2}\bigr) \\
&\ge
\gamma_{0}-\frac{\gamma_0}{4R_V}(2R_{V})
=
\frac{\gamma_0}{2}.
\end{align*}
Since this holds for every $i$ and every $j\neq f(i)$, we obtain $\gamma_{\min}(\widetilde{\boldsymbol{\theta}})\ge \gamma_{0}/2>0$, which implies that $\widetilde{\boldsymbol{\theta}}$ stores the same fact set.
\end{proof}

\begin{lemma}[Bilinear MLPs are Lipschitz in their parameters on stored keys]
\label{lem:bilinear_param_lipschitz_bits}
Fix any $\mathbf{x}\in\mathbb{R}^d$.
Let $\boldsymbol{\theta}=(\mathbf{A},\mathbf{G},\mathbf{B})$ and $\boldsymbol{\theta}'=(\mathbf{A}',\mathbf{G}',\mathbf{B}')$, and define
\[
M:=
\max\Bigl\{
\|\mathbf{A}\|_{\mathrm{op}},\|\mathbf{G}\|_{\mathrm{op}},\|\mathbf{B}\|_{\mathrm{op}},
\|\mathbf{A}'\|_{\mathrm{op}},\|\mathbf{G}'\|_{\mathrm{op}},\|\mathbf{B}'\|_{\mathrm{op}}
\Bigr\}.
\]
Then
\[
\bigl\|
g_{\boldsymbol{\theta}}(\mathbf{x})-g_{\boldsymbol{\theta}'}(\mathbf{x})
\bigr\|_{2}
\le
3M^2\|\mathbf{x}\|_{2}^2
\|\boldsymbol{\theta}-\boldsymbol{\theta}'\|_{2},
\]
where
\[
\|\boldsymbol{\theta}-\boldsymbol{\theta}'\|_{2}^2
:=
\|\mathbf{A}-\mathbf{A}'\|_{F}^2+\|\mathbf{G}-\mathbf{G}'\|_{F}^2+\|\mathbf{B}-\mathbf{B}'\|_{F}^2.
\]
Consequently, on any parameter region on which the operator norms of $\mathbf{A},\mathbf{G},\mathbf{B}$ are uniformly bounded by $M$, the map $\boldsymbol{\theta}\mapsto g_{\boldsymbol{\theta}}(\mathbf{k}_{i})$ is $3M^2R_{K}^2$-Lipschitz for every stored key $\mathbf{k}_{i}$ satisfying $\|\mathbf{k}_{i}\|_{2}\le R_{K}$.
\end{lemma}

\begin{proof}
Write
\begin{align*}
g_{\boldsymbol{\theta}}(\mathbf{x})-g_{\boldsymbol{\theta}'}(\mathbf{x})
&=
\mathbf{B}(\mathbf{A}\mathbf{x}\odot \mathbf{G}\mathbf{x})-\mathbf{B}'(\mathbf{A}'\mathbf{x}\odot \mathbf{G}'\mathbf{x}) \\
&=
(\mathbf{B}-\mathbf{B}')(\mathbf{A}\mathbf{x}\odot \mathbf{G}\mathbf{x})
+
\mathbf{B}'\Bigl((\mathbf{A}\mathbf{x}\odot \mathbf{G}\mathbf{x})-(\mathbf{A}'\mathbf{x}\odot \mathbf{G}'\mathbf{x})\Bigr) \\
&=
(\mathbf{B}-\mathbf{B}')(\mathbf{A}\mathbf{x}\odot \mathbf{G}\mathbf{x})
+
\mathbf{B}'\Bigl(((\mathbf{A}-\mathbf{A}')\mathbf{x})\odot \mathbf{G}\mathbf{x}\Bigr)
+
\mathbf{B}'\Bigl(\mathbf{A}'\mathbf{x}\odot ((\mathbf{G}-\mathbf{G}')\mathbf{x})\Bigr).
\end{align*}
Using
$\|\mathbf{u}\odot \mathbf{v}\|_{2} \le \|\mathbf{u}\|_{2}\|\mathbf{v}\|_\infty \le \|\mathbf{u}\|_{2}\|\mathbf{v}\|_{2}$
together with
$\|\mathbf{T}\mathbf{x}\|_{2}\le \|\mathbf{T}\|_{\mathrm{op}}\|\mathbf{x}\|_{2}$
gives
\begin{align*}
\|(\mathbf{B}-\mathbf{B}')(\mathbf{A}\mathbf{x}\odot \mathbf{G}\mathbf{x})\|_{2}
&\le
\|\mathbf{B}-\mathbf{B}'\|_{\mathrm{op}}\|\mathbf{A}\mathbf{x}\|_{2}\|\mathbf{G}\mathbf{x}\|_{2}
\le
M^2\|\mathbf{x}\|_{2}^2\|\mathbf{B}-\mathbf{B}'\|_{F},\\
\|\mathbf{B}'(((\mathbf{A}-\mathbf{A}')\mathbf{x})\odot \mathbf{G}\mathbf{x})\|_{2}
&\le
\|\mathbf{B}'\|_{\mathrm{op}}\|(\mathbf{A}-\mathbf{A}')\mathbf{x}\|_{2}\|\mathbf{G}\mathbf{x}\|_{2}
\le
M^2\|\mathbf{x}\|_{2}^2\|\mathbf{A}-\mathbf{A}'\|_{F},\\
\|\mathbf{B}'(\mathbf{A}'\mathbf{x}\odot ((\mathbf{G}-\mathbf{G}')\mathbf{x}))\|_{2}
&\le
\|\mathbf{B}'\|_{\mathrm{op}}\|\mathbf{A}'\mathbf{x}\|_{2}\|(\mathbf{G}-\mathbf{G}')\mathbf{x}\|_{2}
\le
M^2\|\mathbf{x}\|_{2}^2\|\mathbf{G}-\mathbf{G}'\|_{F}.
\end{align*}
Summing the three bounds and using
$a+b+c\le 3\sqrt{a^2+b^2+c^2}$ yields
\[
\|g_{\boldsymbol{\theta}}(\mathbf{x})-g_{\boldsymbol{\theta}'}(\mathbf{x})\|_{2}
\le
3M^2\|\mathbf{x}\|_{2}^2
\sqrt{
\|\mathbf{A}-\mathbf{A}'\|_{F}^2+\|\mathbf{G}-\mathbf{G}'\|_{F}^2+\|\mathbf{B}-\mathbf{B}'\|_{F}^2
}
=
3M^2\|\mathbf{x}\|_{2}^2\|\boldsymbol{\theta}-\boldsymbol{\theta}'\|_{2}.
\]
\end{proof}

\begin{theorem}[Bounded-bit implementation]
\label{thm:bounded_bit_implementation}
Fix a fact set and a real-valued parameter vector $\boldsymbol{\theta}^\star$ for a bilinear MLP of the form in \Cref{eq:bilinear_mlp_bits}.
Assume:
\begin{enumerate}
    \item[(i)] \textbf{Positive margin:}
    \[
    \gamma_{\min}(\boldsymbol{\theta}^\star)\ge \gamma_{0}>0.
    \]
    \item[(ii)] \textbf{Bounded dynamic range:}
    every coordinate of $\boldsymbol{\theta}^\star$ lies in $[-C_{\mathrm{rng}},C_{\mathrm{rng}}]$.
    \item[(iii)] \textbf{Parameter Lipschitzness on stored keys:}
    there exists $L_\theta>0$ such that for all
    $\boldsymbol{\theta},\boldsymbol{\theta}'\in[-C_{\mathrm{rng}},C_{\mathrm{rng}}]^P$,
    \begin{equation}
    \max_{i\in[F]}
    \|g_{\boldsymbol{\theta}}(\mathbf{k}_{i})-g_{\boldsymbol{\theta}'}(\mathbf{k}_{i})\|_{2}
    \le
    L_\theta\|\boldsymbol{\theta}-\boldsymbol{\theta}'\|_{2}.
    \label{eq:param_lipschitz_assumption_bits}
    \end{equation}
\end{enumerate}
Then there exists a floating-point parameter vector $\widetilde{\boldsymbol{\theta}}$
that stores the same fact set and satisfies
\[
\gamma_{\min}(\widetilde{\boldsymbol{\theta}})\ge \frac{\gamma_0}{2}.
\]
Moreover, one may choose $\widetilde{\boldsymbol{\theta}}$ by rounding each coordinate
of $\boldsymbol{\theta}^\star$ to a binary floating-point representation with $t$ mantissa bits, where
\[
t
=
O\!\left(
\log\!\left(
\frac{4C_{\mathrm{rng}}L_\theta R_{V}\sqrt{P}}{\gamma_0}
\right)
\right).
\]
In this case, each coordinate can be encoded using at most
\[
O\!\left(
\log\log C_{\mathrm{rng}}
+
\log\!\left(
\frac{4C_{\mathrm{rng}}L_\theta R_{V}\sqrt{P}}{\gamma_0}
\right)
\right)
\]
bits. Hence, the total number of bits needed to encode the quantized parameter vector
\(\widetilde{\boldsymbol{\theta}}\) is at most
\begin{equation}
\mathrm{Bits}(\widetilde{\boldsymbol{\theta}})
\le
P\cdot
O\!\left(
\log\log C_{\mathrm{rng}}
+
\log\!\left(
\frac{4C_{\mathrm{rng}}L_\theta R_{V}\sqrt{P}}{\gamma_0}
\right)
\right).
\label{eq:total_bits_bits}
\end{equation}
\end{theorem}

\begin{proof}
Let $\widetilde{\boldsymbol{\theta}}$ be obtained by rounding each coordinate of
$\boldsymbol{\theta}^\star$ to a binary floating-point representation with $t$ mantissa bits.
For standard floating-point rounding, each coordinate incurs relative error at most
$O(2^{-t})$, and hence absolute error at most
\[
O(2^{-t})\,C_{\mathrm{rng}}.
\]
Therefore
\[
\|\widetilde{\boldsymbol{\theta}}-\boldsymbol{\theta}^\star\|_{2}
\le
O(2^{-t})\,C_{\mathrm{rng}}\sqrt{P}.
\]
Choosing
\[
t
=
O\!\left(
\log\!\left(
\frac{4C_{\mathrm{rng}}L_\theta R_{V}\sqrt{P}}{\gamma_0}
\right)
\right)
\]
ensures that
\[
\|\widetilde{\boldsymbol{\theta}}-\boldsymbol{\theta}^\star\|_{2}
\le
\frac{\gamma_0}{4L_\theta R_V}.
\]
Applying the Lipschitz assumption (\Cref{eq:param_lipschitz_assumption_bits}),
\[
\max_{i\in[F]}
\|g_{\widetilde{\boldsymbol{\theta}}}(\mathbf{k}_{i})-g_{\boldsymbol{\theta}^\star}(\mathbf{k}_{i})\|_{2}
\le
L_\theta\|\widetilde{\boldsymbol{\theta}}-\boldsymbol{\theta}^\star\|_{2}
\le
\frac{\gamma_0}{4R_V}.
\]
Lemma~\ref{lem:margin_robustness_bits} therefore implies that
$\widetilde{\boldsymbol{\theta}}$ stores the same fact set and that
$\gamma_{\min}(\widetilde{\boldsymbol{\theta}})\ge \gamma_{0}/2$.

It remains to count bits.
Each nonzero floating-point coordinate can be written in the form
\[
\pm m\,2^e,
\]
where the mantissa $m$ is represented to $t$ bits of precision.
Since every coordinate lies in $[-C_{\mathrm{rng}},C_{\mathrm{rng}}]$, the exponent satisfies
$|e|=O(\log C_{\mathrm{rng}})$, and hence the exponent can be encoded using
$O(\log\log C_{\mathrm{rng}})$ bits.
Thus each coordinate can be encoded using at most
\[
O\!\left(
\log\log C_{\mathrm{rng}}
+
t
\right)
=
O\!\left(
\log\log C_{\mathrm{rng}}
+
\log\!\left(
\frac{4C_{\mathrm{rng}}L_\theta R_{V}\sqrt{P}}{\gamma_0}
\right)
\right)
\]
bits.
Multiplying by the \(P\) coordinates yields~\Cref{eq:total_bits_bits}.
\end{proof}

Next, we specialize the previous bit complexity theorem to our sketched bilinear MLP construction (\Cref{alg:sketched_k2_hebbian_whitening}) by bounding its dynamic range and Lipschitz constant. In the isotropic unit-norm regime of \Cref{sec:isotropic-case} and Appendix~\ref{app:theory:margin_bounds:rkrv}, they reduce to $C_{\mathrm{rng}},L_\theta\le d^{O(1)}$, yielding an explicit bit-complexity bound.

\begin{proposition}[Dynamic range and stored-key Lipschitzness for the unwhitened sketched-$K_{2}$ construction]
\label{prop:unwhitened_k2_norm_bounds}
Let \(\mathbf{a}_1,\dots,\mathbf{a}_m,\mathbf{b}_1,\dots,\mathbf{b}_m \in \mathbb{R}^d\) be sampled i.i.d. from \(N(0,\mathbf{I}_d)\), and consider the sketched-$K_{2}$ feature map
\[
g(\mathbf{x})=\frac{1}{\sqrt m}\bigl((\mathbf{a}_r^\top \mathbf{x})(\mathbf{b}_r^\top \mathbf{x})\bigr)_{r=1}^m.
\]
Let \(\bar{\mathbf{A}},\bar{\mathbf{G}}\in\mathbb{R}^{m\times d}\) be the matrices whose rows are \(\mathbf{a}_r^\top\) and \(\mathbf{b}_r^\top\), respectively, and realize \(g\) in bilinear MLP form by setting
\[
\mathbf{A}=m^{-1/4}\bar{\mathbf{A}},
\qquad
\mathbf{G}=m^{-1/4}\bar{\mathbf{G}}.
\]
Let the raw Hebbian readout be
\[
\mathbf{B}_0=\frac{1}{F}\mathbf{C}_f^\top \boldsymbol{\Phi},
\]
where the rows of \(\boldsymbol{\Phi}\in\mathbb{R}^{F\times m}\) are \(g(\mathbf{k}_i)^\top\)
and the rows of \(\mathbf{C}_f\in\mathbb{R}^{F\times d}\) are the value embeddings \(\mathbf{v}_{f(i)}^\top\).

Assume
\[
m,\;F,\;\delta^{-1}\le d^{O(1)}.
\]
Then, with probability at least \(1-\delta\) over the draw of
\(\{\mathbf{a}_r,\mathbf{b}_r\}_{r=1}^m\), the resulting bilinear MLP
\[
\mathbf{x}\mapsto \mathbf{B}_{0}\bigl((\mathbf{A}\mathbf{x})\odot(\mathbf{G}\mathbf{x})\bigr)
\]
has parameter dynamic range and stored-key parameter Lipschitz constant bounded by
\[
C_{\mathrm{rng}},\;L_\theta \le d^{O(1)}\,\poly(R_{K},R_{V}).
\]
where one may take
\[
C_{\mathrm{rng}} \le d^{O(1)}\max\{1,R_VR_{K}^2\},
\qquad
L_\theta \le d^{O(1)}\,R_{K}^2\max\{1,R_{V}^2R_{K}^4\}.
\]
\end{proposition}

\begin{proof}
Standard Gaussian random matrix bounds~\citep{vershynin2018high} imply that with probability at least $1-\delta$ over the draw of \(\{\mathbf{a}_r,\mathbf{b}_r\}_{r=1}^m\),
\[
\|\mathbf{A}\|_{\mathrm{op}},\|\mathbf{G}\|_{\mathrm{op}}
\lesssim
m^{-1/4}\Bigl(\sqrt m+\sqrt d+\sqrt{\log(1/\delta)}\Bigr).
\]
Under the assumption $m,\delta^{-1}\le d^{O(1)}$, it follows that
\[
\|\mathbf{A}\|_{\mathrm{op}},\|\mathbf{G}\|_{\mathrm{op}} \le d^{O(1)}.
\]

Next, since each value embedding has norm at most $R_{V}$,
\[
\|\mathbf{C}_{f}\|_{\mathrm{op}}\le \|\mathbf{C}_{f}\|_{F} \le \sqrt{F}\,R_{V}.
\]
Also,
\[
\|\boldsymbol{\Phi}\|_{\mathrm{op}}
\le
\|\boldsymbol{\Phi}\|_{F}
\le
\sqrt{F}\,\max_{i\in[F]}\|g(\mathbf{k}_{i})\|_{2}.
\]
For every stored key $\mathbf{k}_{i}$,
\[
\|g(\mathbf{k}_{i})\|_{2}
=
\bigl\|(\mathbf{A}\mathbf{k}_{i})\odot(\mathbf{G}\mathbf{k}_{i})\bigr\|_{2}
\le
\|\mathbf{A}\mathbf{k}_{i}\|_{2}\,\|\mathbf{G}\mathbf{k}_{i}\|_{2}
\le
\|\mathbf{A}\|_{\mathrm{op}}\|\mathbf{G}\|_{\mathrm{op}}R_{K}^2.
\]
Therefore
\[
\|\boldsymbol{\Phi}\|_{\mathrm{op}}
\le
\sqrt{F}\,\|\mathbf{A}\|_{\mathrm{op}}\|\mathbf{G}\|_{\mathrm{op}}R_{K}^2,
\]
and hence
\[
\|\mathbf{B}_{0}\|_{\mathrm{op}}
\le
\frac{1}{F}\|\mathbf{C}_{f}\|_{\mathrm{op}}\|\boldsymbol{\Phi}\|_{\mathrm{op}}
\le
R_{V}\|\mathbf{A}\|_{\mathrm{op}}\|\mathbf{G}\|_{\mathrm{op}}R_{K}^2
\le
d^{O(1)}R_VR_{K}^2
\]
with probability at least \(1-\delta\).

Now define
\[
M:=\max\{\|\mathbf{A}\|_{\mathrm{op}},\|\mathbf{G}\|_{\mathrm{op}},\|\mathbf{B}_{0}\|_{\mathrm{op}}\}.
\]
From the previous bounds,
\[
M \le d^{O(1)}\max\{1,R_VR_{K}^2\}
\]
with probability at least \(1-\delta\).

To pass from operator norms to coordinate bounds, note that for any matrix $\mathbf{T}$,
\[
|T_{ab}| = |\mathbf{e}_{a}^\top \mathbf{T} \mathbf{e}_{b}| \le \|\mathbf{T}\|_{\mathrm{op}}.
\]
Hence every scalar entry of $\mathbf{A},\mathbf{G},\mathbf{B}_{0}$ is bounded by $M$, and so one may take
\[
C_{\mathrm{rng}}\le M \le d^{O(1)}\max\{1,R_VR_{K}^2\}.
\]

Finally, \Cref{lem:bilinear_param_lipschitz_bits} yields
\[
L_\theta \le 3M^2R_{K}^2
\le
d^{O(1)}\,R_{K}^2\max\{1,R_{V}^2R_{K}^4\},
\]
which proves the claim.
\end{proof}

\begin{corollary}[Bit complexity in the isotropic regime for the unwhitened construction]
\label{cor:bit_complexity_isotropic}
Assume the isotropic key/value regime of
\Cref{cor:optimal-fact-storage-capacity}, and consider the sketched-$K_{2}$ construction of~\Cref{alg:sketched_k2_hebbian_whitening}.
Choose the width $m$ so that the corresponding real-valued construction
satisfies
\[
P\asymp md \asymp F\log(F/\delta)
\]
and has constant slack margin
\[
\gamma_{\min}(\boldsymbol{\theta}^\star)\ge c_{0}>0
\]
with high probability.
Assume further that $F,\delta^{-1}\le d^C$ for an absolute constant $C$.
Then with high probability the same fact set can be stored using
\begin{equation}
\mathrm{Bits}(\widetilde{\boldsymbol{\theta}})
=
O\!\Bigl(
F\log(F/\delta)\log d
\Bigr).
\label{eq:bit_complexity_isotropic_bits}
\end{equation}
\end{corollary}

\begin{proof}
By the isotropic real-valued capacity result, the unwhitened construction stores
$F$ facts using
\[
P\asymp md\asymp F\log(F/\delta)
\]
parameters, up to an absolute constant factor.
In the isotropic unit-norm regime, $R_{K},R_{V}\le 1$.
Since $F \leq d^{O(1)}$, \Cref{prop:unwhitened_k2_norm_bounds} yields
\[
C_{\mathrm{rng}},L_\theta \le d^{O(1)}
\]
with high probability.
Applying \Cref{thm:bounded_bit_implementation},
\[
\mathrm{Bits}(\widetilde{\boldsymbol{\theta}})
\le
P\,
O\!\left(
\log\log C_{\mathrm{rng}}
+
\log\!\left(
\frac{4C_{\mathrm{rng}}L_\theta R_{V}\sqrt{P}}{c_0}
\right)
\right).
\]
Since $C_{\mathrm{rng}},L_\theta,P\le d^{O(1)}$ under the standing assumption
$F,\delta^{-1}\le d^C$, the quantity inside the outer \(O(\cdot)\) is \(O(\log d)\).
Combining this with $P\asymp F\log(F/\delta)$ yields
\Cref{eq:bit_complexity_isotropic_bits}.
\end{proof}

%% file: appendix/theory/02_hebbian_mlps/04_ntk_construction.tex
For completeness, we describe the NTK baseline we implement from
\citet{nichani2024understandingfactualrecalltransformers}.

Let $\mathbf{K}\in\mathbb{R}^{F\times d}$ be the matrix whose $i$th row is $\mathbf{k}_{i}^\top$, and let
$\mathbf{C}_{f}\in\mathbb{R}^{F\times d_c}$ be the matrix whose $i$th row is $\mathbf{c}_{f(i)}^\top$,
where typically $\mathbf{c}_{j}=\mathbf{v}_{j}$. Given hidden width $m$, sample random gate directions
$\mathbf{w}_{r}\stackrel{i.i.d.}{\sim}\mathcal{N}(0,I_{d})$ and random output directions
$\mathbf{p}_{r}\in\mathbb{R}^{d_c}$ with $\|\mathbf{p}_{r}\|_{2}=1$ for $r=1,\dots,m$. Writing
$\mathbf{W}_{\mathrm{gate}}\in\mathbb{R}^{m\times d}$ for the matrix with rows $\mathbf{w}_{r}^\top$ and
$\mathbf{P}=[\mathbf{p}_{1},\dots,\mathbf{p}_{m}]\in\mathbb{R}^{d_{c}\times m}$, define the degree-$1$ Hermite feature
matrix
\[
\mathbf{H} \;:=\; \mathrm{He}_{1}(\mathbf{K}\mathbf{W}_{\mathrm{gate}}^\top)
\;=\;
\mathbf{K}\mathbf{W}_{\mathrm{gate}}^\top \in \mathbb{R}^{F\times m},
\]
where $\mathrm{He}_{1}(t)=t$ is applied entrywise. The up-projection is then chosen as
\[
\mathbf{W}_{\mathrm{up}}
\;:=\;
\frac{1}{m}\big(\mathbf{H}\odot(\mathbf{C}_f\mathbf{P})\big)^\top \mathbf{K} \in \mathbb{R}^{m\times d}.
\]

The NTK MLP construction is then
\[
\hat{\mathbf{c}}_{\mathrm{NTK}}(\mathbf{x})
\;=\;
\mathbf{P}\Big(\sigma(\mathbf{W}_{\mathrm{gate}}\mathbf{x})\odot(\mathbf{W}_{\mathrm{up}}\mathbf{x})\Big).
\]
Throughout this work, we take $\sigma=\mathrm{ReLU}$.

Note that \citet{nichani2024understandingfactualrecalltransformers}'s construction requires choosing a Hermite degree $k$.
This choice plays a role analogous to our kernel choice: it determines which degree polynomial interactions are emphasized by the construction.
In the most favorable case $k=1$, the fit uses the linear features $\mathrm{He}_{1}(\mathbf{K}\mathbf{W}_{\mathrm{gate}}^\top)=\mathbf{K}\mathbf{W}_{\mathrm{gate}}^\top$, and the realized finite-width model remains a gated bilinear MLP.
Thus, in expectation over the random features, the $k=1$ baseline captures quadratic interactions, like how our sketched bilinear construction approximates a quadratic kernel.
This result, and the fact that the capacity bounds in~\citet{nichani2024understandingfactualrecalltransformers} degrade exponentially with $k$, makes using $k=1$ the fairest comparison to our sketched-$K_{2}$ construction.
The main difference is that the NTK construction folds the target codes into the hidden coefficients through $\mathbf{C}_f\mathbf{P}$, whereas our construction uses an explicit bilinear random-feature map followed by a Hebbian readout.

%% file: appendix/theory/03_margin_bounds_full.tex
\newcommand{\vX}{\mathbf{X}}
\newcommand{\vY}{\mathbf{Y}}

\renewcommand{\paragraph}[1]{\par\medskip\noindent\textbf{#1}\enspace}

\newpage

\subsection{Margin Bounds Setup}
\label{app:theory:margin_bounds}
We lay out the setup we'll use throughout the appendix to prove our margin bounds.
\subsection*{Stored items and kernel}
We store $F$ key--value items $\{(\vk_{t},\vv_{t})\}_{t=1}^F$ with keys $\vk_{t}\in\R^d$ and values $\vv_{t}\in\R^d$.
A feature map $\phi:\R^d\to\R^p$ induces a kernel $K:\R^d\times \R^d \to \R$
\[
\hat{\matK}_{ti} \defeq K(\vk_{t}, \vk_{i})=\ip{\phi(\vk_{t})}{\phi(\vk_{i})}\qquad (t,i\in[F]).
\]
We write $\hat{\matK}\in\R^{F\times F}$ for the Gram matrix.

\subsection*{Codes and retrieval output}
Each index $t\in[F]$ has an associated code vector $\vc_{t}\in\R^{d}$.
Given a stored query at index $i$ (i.e. query key $\vk_{i}$), the retrieval output is
\begin{equation}
 \vy_{i} \defeq \sum_{t=1}^F \vc_{t}\,\hat{\matK}_{ti}\in\R^{d}.
 \label{eq:yi}
\end{equation}
Unless otherwise noted, for simplicity, we assume $\vc_{t}=\vv_{t}$ for all the upcoming theorems and their proofs.

\subsection*{Pairwise margin}
For a stored index $i$ and a competitor $j\neq i$, we define the pairwise margin
\begin{equation}
\gamma_{ij} \defeq \ip{\vv_{i}-\vv_j}{\vy_i}.
\label{eq:margin}
\end{equation}
We write $\gamma_{\min}\defeq\min_{i\neq j}\gamma_{ij}$ for the worst-case (minimum) pairwise margin.
Expanding, we see that the pairwise margin decomposes into signal and cross-talk components.

\begin{equation}
\gamma_{ij}
= \underbrace{\hat{\matK}_{ii}\,\ip{\vv_{i}-\vv_j}{\vc_i}}_{\text{signal }}
+ \underbrace{\sum_{t\neq i}\hat{\matK}_{ti}\,\ip{\vv_{i}-\vv_j}{\vc_t}}_{\text{cross-talk }}.
\label{eq:signal-crosstalk}
\end{equation}

\subsection*{Margin convention (absorbing the competitor)}
For the pairwise margin $\gamma_{ij}$, the single term $t=j$ inside the cross-talk sum in \eqref{eq:signal-crosstalk}
corresponds to the specific \emph{competitor} item. For simplicity of our proofs, in the rest of the appendix, we re-define the signal and cross-talk by absorbing the competitor term into the signal:

\begin{equation}
\gamma_{ij} = \underbrace{\widetilde{s}_{ij}}_{\text{signal}} + \underbrace{\widetilde{z}_{ij}}_{\text{cross-talk}}
\label{eq:new-signal-crosstalk}
\end{equation}
with
\begin{equation}
    \widetilde{s}_{ij} \defeq \hat{\matK}_{ii}\,\ip{\vv_{i}-\vv_j}{\vc_i} + \hat{\matK}_{ji}\,\ip{\vv_{i}-\vv_j}{\vc_j}, \qquad
    \widetilde{z}_{ij} \defeq \sum_{t\notin\{i,j\}}\hat{\matK}_{ti}\,\ip{\vv_{i}-\vv_j}{\vc_t}
    \label{eq:new-signal-crosstalk-sij-zij}
\end{equation}

\subsection{Embedding Geometric Summary Statistics Definitions}\label{app:emb-geo-defns}
We collect here the formal definitions of all geometric summary statistics that enter our margin bounds.
Throughout, $\hat{\matK}$ denotes the kernel Gram matrix with entries $\hat{\matK}_{ti} = K(\vk_t, \vk_i)$, and $\vv_1,\dots,\vv_F \in \R^d$ are the stored value vectors with associated Hebbian codes $\vc_1,\dots,\vc_F \in \R^d$.

\paragraph{Key-geometry statistics.}

\begin{definition}[Diagonal and off-diagonal kernel energies]\label{def:Kdiag}
\[
K_{\min}^{\mathrm{diag}} \;\defeq\; \min_{i \in [F]}\, \hat{\matK}_{ii},
\qquad
K_{\max}^{\mathrm{off}} \;\defeq\; \max_{i \neq j}\, |\hat{\matK}_{ij}|.
\]
$K_{\min}^{\mathrm{diag}}$ is the minimum kernel self-similarity; $K_{\max}^{\mathrm{off}}$ is the maximum off-diagonal kernel entry.
\end{definition}

\begin{definition}[Kernel column energy]\label{def:Ecol}
\[
E_K \;\defeq\; \max_{i \in [F]}\, \sum_{t \neq i} \hat{\matK}_{ti}^2.
\]
$E_K$ measures the worst-case squared $\ell_2$ energy of an off-diagonal kernel column, capturing key-embedding crowding: it grows when keys cluster in embedding space and kernel overlaps are large.
\end{definition}

\paragraph{Value-geometry statistics.}

\begin{definition}[Value separability]\label{def:Vmin}
\[
V_{\min} \;\defeq\; \min_{i \neq j}\, \langle \vv_i - \vv_j,\, \vv_i \rangle,
\qquad
V_{\max} \;\defeq\; \max_{i \neq j}\, |\langle \vv_i - \vv_j,\, \vv_j \rangle|.
\]
$V_{\min}$ is the signal-side value separability floor; $V_{\max}$ is the cross-talk-side value interaction ceiling.
\end{definition}

For each pair $i \neq j$, let $\vY^{(ij)} \defeq \bigl(\langle \vv_i - \vv_j,\, \vc_t \rangle\bigr)_{t \notin \{i,j\}} \in \R^{F-2}$ denote the vector of value-difference inner products with all non-target, non-competitor codes, and let $\vone \in \R^{F-2}$ denote the all-ones vector.

\begin{definition}[Mean competitor alignment]\label{def:BY}
\[
B_Y
\defeq
\max_{i\neq j}
\left|
\sum_{t\notin\{i,j\}}
\left\langle \vv_i-\vv_j,\vc_t\right\rangle
\right|
=
\max_{i\neq j}
\left|
\left\langle \vone,\vY^{(ij)}\right\rangle
\right|.
\]
$B_Y$ controls the bias contribution to cross-talk arising from the mean of the kernel off-diagonal entries (equal to $1/d$ under isotropic keys).
\end{definition}

\begin{definition}[Value-difference energy]\label{def:EY}
\[
E_v \;\defeq\; \max_{i \in [F]}\max_{j \neq i}\, \sum_{t \neq i} \langle \vv_i - \vv_j,\, \vc_t \rangle^2.
\]
$E_v$ measures the worst-case squared $\ell_2$ energy of the value-difference inner-product vector, capturing value-embedding interference: it grows when value embeddings cluster and non-target codes align with the target direction.
\end{definition}

\begin{definition}[Value sparsity]\label{def:Lv}
Recall $\vY^{(ij)}$ from \cref{def:BY}. Define
\[
L_v \;\defeq\; \max_{i \neq j}\, \frac{\|\vY^{(ij)}\|_1^2}{\|\vY^{(ij)}\|_2^2}
\]
$L_v \in [1, F-2]$ is an effective sparsity parameter: it equals $1$ when the energy of $\vY^{(ij)}$ is concentrated on a single coordinate and equals $F-2$ when it is spread uniformly. It amplifies cross-talk when value-difference inner products have heavy-tailed distributions across competitors.
\end{definition}

\paragraph{Key--value coupling.}

\begin{definition}[Coupling factor]\label{def:kappa}
For stored index $i$ and competitor $j \neq i$, let $E_K(i) \defeq \sum_{t \neq i} \hat{\matK}_{ti}^2$ and $E_v(i,j) \defeq \sum_{t \neq i} \langle \vv_i - \vv_j, \vc_t \rangle^2$ be the pairwise energies. Define the pairwise coupling
\[
\kappa^{ij}
\;\defeq\;
\frac{\displaystyle\left|\sum_{t \neq i} \hat{\matK}_{ti}\,\langle \vv_i - \vv_j,\, \vc_t \rangle\right|}{\sqrt{E_K(i)}\,\sqrt{E_v(i,j)}}
\;\in\; [0,1],
\]
\[
\kappa \;\defeq\; \max_{i \in [F]}\,\max_{j \neq i}\; \kappa^{ij}.
\]
$\kappa$ quantifies the worst-case alignment between the kernel column pattern and the value-interference pattern: $\kappa = 0$ when the two are orthogonal and $\kappa = 1$ when they are perfectly aligned.
\end{definition}

\subsection{Cross-Talk Bounds}
We start by providing upper bounds for the cross-talk term $\widetilde{z}_{ij}$.
\paragraph{Cross-talk as inner product.} First, we define cross-talk as an inner product. Concretely, fix $(i,j)$ with $j\neq i$ and define the off-diagonal kernel column
\begin{equation}
\vX^{(ij)} \defeq \big(\hat{\matK}_{ti}\big)_{t\notin \{i,j\}}\in\R^{F-2}.
\label{eq:Xdef}
\end{equation}
And for the value/code side:
\begin{equation}
\vY^{(ij)} \defeq \big(\ip{\vv_{i}-\vv_j}{\vc_t})_{t\notin \{ i,j\}}\in\R^{F-2},
\label{eq:Ydef}
\end{equation}
Then the cross-talk term is the inner product
\begin{equation}
 \widetilde{z}_{ij}
 \defeq \sum_{t\notin\{i,j\}}\hat{\matK}_{ti}\,\ip{\vv_{i}-\vv_j}{\vc_t}
 =\ip{\vX^{(ij)}}{\vY^{(ij)}}.
 \label{eq:Zij}
\end{equation}
We consider the largest (worst case) possible cross-talk:
\begin{equation}
 \widetilde{z}_{\max}\ \defeq\ \max_{i\neq j}|\widetilde{z}_{ij}|.
 \label{eq:zmax}
\end{equation}
Note that $\widetilde{z}_{ij}\le \widetilde{z}_{\max}$ for all $i\neq j$.

\paragraph{Cross-talk summary statistics.} For each $(i,j)$ define the (squared) energies
\begin{equation}
\widetilde{E}_{v}(i,j) \defeq \norm{\vY^{(ij)}}_{2}^2 = \sum_{t\notin\{i,j\}}\ip{\vv_{i}-\vv_j}{\vc_t}^2,
\qquad
\widetilde{E}_{K}(i,j) \defeq \norm{\vX^{(ij)}}_{2}^2 = \sum_{t\notin\{i,j\}}\hat{\matK}_{ti}^2.
\label{eq:EvEk}
\end{equation}
Define the coupling factor
\begin{equation}
\widetilde{\kappa}^{ij} \defeq \frac{|\ip{\vX^{(ij)}}{\vY^{(ij)}}|}{\norm{\vX^{(ij)}}_{2}\,\norm{\vY^{(ij)}}_2}\in[0,1],
\label{eq:kappaij}
\end{equation}
(setting $\widetilde{\kappa}^{ij}=0$ if a norm is zero).
Now define the worst-case summary statistics
\begin{equation}
\widetilde{E}_{v} \defeq \max_{i\neq j} \widetilde{E}_{v}(i,j),
\qquad
\widetilde{E}_{K} \defeq \max_{i\neq j} \widetilde{E}_{K}(i,j),
\qquad
\widetilde{\kappa} \defeq \max_{i\neq j} \widetilde{\kappa}^{ij}.
\label{eq:summary-stats}
\end{equation}

\subsubsection{Arbitrary Keys, Arbitrary Values}
\label{app:deterministic}

\begin{theorem}[Cross-talk bound --- arbitrary keys, arbitrary values]
\label{lem:deterministic}
With the summary statistics (\eqref{eq:summary-stats}),
\begin{equation}
\widetilde{z}_{\max}\ \le\ \sqrt{\widetilde{E}_K}\,\sqrt{\widetilde{E}_v}\,\widetilde{\kappa}.
\label{eq:deterministic-bound-app}
\end{equation}
\end{theorem}

\begin{proof}
Fix $i\in[F]$ and $j\neq i$.
If $\|\vX^{(ij)}\|_{2}=0$ or $\|\vY^{(ij)}\|_{2}=0$, then $\widetilde{z}_{ij}=0$ and $\widetilde{\kappa}^{ij}=0$, so the claim is trivial.
Otherwise, by definition of $\widetilde{\kappa}^{ij}$ in \eqref{eq:kappaij},
\[
|\widetilde{z}_{ij}| = |\ip{\vX^{(ij)}}{\vY^{(ij)}}| = \|\vX^{(ij)}\|_{2}\,\|\vY^{(ij)}\|_{2}\,\widetilde{\kappa}^{ij}
= \sqrt{\widetilde{E}_{K}(i,j)}\,\sqrt{\widetilde{E}_{v}(i,j)}\,\widetilde{\kappa}^{ij}.
\]
Using $\widetilde{E}_{K}(i,j)\le \widetilde{E}_{K}$, $\widetilde{E}_{v}(i,j)\le \widetilde{E}_{v}$, and $\widetilde{\kappa}^{ij}\le\widetilde{\kappa}$ and taking $\max_{i\neq j}$ yields \eqref{eq:deterministic-bound-app}.
\end{proof}

\subsubsection{Arbitrary Keys, Isotropic Values}
\begin{theorem}[Cross-talk bound --- arbitrary keys, isotropic values]
\label{app:regime1}
Assume $\vv_{1},\dots,\vv_{F}$ are i.i.d.\ uniform on $\Sph\subset\R^{d}$.
Treat the kernel matrix $\hat{\matK}$ as arbitrary.
Fix $\delta\in(0,1)$ and let
\[
L\ \defeq\ \log\Big(\frac{C_{0} F^2}{\delta}\Big)
\]
for a sufficiently large absolute constant $C_{0}$.
Then with probability at least $1-\delta$ (over the values), the following hold:
\begin{align*}
\widetilde{E}_{v}
&\ \le\ C_{1}\,\frac{(F-2)+L}{d-1},\\
\widetilde{\kappa}
&\ \le\ C_{2}\,\sqrt{\frac{L}{F-2}}.
\end{align*}
Consequently, combining these summary-statistic bounds with Theorem~\ref{lem:deterministic} yields
\[
\widetilde{z}_{\max}
\ \le\
 C_{3}\,\sqrt{\widetilde{E}_K}\,\sqrt{\frac{L}{d-1}}\,\sqrt{1+\frac{L}{F-2}}.
\]
In particular, if $F-2\ge L$ and $d\ge 2$, then
\[
\widetilde{z}_{\max}
\ \le\
 C_{4}\,\sqrt{\widetilde{E}_K}\,\sqrt{\frac{L}{d}}.
\]
\end{theorem}

\begin{proof}
The proof combines two auxiliary concentration bounds via a union bound.

\paragraph{Concentration of $\widetilde{E}_{v}$ (invoke Lemma~\ref{lem:Ev-iso-values}).}
Apply Lemma~\ref{lem:Ev-iso-values} with failure probability $\delta/2$.
This yields an event $\mathcal{E}_{1}$ such that
$\Pbb(\mathcal{E}_{1}^c)\le \delta/2$ and on $\mathcal{E}_{1}$,
\(
\widetilde{E}_{v} \le C\frac{(F-2)+\widetilde L}{d-1}
\)
with $\widetilde L=\log(\tfrac{C_{0} F^2}{\delta/2})$.
Since $\widetilde L=\log(\tfrac{2C_{0} F^2}{\delta})\le \log(\tfrac{C_{0}' F^2}{\delta})$ for $C_{0}'=2C_{0}$,
we may rewrite the bound using $L=\log(\tfrac{C_{0}' F^2}{\delta})$ after adjusting the leading constant.

\paragraph{Concentration of $\widetilde{\kappa}$ (invoke Lemma~\ref{lem:kappa-iso-values}).}
Apply Lemma~\ref{lem:kappa-iso-values} with failure probability $\delta/2$.
This yields an event $\mathcal{E}_{2}$ such that
$\Pbb(\mathcal{E}_{2}^c)\le \delta/2$ and on $\mathcal{E}_{2}$,
\(
\widetilde{\kappa} \le C\sqrt{\tfrac{\widetilde L}{F-2}},
\)
with the same logarithmic factor $\widetilde L$ as above.
Again we rewrite in terms of $L=\log(\tfrac{C_{0} F^2}{\delta})$ by enlarging the absolute constant.

\paragraph{Union bound and plug-in.}
By a union bound,
\[
\Pbb(\mathcal{E}_{1}\cap\mathcal{E}_{2})\ \ge\ 1-\Pbb(\mathcal{E}_{1}^c)-\Pbb(\mathcal{E}_{2}^c)\ \ge\ 1-\delta.
\]
On $\mathcal{E}_{1}\cap\mathcal{E}_{2}$, plug the bounds on $\widetilde{E}_{v}$ and $\widetilde{\kappa}$ into
Theorem~\ref{lem:deterministic} to obtain the stated cross-talk bound.
If additionally $F-2\ge L$, then $\sqrt{1+L/(F-2)}\le \sqrt{2}$ and hence
\[
\widetilde{z}_{\max}\ \le\ C\,\sqrt{\widetilde{E}_K}\,\sqrt{\frac{L}{d-1}}\ \le\ C\,\sqrt{\widetilde{E}_K}\,\sqrt{\frac{L}{d}}
\]
which is exactly the final bound stated in the theorem.
\end{proof}

\subsubsection{Isotropic Keys, Arbitrary Values (Bilinear Kernel)}
\begin{lemma}[Mean$+$residual deterministic cross-talk decomposition]
\label{lem:mean-residual-app}
Fix a pair $(i,j)$ with $j\neq i$.
For any scalar $\mu\in\R$, define the \emph{centered} truncated column
\[
\vX^{\circ(ij)} \defeq \vX^{(ij)}-\mu\,\vone.
\]
Then
\[
\widetilde{z}_{ij}=\ip{\vX^{(ij)}}{\vY^{(ij)}}
= \mu\,\ip{\vone}{\vY^{(ij)}} + \ip{\vX^{\circ(ij)}}{\vY^{(ij)}}.
\]
Consequently, define
\[
B_{Y} \defeq \max_{i\neq j}\big|\ip{\vone}{\vY^{(ij)}}\big|,
\qquad
E_{K}^\circ(i,j) \defeq \|\vX^{\circ(ij)}\|_{2}^2,
\qquad
\kappa^{\circ(ij)} \defeq \frac{|\ip{\vX^{\circ(ij)}}{\vY^{(ij)}}|}{\|\vX^{\circ(ij)}\|_{2}\,\|\vY^{(ij)}\|_2},
\]
with the convention $\kappa^{\circ(ij)}=0$ if $\|\vX^{\circ(ij)}\|_{2}\,\|\vY^{(ij)}\|_{2}=0$.
Also define the worst-case centered energy and centered coupling
\[
E_{K}^\circ \defeq \max_{i\neq j}E_{K}^\circ(i,j),
\qquad
\kappa^\circ \defeq \max_{i\neq j}\kappa^{\circ(ij)}.
\]
Recalling $\widetilde{E}_{v}$ from \eqref{eq:summary-stats}, we have the deterministic bound
\[
\widetilde{z}_{\max}\ \le\ |\mu|\,B_{Y}\ +\ \sqrt{E_{K}^\circ}\,\sqrt{\widetilde{E}_v}\,\kappa^\circ.
\]
Moreover, an ``effective coupling'' form follows by taking the maximum
\emph{before} separating the kernel-side summary statistics:
\[
\widetilde{z}_{\max}
\ \le\
|\mu|\,B_{Y}
\ +\
\sqrt{\widetilde{E}_v}\,\max_{i\neq j}\Big(\sqrt{E_{K}^\circ(i,j)}\,\kappa^{\circ(ij)}\Big).
\]
Defining
\[
\kappa_{\mathrm{eff}}^\circ
\ \defeq\
\max_{i\neq j}\Big(\sqrt{E_{K}^\circ(i,j)}\,\kappa^{\circ(ij)}\Big)
\ =\
\max_{i\neq j}\frac{|\ip{\vX^{\circ(ij)}}{\vY^{(ij)}}|}{\|\vY^{(ij)}\|_2}
\]
we have the deterministic bound
\[
\widetilde{z}_{\max}\ \le\ |\mu|\,B_{Y}\ +\ \sqrt{\widetilde{E}_v}\,\kappa_{\mathrm{eff}}^\circ.
\]
\end{lemma}

\begin{proof}
The decomposition follows by substituting
\(\vX^{(ij)}=\mu\,\vone+\vX^{\circ(ij)}\)
into \(\widetilde{z}_{ij}=\ip{\vX^{(ij)}}{\vY^{(ij)}}\).

For the bound, fix $(i,j)$. By the triangle inequality,
\[
|\widetilde{z}_{ij}|\ \le\ |\mu|\,\big|\ip{\vone}{\vY^{(ij)}}\big|\ +\ \big|\ip{\vX^{\circ(ij)}}{\vY^{(ij)}}\big|.
\]
By definition of the pairwise coupling $\kappa^{\circ(ij)}$,
\(
\big|\ip{\vX^{\circ(ij)}}{\vY^{(ij)}}\big|
\le
\|\vX^{\circ(ij)}\|_{2}\,\|\vY^{(ij)}\|_{2}\,\kappa^{\circ(ij)}.
\)
Since $\|\vY^{(ij)}\|_{2}\le \sqrt{\widetilde{E}_v}$ and $\|\vX^{\circ(ij)}\|_{2}=\sqrt{E_{K}^\circ(i,j)}$, we obtain
\[
|\widetilde{z}_{ij}|
\ \le\
|\mu|\,B_{Y}
\ +\
\sqrt{\widetilde{E}_v}\,\sqrt{E_{K}^\circ(i,j)}\,\kappa^{\circ(ij)}.
\]
Taking $\max_{i\neq j}$ gives
\[
\widetilde{z}_{\max}
\ \le\
|\mu|\,B_{Y}
\ +\
\sqrt{\widetilde{E}_v}\,\max_{i\neq j}\Big(\sqrt{E_{K}^\circ(i,j)}\,\kappa^{\circ(ij)}\Big)
\ =\
|\mu|\,B_{Y}\ +\ \sqrt{\widetilde{E}_v}\,\kappa_{\mathrm{eff}}^\circ.
\]
Finally, since $\sqrt{E_{K}^\circ(i,j)}\le \sqrt{E_{K}^\circ}$ and $\kappa^{\circ(ij)}\le \kappa^\circ$ for all $(i,j)$,
we also have
\(
\kappa_{\mathrm{eff}}^\circ\le \sqrt{E_{K}^\circ}\,\kappa^\circ,
\)
which yields the stated bound
\(
\widetilde{z}_{\max}\le |\mu|\,B_{Y}+\sqrt{E_{K}^\circ}\,\sqrt{\widetilde{E}_v}\,\kappa^\circ.
\)
\end{proof}

\begin{theorem}[Cross-talk bound --- isotropic keys, arbitrary values (bilinear kernel)]
\label{app:regime2}
Assume keys $\vk_{1},\dots,\vk_{F}$ are i.i.d.\ uniform on $\Sph$ and independent of the feature weights
$(\va_{r},\vb_{r})_{r=1}^m$, as defined in the bilinear random features setup (\cref{lem:bilinear-sketched-k2}).
Fix \emph{deterministic} values $\{\vv_{t}\}_{t=1}^F\subset\R^{d}$ and codes $\{\vc_{t}\}_{t=1}^F\subset\R^{d}$.
Fix $\delta\in(0,1)$ and set $L\defeq \log\big(\tfrac{C_{0} F^2}{\delta}\big)$.
Assume $m\le d^2/L$ and let $\mu=1/d$.
Let $\sigma^2\defeq \E\big[(\hat{\matK}_{ti}-\mu)^2\big]=\Theta(\tfrac{1}{d^2}+\tfrac{1}{m})$.
Let $L_v$ be as in \cref{def:Lv}.
Then with probability at least $1-\delta$ (over keys and features),
\begin{equation}
\widetilde{z}_{\max}
\ \le\ \frac{1}{d}\,B_{Y}
\ +\ C\,\sqrt{\widetilde{E}_v}\left(\sigma\sqrt{L}+\frac{L^2}{m}\right)\sqrt{L_v}.
\label{eq:regime2-cross-talk-keff}
\end{equation}
In particular, under $m\le d^2/L$ one has $\sigma\asymp 1/\sqrt{m}$, so
\[
\widetilde{z}_{\max}
\ \le\ \frac{1}{d}\,B_{Y}
\ +\ C\,\sqrt{\widetilde{E}_v}\left(\sqrt{\frac{L}{m}}+\frac{L^2}{m}\right)\sqrt{L_v}.
\]
\end{theorem}

\begin{proof}
We use the effective-coupling reduction from Lemma~\ref{lem:mean-residual-app} and then invoke
Lemma~\ref{lem:keff-iso-keys}.

\paragraph{Deterministic reduction to $\kappa_{\mathrm{eff}}^\circ$.}
Lemma~\ref{lem:mean-residual-app} (with $\mu=1/d$) gives
\[
\widetilde{z}_{\max}
\ \le\
\frac{1}{d}\,B_{Y}
\ +\
\sqrt{\widetilde{E}_v}\,\kappa_{\mathrm{eff}}^\circ.
\]

\paragraph{Concentration of $\kappa_{\mathrm{eff}}^\circ$ (invoke Lemma~\ref{lem:keff-iso-keys}).}
Apply Lemma~\ref{lem:keff-iso-keys} with failure probability $\delta$ to obtain, with probability at least $1-\delta$,
\[
\kappa_{\mathrm{eff}}^\circ
\ \le\
C\left(\sigma\sqrt{L}+\frac{L^2}{m}\right)\sqrt{L_v}.
\]
Substitute this bound into the deterministic reduction to obtain \eqref{eq:regime2-cross-talk-keff}.
\end{proof}

\subsubsection{Isotropic Keys, Isotropic Values (Bilinear Kernel)}

\begin{theorem}[Cross-talk bound --- isotropic keys, isotropic values (bilinear kernel)]
\label{app:isoiso}
Assume \(F\ge3\). Keys \(\vk_1,\dots,\vk_F\) are i.i.d. uniform on \(\Sph\).
Values \(\vv_1,\dots,\vv_F\) are i.i.d. uniform on \(\Sph\) and independent of the
keys and features. The kernel is the bilinear random-feature kernel from
\cref{lem:bilinear-sketched-k2} with \(m\) features.

Fix \(\delta\in(0,1)\), and define
\[
L\defeq \log\left(\frac{C_0F^2}{\delta}\right).
\]
Let
\[
\sigma^2\defeq
\E[(\hat{\matK}_{ti}-1/d)^2]
=
\Theta\left(\frac1{d^2}+\frac1m\right),
\qquad t\neq i.
\]
Assume
\[
m\le \frac{d^2}{L},
\qquad
L^3\le c_0\sigma^2m^2.
\]
Then, with probability at least \(1-\delta\), the following hold simultaneously:
\begin{align*}
B_Y
&\ \le\ C_1\sqrt{\frac{FL}{d}},\\
\widetilde E_v
&\ \le\ C_2\,\frac{(F-2)+L}{d-1},\\
\sqrt{E_K^\circ}
&\ \le\ C_3\,\sigma\sqrt{(F-2)+L},\\
\kappa^\circ
&\ \le\ C_4\sqrt{\frac{L}{F-2}}.
\end{align*}
Consequently,
\[
\widetilde z_{\max}
\le
\frac1dB_Y+
\sqrt{E_K^\circ}\sqrt{\widetilde E_v}\,\kappa^\circ.
\]
In particular, if \(F-2\ge L\), then
\[
\widetilde z_{\max}
\le
C_6\sqrt L\sqrt{\frac{F}{d^3}}
+
C_7\sqrt L\sqrt{\frac{F}{md}}.
\]
\end{theorem}

\begin{proof}
\paragraph{Concentration and union bound.}
Apply Lemma~\ref{lem:BY-iso-values}, Lemma~\ref{lem:Ev-iso-values},
Lemma~\ref{lem:EK-iso-keys}, and Lemma~\ref{lem:kappa-centered-iso-values}, each
with failure probability \(\delta/4\). Increasing \(C_0\) if necessary lets all four
events be written with the same
\[
L=\log\left(\frac{C_0F^2}{\delta}\right).
\]
A union bound gives simultaneous validity with probability at least \(1-\delta\).

\paragraph{Mean--residual reduction.}
On this event, Lemma~\ref{lem:mean-residual-app} with \(\mu=1/d\) gives
\[
\widetilde z_{\max}
\le
\frac1dB_Y+
\sqrt{E_K^\circ}\sqrt{\widetilde E_v}\,\kappa^\circ.
\]

\paragraph{Mean term.}
The mean term satisfies
\[
\frac1dB_Y
\le
C\sqrt{\frac{FL}{d^3}}
=
C\sqrt L\sqrt{\frac{F}{d^3}}.
\]

\paragraph{Residual term.}
Assume now \(F-2\ge L\). By Lemma~\ref{lem:EK-iso-keys},
\[
\sqrt{E_K^\circ}\le C\sigma\sqrt{F-2}.
\]
Therefore
\[
\sqrt{E_K^\circ}\sqrt{\widetilde E_v}\,\kappa^\circ
\le
C\sigma\sqrt{F-2}
\sqrt{\frac{(F-2)+L}{d-1}}
\sqrt{\frac{L}{F-2}}.
\]
Canceling \(\sqrt{F-2}\),
\[
\sqrt{E_K^\circ}\sqrt{\widetilde E_v}\,\kappa^\circ
\le
C\sigma\sqrt L
\sqrt{\frac{(F-2)+L}{d-1}}.
\]
Since \(F-2\ge L\),
\[
(F-2)+L\le2(F-2)\le2F.
\]
Also \(d-1\asymp d\) after adjusting constants. Hence
\[
\sqrt{E_K^\circ}\sqrt{\widetilde E_v}\,\kappa^\circ
\le
C\sigma\sqrt{\frac{FL}{d}}.
\]
Finally, \(m\le d^2/L\) implies
\[
\frac1{d^2}\le \frac1{mL}\le \frac1m,
\]
so
\[
\sigma^2=\Theta\left(\frac1{d^2}+\frac1m\right)
\le
\frac{C}{m}.
\]
Thus
\[
\sigma\sqrt{\frac{FL}{d}}
\le
C\sqrt{\frac{FL}{md}}
=
C\sqrt L\sqrt{\frac{F}{md}}.
\]

\paragraph{Combine.}
Combining the mean and residual estimates gives the final bound.
\end{proof}

\subsection{Signal Bounds}
We now provide lower bounds for the signal term $\widetilde{s}_{ij}$.

\paragraph{Signal as difference of weighted kernel values.}
Recall that under the ``absorbing the competitor'' convention (\eqref{eq:new-signal-crosstalk}), the two-term signal is
\begin{equation}
\widetilde{s}_{ij}
=\hat{\matK}_{ii}\,\ip{\vv_{i}-\vv_j}{\vv_i}
+\hat{\matK}_{ji}\,\ip{\vv_{i}-\vv_j}{\vv_j}
=\hat{\matK}_{ii}\,\ip{\vv_{i}-\vv_j}{\vv_i}
-\hat{\matK}_{ji}\,\ip{\vv_{j}-\vv_i}{\vv_j}.
\label{eq:sij-two-term}
\end{equation}
Since $\hat{\matK}$ is a Gram matrix, we always have
\(
\hat{\matK}_{ii}=\|\phi(\vk_{i})\|_{2}^2\ge 0.
\)
Further, define the smallest (worst-case) signal by
\begin{equation}
\widetilde{s}_{\min}\ \defeq\ \min_{i\neq j}\widetilde{s}_{ij}.
\label{eq:smin}
\end{equation}

\paragraph{Signal summary statistics.}
Define the value-side inner products
\begin{equation}
V_{ij}\ \defeq\ \ip{\vv_{i}-\vv_j}{\vv_i},
\qquad
B_{ij}\ \defeq\ \ip{\vv_{i}-\vv_j}{\vv_j},
\qquad (i\neq j).
\label{eq:AijBij}
\end{equation}
Across our bounds below we will assume, unless stated otherwise,
\begin{equation}
V_{ij}\ \ge\ 0
\qquad\text{for all }i\neq j.
\label{eq:Apos}
\end{equation}
This holds automatically when all values share a common norm:
if $\|\vv_{i}\|_{2}=\|\vv_{j}\|_{2}=r$ then by Cauchy--Schwarz
\(
\ip{\vv_i}{\vv_j}\le \|\vv_{i}\|_{2}\|\vv_{j}\|_{2}=r^2,
\)
hence
\(
V_{ij}=r^2-\ip{\vv_i}{\vv_j}\ge 0.
\)

Recall from \cref{def:Kdiag,def:Vmin} the kernel and value-side extrema:
\begin{equation}
K_{\min}^{\mathrm{diag}}\ =\ \min_{i\in[F]} \hat{\matK}_{ii},
\qquad
K_{\max}^{\mathrm{off}}\ =\ \max_{i\neq j}|\hat{\matK}_{ij}|,
\label{eq:KminKmax}
\end{equation}
\begin{equation}
V_{\min}\ =\ \min_{i\neq j}\ip{\vv_{i}-\vv_j}{\vv_i},
\qquad
V_{\max}\ =\ \max_{i\neq j}\big|\ip{\vv_{i}-\vv_j}{\vv_j}\big|.
\label{eq:AminBmax}
\end{equation}

\subsubsection{Arbitrary Keys, Arbitrary Values}
\label{app:signal-det-arb-arb}

\begin{theorem}[Deterministic signal lower bound --- arbitrary keys, arbitrary values]
Under the positivity condition (\eqref{eq:Apos}), and with the summary statistics (\cref{def:Kdiag,def:Vmin}),
\[
\widetilde{s}_{\min}\ \ge\ K_{\min}^{\mathrm{diag}}\,V_{\min}\ -\ K_{\max}^{\mathrm{off}}\,V_{\max}.
\]
\end{theorem}

\begin{proof}
Fix a pair $(i,j)$ with $i\neq j$.
From \eqref{eq:sij-two-term} and \eqref{eq:AijBij},
\[
\widetilde{s}_{ij}=\hat{\matK}_{ii}\,V_{ij}+\hat{\matK}_{ji}\,B_{ij}.
\]

By definition, $\hat{\matK}_{ii}\ge K_{\min}^{\mathrm{diag}}$.
Also $V_{ij}\ge V_{\min}$ by definition of $V_{\min}$.
Under the positivity assumption (\eqref{eq:Apos}), we have $V_{ij}\ge 0$ for all $i\neq j$, and therefore
multiplying the inequalities $\hat{\matK}_{ii}\ge K_{\min}^{\mathrm{diag}}$ and $V_{ij}\ge V_{\min}$ preserves order:
\[
\hat{\matK}_{ii}\,V_{ij}\ \ge\ K_{\min}^{\mathrm{diag}}\,V_{\min}.
\]

For the second term, we use $xy\ge -|x|\,|y|$:
\[
\hat{\matK}_{ji}\,B_{ij}\ \ge\ -|\hat{\matK}_{ji}|\,|B_{ij}|.
\]
By definition $|\hat{\matK}_{ji}|\le K_{\max}^{\mathrm{off}}$ and $|B_{ij}|\le V_{\max}$, hence
\[
\hat{\matK}_{ji}\,B_{ij}\ \ge\ -K_{\max}^{\mathrm{off}}\,V_{\max}.
\]

Adding the two bounds yields
\(
\widetilde{s}_{ij}\ge K_{\min}^{\mathrm{diag}}V_{\min}-K_{\max}^{\mathrm{off}}V_{\max},
\)
resulting in the claimed bound for $\widetilde{s}_{\min}$.
\end{proof}

\subsubsection{Arbitrary Keys, Isotropic Values}
\begin{theorem}[Signal lower bound --- arbitrary keys, isotropic values]
\label{app:signal-arbkeys-isovalues}
Assume $\vv_{1},\dots,\vv_{F}$ are i.i.d.\ uniform on $\Sph\subset\R^d$, and treat the kernel matrix $\hat{\matK}$ as arbitrary (or condition on it).
Fix $\delta\in(0,1)$ and let
\[
L\ \defeq\ \log\Big(\frac{C_{0} F^2}{\delta}\Big),
\qquad
\varepsilon_{v}\ \defeq\ C_{1}\sqrt{\frac{L}{d-1}}.
\]
Then with probability at least $1-\delta$ (over the values),
\[
\widetilde{s}_{\min}
\ \ge\
K_{\min}^{\mathrm{diag}}\,(1-\varepsilon_{v})
\ -\
K_{\max}^{\mathrm{off}}\,(1+\varepsilon_{v}),
\]
where $K_{\min}^{\mathrm{diag}}$ and $K_{\max}^{\mathrm{off}}$ are the kernel extrema (\cref{def:Kdiag}).
\end{theorem}

\begin{proof}
The proof relies on a single auxiliary concentration bound.

\paragraph{Concentration of the value-side extrema.}
Apply Lemma~\ref{lem:max-pairwise-ip-values} with failure probability $\delta$.
This yields an event $\mathcal{E}_{v}$ such that
$\Pbb(\mathcal{E}_{v}^c)\le \delta$ and on $\mathcal{E}_{v}$,
\[
\max_{i\neq j}\big|\ip{\vv_i}{\vv_j}\big|\ \le\ \varepsilon_{v}.
\]
Hence, on $\mathcal{E}_{v}$,
\[
V_{\min}\ \ge\ 1-\varepsilon_{v},
\qquad
V_{\max}\ \le\ 1+\varepsilon_{v}.
\]

\paragraph{Plug-in.}
On $\mathcal{E}_{v}$, the deterministic signal lower bound gives
\[
\widetilde{s}_{\min}\ \ge\ K_{\min}^{\mathrm{diag}}V_{\min}-K_{\max}^{\mathrm{off}}V_{\max}
\ \ge\ K_{\min}^{\mathrm{diag}}(1-\varepsilon_{v})-K_{\max}^{\mathrm{off}}(1+\varepsilon_{v}),
\]
which is exactly the claimed bound.
\end{proof}

\subsubsection{Isotropic Keys, Arbitrary Values (Bilinear Kernel)}
\begin{theorem}[Signal lower bound --- isotropic keys, arbitrary values (bilinear kernel)]
\label{app:signal-isokeys-arbvalues}
Assume keys $\vk_{1},\dots,\vk_{F}$ are i.i.d.\ uniform on $\Sph$, the kernel is the bilinear random-feature kernel from \cref{lem:bilinear-sketched-k2} with $m$ features, and the values $\{\vv_{t}\}_{t=1}^F$ are deterministic and satisfy the positivity condition (\eqref{eq:Apos}).
Let $V_{\min}$ and $V_{\max}$ be as in \cref{def:Vmin}.
Fix $\delta\in(0,1)$ and set $L\defeq \log(\tfrac{C_{0} F^2}{\delta})$.
Let $\mu\defeq 1/d$ and let $\sigma^2\defeq \E[(\hat{\matK}_{ij}-\mu)^2]=\Theta(\tfrac{1}{d^2}+\tfrac{1}{m})$
(see Lemma~\ref{lem:bilinear-mean-var}).
Then with probability at least $1-\delta$ (over keys and features),
\[
\widetilde{s}_{\min}
\ \ge\
\Big(1-C_{1}\sqrt{\tfrac{L}{m}}\Big)\,V_{\min}
\ -\
\Big(\mu+C_{2}(\sigma\sqrt{L}+L^2/m)\Big)\,V_{\max}.
\]
Moreover, if $L\ge 1$ and $m \geq L^3$,
then the off-diagonal coefficient simplifies and
\[
\widetilde{s}_{\min}
\ \ge\
\Big(1-C_{1}\sqrt{\tfrac{L}{m}}\Big)\,V_{\min}
\ -\
C_{3}\,\sigma\sqrt{L}\,V_{\max}.
\]
\end{theorem}

\begin{proof} The proof combines two auxiliary concentration bounds via a union bound.
\paragraph{Concentration of the diagonal term.}
Apply Lemma~\ref{lem:Kdiag-lower-bilinear} with failure probability $\delta/2$.
This yields an event $\mathcal{E}_{1}$ such that $\Pbb(\mathcal{E}_{1}^c)\le \delta/2$ and on $\mathcal{E}_{1}$,
\[
K_{\min}^{\mathrm{diag}}\ \ge\ 1-C\sqrt{\frac{L}{m}},
\]
after adjusting the absolute constant hidden in $L=\log(\tfrac{C_{0} F^2}{\delta})$.

\paragraph{Concentration of the off-diagonal term.}
Apply Lemma~\ref{lem:Koff-iso-keys} with failure probability $\delta/2$.
This yields an event $\mathcal{E}_{2}$ such that $\Pbb(\mathcal{E}_{2}^c)\le \delta/2$ and on $\mathcal{E}_{2}$,
\[
K_{\max}^{\mathrm{off}}\ \le\ \mu+C\Big(\sigma\sqrt{L}+\frac{L^2}{m}\Big),
\]
again after adjusting the absolute constant hidden in $L$.

\paragraph{Union bound and plug-in.}
By a union bound,
\[
\Pbb(\mathcal{E}_{1}\cap\mathcal{E}_{2})\ \ge\ 1-\Pbb(\mathcal{E}_{1}^c)-\Pbb(\mathcal{E}_{2}^c)\ \ge\ 1-\delta.
\]
On $\mathcal{E}_{1}\cap\mathcal{E}_{2}$, the deterministic signal lower bound gives
\[
\widetilde{s}_{\min}\ \ge\ K_{\min}^{\mathrm{diag}}V_{\min}-K_{\max}^{\mathrm{off}}V_{\max}.
\]
Substituting the bounds for $K_{\min}^{\mathrm{diag}}$ and $K_{\max}^{\mathrm{off}}$ gives, after renaming absolute constants,
\[
\widetilde{s}_{\min}
\ \ge\
\Big(1-C_{1}\sqrt{\tfrac{L}{m}}\Big)V_{\min}
\ -\
\Big(\mu+C_{2}(\sigma\sqrt{L}+L^2/m)\Big)V_{\max},
\]
which is exactly the stated lower bound.

\paragraph{Simplification of the off-diagonal coefficient.}
If \(L\ge 1\) and $m \geq L^3$, then
\(
\sigma^2=\Theta(\tfrac{1}{d^2}+\tfrac{1}{m})
\)
implies
\(
\sigma\gtrsim 1/d=\mu,
\)
so
\(
\mu\le C\sigma\le C\sigma\sqrt{L}.
\)
Hence
\[
\mu+C\Big(\sigma\sqrt{L}+\frac{L^2}{m}\Big)\ \le\ C\sigma\sqrt{L},
\]
and substituting this into the previous bound yields the refined inequality.
\end{proof}

\subsubsection{Isotropic Keys, Isotropic Values (Bilinear Kernel)}
\begin{theorem}[Signal lower bound --- isotropic keys, isotropic values (bilinear kernel)]
\label{app:signal-isokeys-isovalues}
Assume keys $\vk_{1},\dots,\vk_{F}$ are i.i.d.\ uniform on $\Sph$, values $\vv_{1},\dots,\vv_{F}$ are i.i.d.\ uniform on $\Sph$ and independent of the keys and features, and the kernel is the bilinear random-feature kernel from \cref{lem:bilinear-sketched-k2} with $m$ features.
Fix $\delta\in(0,1)$ and set $L\defeq \log(\tfrac{C_{0} F^2}{\delta})$.
Assume $m\le d^2/L$.
Let $\mu\defeq 1/d$ and $\sigma^2\defeq \E[(\hat{\matK}_{ij}-\mu)^2]=\Theta(\tfrac{1}{d^2}+\tfrac{1}{m})$.
Then with probability at least $1-\delta$ (over keys, features, and values),
\[
\widetilde{s}_{\min}
\ \ge\
\Big(1-C_{1}\sqrt{\tfrac{L}{m}}\Big)\,(1-\varepsilon_{v})
\ -\
\Big(\mu+C_{2}(\sigma\sqrt{L}+L^2/m)\Big)\,(1+\varepsilon_{v}),
\qquad
\varepsilon_{v}=C_{3}\sqrt{\tfrac{L}{d-1}}.
\]
Moreover, if \(L\ge 1\) and $m \geq L^3$, then
\[
\widetilde{s}_{\min}
\ \ge\
\Big(1-C_{1}\sqrt{\tfrac{L}{m}}\Big)\,(1-\varepsilon_{v})
\ -\
C_{4}\,\sigma\sqrt{L}\,(1+\varepsilon_{v}),
\qquad
\varepsilon_{v}=C_{3}\sqrt{\tfrac{L}{d-1}}.
\]
\end{theorem}

\begin{proof}
The proof combines three auxiliary concentration bounds via a union bound.
\paragraph{Concentration of the value-side extrema.}
Apply Lemma~\ref{lem:max-pairwise-ip-values} with failure probability $\delta/3$.
This yields an event $\mathcal{E}_{v}$ such that $\Pbb(\mathcal{E}_{v}^c)\le \delta/3$ and on $\mathcal{E}_{v}$,
\[
\max_{i\neq j}\big|\ip{\vv_i}{\vv_j}\big|\ \le\ \varepsilon_{v}.
\]
Hence, on $\mathcal{E}_{v}$,
\[
V_{\min}\ \ge\ 1-\varepsilon_{v},
\qquad
V_{\max}\ \le\ 1+\varepsilon_{v}.
\]

\paragraph{Concentration of the diagonal term.}
Apply Lemma~\ref{lem:Kdiag-lower-bilinear} with failure probability $\delta/3$.
This yields an event $\mathcal{E}_{1}$ such that $\Pbb(\mathcal{E}_{1}^c)\le \delta/3$ and on $\mathcal{E}_{1}$,
\[
K_{\min}^{\mathrm{diag}}\ \ge\ 1-C\sqrt{\frac{L}{m}},
\]
after adjusting the absolute constant hidden in $L$.

\paragraph{Concentration of the off-diagonal term.}
Apply Lemma~\ref{lem:Koff-iso-keys} with failure probability $\delta/3$.
This yields an event $\mathcal{E}_{2}$ such that $\Pbb(\mathcal{E}_{2}^c)\le \delta/3$ and on $\mathcal{E}_{2}$,
\[
K_{\max}^{\mathrm{off}}\ \le\ \mu+C\Big(\sigma\sqrt{L}+\frac{L^2}{m}\Big),
\]
again after adjusting the absolute constant hidden in $L$.

\paragraph{Union bound and plug-in.}
By a union bound,
\[
\Pbb(\mathcal{E}_{v}\cap\mathcal{E}_{1}\cap\mathcal{E}_{2})
\ \ge\ 1-\Pbb(\mathcal{E}_{v}^c)-\Pbb(\mathcal{E}_{1}^c)-\Pbb(\mathcal{E}_{2}^c)
\ \ge\ 1-\delta.
\]
On this intersection, the deterministic signal lower bound gives
\[
\widetilde{s}_{\min}\ \ge\ K_{\min}^{\mathrm{diag}}V_{\min}-K_{\max}^{\mathrm{off}}V_{\max}.
\]
Substituting the bounds for $K_{\min}^{\mathrm{diag}}$, $K_{\max}^{\mathrm{off}}$, $V_{\min}$, and $V_{\max}$ yields the stated lower bound.

\paragraph{Simplification of the off-diagonal coefficient.}
If \(L\ge 1\) and $m \geq L^3$, then the same reasoning as above gives
\[
\mu+C\Big(\sigma\sqrt{L}+\frac{L^2}{m}\Big)\ \le\ C\sigma\sqrt{L},
\]
and substituting this estimate into the previous bound yields the refined inequality.
\end{proof}

\subsection{Margin Bounds}
We now combine our signal and cross-talk bounds to bound the Hebbian memory margin.
\subsubsection{Arbitrary Keys, Arbitrary Values}
\label{app:combined-deterministic}
\label{app:theory:margin_bounds:akav}
\phantomsection\label{app:theory:proof:arb-arb-det}
\begin{theorem}[Margin bound --- arbitrary keys, arbitrary values]
\label{thm:combined-deterministic}
Assume the positivity condition~\eqref{eq:Apos} (i.e., \(V_{ij}\ge 0\) for all \(i\neq j\)).
Let \(K_{\min}^{\mathrm{diag}}\), \(K_{\max}^{\mathrm{off}}\) (\cref{def:Kdiag}), \(V_{\min}\), and \(V_{\max}\) (\cref{def:Vmin}) be the signal-side summary statistics, and let \(\widetilde{E}_{K}\), \(\widetilde{E}_{v}\), and \(\widetilde{\kappa}\) be the
cross-talk summary statistics from \eqref{eq:summary-stats}.
Then
\begin{equation}
\gamma_{\min}\ \ge\ K_{\min}^{\mathrm{diag}}\,V_{\min}\ -\ K_{\max}^{\mathrm{off}}\,V_{\max}\ -\ \sqrt{\widetilde{E}_K}\,\sqrt{\widetilde{E}_v}\,\widetilde{\kappa}.
\label{eq:combined-det-bound}
\end{equation}
Moreover, if we instead keep the original signal/cross-talk decomposition (\eqref{eq:signal-crosstalk})
(i.e., without absorbing the competitor term), then applying Cauchy--Schwarz with the coupling factor
$\kappa$ from \cref{def:kappa} yields the alternative bound
\begin{equation}
\gamma_{\min}\ \ge\ K_{\min}^{\mathrm{diag}}\,V_{\min}\ -\ \sqrt{E_K}\,\sqrt{E_v}\,\kappa,
\label{eq:app:akav:thm}
\end{equation}
where \(E_{K}\) and \(E_{v}\) are as defined in \cref{def:Ecol,def:EY}.

\end{theorem}
\begin{proof}
Under the positivity condition (\eqref{eq:Apos}), the deterministic signal lower bound gives
\[
\widetilde{s}_{\min}\ \ge\ K_{\min}^{\mathrm{diag}}\,V_{\min}\ -\ K_{\max}^{\mathrm{off}}\,V_{\max}.
\]

Further, Theorem~\ref{lem:deterministic} yields
\[
\widetilde{z}_{\max}\ \le\ \sqrt{\widetilde{E}_K}\,\sqrt{\widetilde{E}_v}\,\widetilde{\kappa}.
\]

Recalling that \(\gamma_{\min}\ge \widetilde{s}_{\min}-\widetilde{z}_{\max}\), we therefore have
\[
\gamma_{\min}
\ \ge\ \widetilde{s}_{\min}-\widetilde{z}_{\max}
\ \ge\ K_{\min}^{\mathrm{diag}}\,V_{\min}
-
K_{\max}^{\mathrm{off}}\,V_{\max}
-
\sqrt{\widetilde{E}_K}\,\sqrt{\widetilde{E}_v}\,\widetilde{\kappa},
\]

Finally, to obtain \eqref{eq:app:akav:thm}, start from the original decomposition \eqref{eq:signal-crosstalk}:
\[
\gamma_{ij} \,=\, \hat{\matK}_{ii}\,V_{ij} \,+\, z_{ij},
\qquad
z_{ij}\ \defeq\ \sum_{t\neq i}\hat{\matK}_{ti}\,\ip{\vv_{i}-\vv_j}{\vc_t}.
\]
By definition of the coupling factor $\kappa^{ij}$ (\cref{def:kappa}),
\(
|z_{ij}|\ \le\ \sqrt{E_{K}(i)}\,\sqrt{E_{v}(i,j)}\,\kappa^{ij}
\ \le\ \sqrt{E_K}\,\sqrt{E_v}\,\kappa.
\)
Since $\hat{\matK}_{ii}\ge K_{\min}^{\mathrm{diag}}$ and $V_{ij}\ge V_{\min}$, taking $\min_{i\neq j}$ yields \eqref{eq:app:akav:thm}.

\end{proof}

\begin{corollary}[Arbitrary-embedding margin scaling --- simplified]\label{cor:bilinear-scaling-simplified}
Under the assumptions of \Cref{thm:combined-deterministic}, further reasonably assuming $d\ge 4\log F$ and using the penalty statistics $S_{\mathrm{sig}}$, $P_{\mathrm{key}}$, $P_{\mathrm{val}}$, $P_{\mathrm{align}}$ from \cref{sec:arbk-isov}, for some $C\in(0,1)$ we have
\[
\gamma_{\min}
\;\geq\;
C\,S_{\mathrm{sig}}
\;-\;
P_{\mathrm{key}}\,P_{\mathrm{val}}\,P_{\mathrm{align}}\sqrt{\frac{F\log F}{md}}.
\]
\end{corollary}
\begin{proof}
The statistics $E_K,E_v,\kappa$ (\cref{sec:arbk-isov}) sum over $t\neq i$ (the competitor $t=j$ included), so they coincide with the untruncated statistics of \cref{def:Ecol,def:EY,def:kappa} and \Cref{thm:combined-deterministic} (\cref{eq:app:akav:thm}) applies:
\(
\gamma_{\min}\ge K_{\min}^{\mathrm{diag}}\,V_{\min}-\sqrt{E_K}\,\sqrt{E_v}\,\kappa.
\)
Substituting $\sqrt{E_K}=P_{\mathrm{key}}\sqrt{F/m}$, $\sqrt{E_v}=P_{\mathrm{val}}\sqrt{F/d}$, and $\kappa=P_{\mathrm{align}}\sqrt{\log(F)/F}$ gives the exact identity
\[
\sqrt{E_K}\,\sqrt{E_v}\,\kappa
= P_{\mathrm{key}}\,P_{\mathrm{val}}\,P_{\mathrm{align}}\sqrt{\tfrac{F}{m}\cdot\tfrac{F}{d}\cdot\tfrac{\log F}{F}}
= P_{\mathrm{key}}\,P_{\mathrm{val}}\,P_{\mathrm{align}}\sqrt{\tfrac{F\log F}{md}}.
\]
For the signal, the definition of $S_{\mathrm{sig}}$ gives $K_{\min}^{\mathrm{diag}}\,V_{\min}=\big(1-\sqrt{\log(F)/d}\big)\,S_{\mathrm{sig}}$. Since $d\ge 4\log F$ forces $\sqrt{\log(F)/d}\le\tfrac12$, and $S_{\mathrm{sig}}\ge 0$ (as $K_{\min}^{\mathrm{diag}}\ge 0$ and $V_{\min}\ge 0$ by~\cref{eq:Apos}), we get $K_{\min}^{\mathrm{diag}}\,V_{\min}\ge C\,S_{\mathrm{sig}}$ with $C\defeq\tfrac12$. Combining the two equations proves the claim.
\end{proof}

\begin{corollary}[Arbitrary-embedding fact-storage capacity]\label{cor:optimal-fact-storage-capacity-arb-formal}
Under the assumptions of \Cref{cor:bilinear-scaling-simplified}, the margin lower bound is positive precisely when
\[
P_{\mathrm{key}}\,P_{\mathrm{val}}\,P_{\mathrm{align}}\sqrt{\frac{F\log F}{md}} < C\,S_{\mathrm{sig}},
\qquad\text{i.e.}\qquad
md > \frac{1}{C^2}\,F\log F\left(\frac{P_{\mathrm{key}} P_{\mathrm{val}} P_{\mathrm{align}}}{S_{\mathrm{sig}}}\right)^{\!2}.
\]
Hence our construction stores all $F$ facts with positive margin using a parameter budget
\[
W = md = \Theta\!\left(F\log F\left(\frac{P_{\mathrm{key}} P_{\mathrm{val}} P_{\mathrm{align}}}{S_{\mathrm{sig}}}\right)^{\!2}\right),
\]
the same rate fact-storage rate as in the isotropic setting up to the penalization factors.
\end{corollary}

\subsubsection{Arbitrary Keys, Isotropic Values}
\label{app:combined-regime1}
\label{app:theory:margin_bounds:akrv}
\phantomsection\label{app:theory:proof:arbk-isov}

\begin{theorem}[Margin bound --- arbitrary keys, isotropic values]\label{thm:combined-akrv}
Assume \(\vv_{1},\dots,\vv_{F}\) are i.i.d. uniform on \(\Sph\subset\R^d\), and treat the kernel matrix \(\hat{\matK}\) as arbitrary.
Fix \(\delta\in(0,1)\) and let
\[
L\ \defeq\ \log\!\Big(\frac{C_{0} F^2}{\delta}\Big),
\qquad
\varepsilon_{v}\ \defeq\ C_{1}\sqrt{\frac{L}{d-1}}.
\]
Define the (untruncated) kernel-column energy
\[
E_{K}\ \defeq\ \max_{i\in[F]}\ \sum_{t\neq i}\hat{\matK}_{ti}^2.
\]
Then with probability at least \(1-\delta\) (over the values),
\begin{equation}
\gamma_{\min}\ \ge\ K_{\min}^{\mathrm{diag}}(1-\varepsilon_{v})-K_{\max}^{\mathrm{off}} - C_{2}\sqrt{E_K}\,\sqrt{\frac{L}{d-1}}\,\sqrt{1+\frac{L}{F-2}}.
\label{eq:app:akrv:thm-full}
\end{equation}
In particular, if \(F-2\ge L\), then
\begin{equation}
\gamma_{\min}
\ \ge\
K_{\min}^{\mathrm{diag}}(1-\varepsilon_{v})-K_{\max}^{\mathrm{off}}
\ -\ C_{3}\sqrt{E_K}\,\sqrt{\frac{L}{d}}.
\label{eq:app:akrv:thm}
\end{equation}
\end{theorem}

\begin{proof}
The proof combines the signal and cross-talk bounds via a union bound.
\paragraph{Signal lower bound.}
Apply the signal bound from Theorem~\ref{app:signal-arbkeys-isovalues} with failure probability \(\delta/2\).
This yields an event \(\mathcal E_{\mathrm{sig}}\) such that
\(
\Pbb(\mathcal E_{\mathrm{sig}}^c)\le \delta/2
\)
and on \(\mathcal E_{\mathrm{sig}}\),
\begin{equation}
\widetilde{s}_{\min}
\ \ge\
K_{\min}^{\mathrm{diag}}(1-\varepsilon_{v})
\ -\
K_{\max}^{\mathrm{off}}(1+\varepsilon_{v}),
\label{eq:combined-reg1-signal}
\end{equation}
after the standard adjustment of the absolute constant hidden in
\(
L=\log(\tfrac{C_{0} F^2}{\delta})
\)
so that replacing \(\delta\) by \(\delta/2\) only changes the constant.

\paragraph{Cross-talk bound.}
Apply the cross-talk bound from Theorem~\ref{app:regime1} with failure probability \(\delta/2\).
This yields an event \(\mathcal E_{\mathrm{cross}}\) such that
\(
\Pbb(\mathcal E_{\mathrm{cross}}^c)\le \delta/2
\)
and on \(\mathcal E_{\mathrm{cross}}\),
\begin{equation}
\widetilde{z}_{\max}
\ \le\
C\sqrt{E_K}\,\sqrt{\frac{L}{d-1}}\,\sqrt{1+\frac{L}{F-2}}.
\label{eq:combined-reg1-xtalk}
\end{equation}
Here we used the trivial bound \(\widetilde{E}_{K}\le E_{K}\), since dropping coordinates can only decrease the $\ell_{2}$ norm.
Again we have absorbed the harmless \(\delta\mapsto \delta/2\) change into the absolute constant inside \(L\).

\paragraph{Union bound and combine.}
By a union bound,
\[
\Pbb(\mathcal E_{\mathrm{sig}}\cap \mathcal E_{\mathrm{cross}})
\ \ge\
1-\Pbb(\mathcal E_{\mathrm{sig}}^c)-\Pbb(\mathcal E_{\mathrm{cross}}^c)
\ \ge\ 1-\delta.
\]
On \(\mathcal E_{\mathrm{sig}}\cap \mathcal E_{\mathrm{cross}}\), use \(\gamma_{\min}\ge \widetilde{s}_{\min}-\widetilde{z}_{\max}\).
Substituting \eqref{eq:combined-reg1-signal} and \eqref{eq:combined-reg1-xtalk} yields
\begin{align}
\gamma_{\min}
&\ge
K_{\min}^{\mathrm{diag}}(1-\varepsilon_{v})
-
K_{\max}^{\mathrm{off}}(1+\varepsilon_{v})
-
C_{1}\sqrt{E_K}\,\sqrt{\frac{L}{d-1}}\,\sqrt{1+\frac{L}{F-2}}
\notag\\
&=
K_{\min}^{\mathrm{diag}}(1-\varepsilon_{v})
-
K_{\max}^{\mathrm{off}}
-
\Bigg[
K_{\max}^{\mathrm{off}}\,\varepsilon_{v}
+
C_{1}\sqrt{E_K}\,\sqrt{\frac{L}{d-1}}\,\sqrt{1+\frac{L}{F-2}}
\Bigg].
\label{eq:combined-reg1-before-absorb}
\end{align}

\paragraph{Absorb the competitor term.}
We claim that
\begin{equation}
K_{\max}^{\mathrm{off}}\ \le\ \sqrt{E_K}.
\label{eq:Koff-le-sqrtEK}
\end{equation}
Choose indices \((a,b)\) with \(a\neq b\) such that
\(
|\hat{\matK}_{ba}|=K_{\max}^{\mathrm{off}}.
\)
Then by definition of \(E_{K}\),
\[
E_{K}
\ \ge\
\sum_{t\neq a}\hat{\matK}_{ta}^2
\ \ge\
\hat{\matK}_{ba}^2
\ =\
\big(K_{\max}^{\mathrm{off}}\big)^2,
\]
which proves \eqref{eq:Koff-le-sqrtEK}.
Therefore,
\[
K_{\max}^{\mathrm{off}}\,\varepsilon_{v}
\ \le\
\sqrt{E_K}\,\varepsilon_{v}
\ =\
C_{2}\sqrt{E_K}\,\sqrt{\frac{L}{d-1}}
\ \le\
C_{3}\sqrt{E_K}\,\sqrt{\frac{L}{d-1}}\,\sqrt{1+\frac{L}{F-2}},
\]
since \(\sqrt{1+L/(F-2)}\ge 1\).
Absorbing this into the bracket in \eqref{eq:combined-reg1-before-absorb} yields the claimed bound.

\paragraph{Simplified form when \(F-2\ge L\).}
If \(F-2\ge L\), then \(\sqrt{1+L/(F-2)}\le \sqrt{2}\).
Also \(d-1\asymp d\) for \(d\ge 2\), so after adjusting the absolute constant we obtain the simplified bound stated in the theorem.
\end{proof}

\subsubsection{Isotropic Keys, Arbitrary Values (Bilinear Kernel)}
\label{app:combined-regime2}
\label{app:theory:margin_bounds:rkav}
\phantomsection\label{app:theory:proof:isok-bilinear-hp}

\begin{theorem}[Margin bound --- isotropic keys, arbitrary values (bilinear kernel)]\label{thm:combined-rkav}
Assume keys \(\vk_{1},\dots,\vk_{F}\) are i.i.d. uniform on \(\Sph\), the kernel is the bilinear random-feature kernel from \cref{lem:bilinear-sketched-k2} with $m$ features, the values \(\vv_{1},\dots,\vv_{F}\in\R^d\) are deterministic, and the positivity condition \(\ip{\vv_{i}-\vv_j}{\vv_i}\ge 0\) holds for all \(i\neq j\).
Let
\[
\mu\ \defeq\ \frac{1}{d},
\qquad
\sigma^2\ \defeq\ \E[(\hat{\matK}_{ti}-\mu)^2]=\Theta\!\Big(\frac{1}{d^2}+\frac{1}{m}\Big),
\]
and fix \(\delta\in(0,1)\). Set
\[
L\ \defeq\ \log\!\Big(\frac{C_{0} F^2}{\delta}\Big), \qquad \varepsilon_k \defeq \sqrt{\frac{L}{m}}
\]
and assume in addition that \(F-2\ge L\), \(m \geq L^3\), and \(m\le d^2/L\).

Recall the centered-sum budget \(B_Y\) (\cref{def:BY}) and let \(L_v\) be as in \cref{def:Lv}.
Then with probability at least \(1-\delta\) (over the keys and features),
\begin{equation}
\gamma_{\min}
\ \ge\
\Big(1-C_{1}\varepsilon_k\Big)V_{\min}
-
C_{2}\Big(B_{Y}+\sqrt{E_v}\,\sqrt{L_v}\Big)
\Big(\sigma\sqrt{L}\Big).
\label{eq:app:rkav:thm}
\end{equation}
\end{theorem}

\begin{proof}
The proof combines the signal and cross-talk bounds via a union bound.
\paragraph{Signal lower bound.}
Apply the signal bound from Theorem~\ref{app:signal-isokeys-arbvalues} with failure probability \(\delta/2\).
This yields an event \(\mathcal E_{\mathrm{sig}}\) such that
\(
\Pbb(\mathcal E_{\mathrm{sig}}^c)\le \delta/2
\)
and on \(\mathcal E_{\mathrm{sig}}\),
\begin{equation}
\widetilde{s}_{\min}
\ \ge\
\Big(1-C_{1}\sqrt{\tfrac{L}{m}}\Big)V_{\min}
-
\Big(\frac{1}{d}+C_{2}\big(\sigma\sqrt{L}+\tfrac{L^2}{m}\big)\Big)V_{\max}.
\label{eq:combined-reg2-signal}
\end{equation}
As before, replacing \(\delta\) by \(\delta/2\) only changes the absolute constant hidden in \(L\).

\paragraph{Cross-talk bound.}
Apply the cross-talk bound from Theorem~\ref{app:regime2} with failure probability \(\delta/2\).
This yields an event \(\mathcal E_{\mathrm{cross}}\) such that
\(
\Pbb(\mathcal E_{\mathrm{cross}}^c)\le \delta/2
\)
and on \(\mathcal E_{\mathrm{cross}}\),
\begin{equation}
\widetilde{z}_{\max}
\ \le\
\frac{1}{d}B_{Y}
\ +\ C_{1}\sqrt{\widetilde{E}_v}\Big(\sigma\sqrt{L}+\frac{L^2}{m}\Big)\sqrt{L_v}.
\label{eq:combined-reg2-xtalk}
\end{equation}
After adjusting the absolute constant,
\begin{equation}
\widetilde{z}_{\max}
\ \le\
\frac{1}{d}B_{Y}
\ +\ C_{2}\sqrt{\widetilde{E}_v}\Big(\sigma\sqrt{L}+\frac{L^2}{m}\Big)\sqrt{L_v}.
\label{eq:combined-reg2-xtalk-simple}
\end{equation}

\paragraph{Union bound and combine.}
By a union bound,
\[
\Pbb(\mathcal E_{\mathrm{sig}}\cap\mathcal E_{\mathrm{cross}})\ \ge\ 1-\delta.
\]
On this intersection, use \(\gamma_{\min}\ge \widetilde{s}_{\min}-\widetilde{z}_{\max}\).
Substituting \eqref{eq:combined-reg2-signal} and \eqref{eq:combined-reg2-xtalk-simple} gives
\begin{align}
\gamma_{\min}
&\ge
\Big(1-C_{1}\sqrt{\tfrac{L}{m}}\Big)V_{\min}
-
\frac{1}{d}(B_{Y}+V_{\max})
\notag\\
&\hspace{1.6em}
-
C_{2}\Big(V_{\max}+\sqrt{\widetilde{E}_v}\,\sqrt{L_v}\Big)\Big(\sigma\sqrt{L}+\frac{L^2}{m}\Big).
\label{eq:combined-reg2-combined}
\end{align}

\paragraph{Absorb the value-side terms.}
Since \(V_{\max}\le \sqrt{E_v}\) and \(\sqrt{\widetilde{E}_v}\le \sqrt{E_v}\), and \(L_{v}\ge1\) by \cref{def:Lv} (which holds even in the degenerate case \(\widetilde{E}_{v}=0\), where the cross-talk term vanishes but the signal-side \(V_{\max}\) need not, by the convention \(L_v=1\)), we have \(V_{\max}+\sqrt{\widetilde{E}_v}\sqrt{L_v}\le \sqrt{E_v}\sqrt{L_v}+\sqrt{E_v}\sqrt{L_v}= 2\sqrt{E_v}\sqrt{L_v}\). Therefore, after adjusting the absolute constant,
\[
C_{2}\Big(V_{\max}+\sqrt{\widetilde{E}_v}\,\sqrt{L_v}\Big)\Big(\sigma\sqrt{L}+\frac{L^2}{m}\Big)
\ \le\
C_{3}\sqrt{E_v}\,\sqrt{L_v}\Big(\sigma\sqrt{L}+\frac{L^2}{m}\Big).
\]

\paragraph{Absorb the \(d^{-1}\) factor.}
Since \(L=\log(\tfrac{C_{0} F^2}{\delta})\ge 1\) and
\(
\sigma^2=\Theta\!\big(\tfrac{1}{d^2}+\tfrac{1}{m}\big)
\)
implies \(\sigma\gtrsim 1/d\), we have \(\frac{1}{d}\le C\,\sigma\sqrt{L}\).
Therefore, using also \(V_{\max}\le \sqrt{E_v}\le \sqrt{E_v}\,\sqrt{L_v}\),
\[
\frac{1}{d}(B_{Y}+V_{\max})
\ \le\
C_{1}\Big(B_{Y}+\sqrt{E_v}\,\sqrt{L_v}\Big)\sigma\sqrt{L}
\ \le\
C_{2}\Big(B_{Y}+\sqrt{E_v}\,\sqrt{L_v}\Big)\Big(\sigma\sqrt{L}+\frac{L^2}{m}\Big).
\]
Plugging the last two estimates into \eqref{eq:combined-reg2-combined} and using $m \geq L^3$, we have
\(\gamma_{\min} \ge \Big(1-C_{4}\sqrt{\tfrac{L}{m}}\Big)V_{\min} - C_{5}\Big(B_{Y}+\sqrt{E_v}\,\sqrt{L_v}\Big)\sigma\sqrt{L}\),
which is exactly the stated bound.
\end{proof}

\subsubsection{Isotropic Keys, Isotropic Values (Bilinear Kernel)}
\label{app:combined-regime3}
\label{app:theory:margin_bounds:rkrv}
\phantomsection\label{app:theory:proof:exactK2}
\phantomsection\label{app:theory:proof:bilinear-scaling}
\phantomsection\label{app:theory:proof:bilinear-fact-capacity}

\begin{theorem}[Margin bound --- isotropic keys, isotropic values (bilinear kernel)]\label{thm:combined-rkrv}
Assume keys \(\vk_{1},\dots,\vk_{F}\) are i.i.d. uniform on \(\Sph\), values \(\vv_{1},\dots,\vv_{F}\) are i.i.d. uniform on \(\Sph\) and independent of the keys, and the kernel is the bilinear random-feature kernel from \cref{lem:bilinear-sketched-k2} with $m$ features.
Fix \(\delta\in(0,1)\) and set
\[
L\ \defeq\ \log\!\Big(\frac{C_{0} F^2}{\delta}\Big).
\]
Assume $F \ge 2L$, $m \ge L^3$, and $m\le d^2/L$. Then, with probability at least \(1-\delta\) (over keys, features, and values), the minimum ``absorbed-competitor'' margin satisfies
\begin{equation}
\gamma_{\min}
\ \ge\
1
\;-\;
C_{1}\sqrt{\frac{L}{m}}
\;-\;
\varepsilon_{v}
\;-\;
C_{3}\,\sqrt{L}\sqrt{\frac{F}{d^3}}
\;-\;
C_{4}\,\sqrt{L}\sqrt{\frac{F}{md}},
\qquad
\varepsilon_{v}\defeq C_{2}\sqrt{\frac{L}{d-1}}.
\label{eq:app:rkrv:sketchedk2_full}
\end{equation}
\end{theorem}

\begin{proof}
\paragraph{Signal lower bound.}
Apply the signal bound from Theorem~\ref{app:signal-isokeys-isovalues} with failure probability $\delta/2$. Since $m\ge L^3$, this yields an event $\mathcal E_{\mathrm{sig}}$ with $\Pbb(\mathcal E_{\mathrm{sig}}^c)\le\delta/2$ on which
\[
\widetilde{s}_{\min}
\ \ge\
\Big(1-C_{1}\sqrt{\frac{L}{m}}\Big)(1-\varepsilon_{v})
\ -\
C_{4}\,\sigma\sqrt{L}\,(1+\varepsilon_{v}),
\]
with $\varepsilon_{v}$ as in the theorem statement, after the standard adjustment of the absolute constant hidden in $L$.

\paragraph{Cross-talk bound.}
Apply the cross-talk bound from Theorem~\ref{app:isoiso} with failure probability $\delta/2$. Since $F\ge 2L$ gives $F-2\ge L$, this yields an event $\mathcal E_{\mathrm{cross}}$ with $\Pbb(\mathcal E_{\mathrm{cross}}^c)\le\delta/2$ on which
\[
\widetilde{z}_{\max}
\ \le\
C_{6}\,\sqrt L\sqrt{\frac{F}{d^3}}
\ +\
C_{7}\,\sqrt L\sqrt{\frac{F}{md}}.
\]

\paragraph{Union bound and combine.}
By a union bound, $\Pbb(\mathcal E_{\mathrm{sig}}\cap\mathcal E_{\mathrm{cross}})\ge 1-\delta$. On this intersection, $\gamma_{\min}\ge\widetilde{s}_{\min}-\widetilde{z}_{\max}$ gives
\[
\gamma_{\min}
\ \ge\
\Big(1-C_{1}\sqrt{\frac{L}{m}}\Big)(1-\varepsilon_{v})
\ -\
C_{4}\,\sigma\sqrt{L}\,(1+\varepsilon_{v})
\ -\
C_{6}\,\sqrt L\sqrt{\frac{F}{d^3}}
\ -\
C_{7}\,\sqrt L\sqrt{\frac{F}{md}}.
\]
Since $m\le d^2/L$ forces $\sigma\asymp 1/\sqrt m$, and $\varepsilon_{v}=O(1)$, the signal off-diagonal penalty obeys $C_{4}\,\sigma\sqrt{L}\,(1+\varepsilon_{v})\le C\sqrt{L/m}$ and is absorbed into the first term. Expanding $\big(1-C\sqrt{L/m}\big)(1-\varepsilon_{v})\ge 1-C\sqrt{L/m}-\varepsilon_{v}$ and renaming the absolute constants yields \cref{eq:app:rkrv:sketchedk2_full}.
\end{proof}

\begin{corollary}[Isotropic margin scaling --- simplified]\label{cor:informal-iso-iso-margin}
Under the assumptions of \Cref{thm:combined-rkrv} and $m,d\le F$, retaining only the dominant cross-talk term gives, with probability at least $1-\delta$,
\[
\gamma_{\min}
\;\geq\;
1
\;-\;
C\sqrt{\frac{F\log(F/\delta)}{md}}.
\]
\end{corollary}
\begin{proof}
The dominant cross-talk term in \cref{eq:app:rkrv:sketchedk2_full} is $C_{4}\sqrt{L}\sqrt{F/(md)}$; under the standing assumptions, the remaining terms are lower order and are absorbed into the constant $C$. Since $L = \log(C_0 F^2/\delta) = \Theta(\log(F/\delta))$, we have $\sqrt{L}\sqrt{F/(md)} = \Theta\!\big(\sqrt{F\log(F/\delta)/(md)}\big)$, which yields the stated form.
\end{proof}

\begin{corollary}[Isotropic fact-storage capacity]\label{cor:optimal-fact-storage-capacity-formal}
Under the assumptions of \Cref{cor:informal-iso-iso-margin}, the margin lower bound is positive precisely when
\[
C\sqrt{\frac{F\log(F/\delta)}{md}} < 1,
\qquad\text{i.e.}\qquad
md > C^2\,F\log(F/\delta).
\]
Hence our construction stores all $F$ facts with positive margin using a parameter budget
\[
W = md = \Theta\!\left(F\log F\right),
\]
the information-theoretically optimal fact-storage rate.
\end{corollary}

\subsubsection{Summary Table: Margin Bounds Across Regimes}

\begin{table}[H]
\centering
\renewcommand{\arraystretch}{1.5}
\begin{tabular}{c|c|c}
&
\makecell{\textbf{Random}\ \textbf{keys}}
&
\makecell{\textbf{Arbitrary}\ \textbf{keys}}
\\ \hline
\makecell{\textbf{Random}\\ \textbf{values}}
&
$\begin{aligned}
\gamma_{\min} &\ge\; 1 \\
&-\, C\sqrt{\frac{FL}{md}}
\end{aligned}$
&
$\begin{aligned}
\gamma_{\min} &\ge\; K_{\min}^{\mathrm{diag}}(1-\varepsilon_{v})-K_{\max}^{\mathrm{off}} \\
&-\, C\sqrt{E_K}\,\sqrt{\frac{L}{d}}
\end{aligned}$
\\ \hline
\makecell{\textbf{Arbitrary}\\ \textbf{values}}
&
$\begin{aligned}
\gamma_{\min} &\ge\; \Big(1-\varepsilon_{k}\Big)V_{\min} \\
&-\, C\!\left(B_Y+\sqrt{E_v}\sqrt{L_v}\right)\!\sqrt{\frac{L}{m}}
\end{aligned}$
&
$\begin{aligned}
\gamma_{\min} &\ge\; K_{\min}^{\mathrm{diag}}V_{\min} \\
&-\, \sqrt{E_K}\,\sqrt{E_v}\,\kappa
\end{aligned}$
\end{tabular}
\caption{\textbf{Margin bounds across key/value geometry regimes.}
Each cell shows the decoding margin scaling for a given key/value geometry.
The top-left (isotropic) cell is the baseline.
Relaxing the key geometry introduces key-geometry statistics ($K_{\min}^{\mathrm{diag}}$, $K_{\max}^{\mathrm{off}}$, $E_{K}$);
relaxing value geometry introduces value-geometry statistics ($V_{\min}$, $B_{Y}$, $V_{\max}$, $E_{v}$, $L_{v}$);
relaxing both geometries introduces a coupling factor $\kappa$.
The arbitrary keys and values cell uses the raw summary-statistic form of \Cref{thm:combined-deterministic}, which is equivalent to the penalty-statistic general margin/capacity bound in \Cref{cor:optimal-fact-storage-capacity-arb-formal} of the main text.
See Theorems~\ref{thm:combined-deterministic}, \ref{thm:combined-akrv}, \ref{thm:combined-rkav} and \ref{thm:combined-rkrv} for formal statements and Appendix~\ref{app:exp:section4:margins} for empirical scaling results.}
\label{tab:2x2}
\end{table}

\input{appendix/theory/03_margin_bounds/02a_welch_upper_bound_iso}

\subsection{Auxiliary Results}

\subsubsection{Signal Bounds: Isotropic Values}
\label{app:max-pairwise-ip-values}

\begin{lemma}[Isotropic inner-product concentration]
\label{lem:iso-innerprod}
Let $\vv\sim\mathrm{Unif}(\Sph)$ with $d\ge 2$.
Then for any fixed $\vw\in\R^d$ and all $t\ge 0$,
$\|\ip{\vw}{\vv}\|_{\psi_2}\lesssim \|\vw\|_{2}/\sqrt{d-1}$ and
$\big\|\ip{\vw}{\vv}^2-\E[\ip{\vw}{\vv}^2]\big\|_{\psi_1}\ \lesssim\ \frac{\norm{\vw}_{2}^2}{d-1}$. I.e.\ $\ip{\vw}{\vv}$ is sub-Gaussian and $\ip{\vw}{\vv}^2-\E[\ip{\vw}{\vv}^2]$ is sub-exponential, with the corresponding scale parameters.
\end{lemma}
\begin{proof}
    This is a standard consequence of concentration of measure on the sphere, see \cite{vershynin2026high}.
\end{proof}

\begin{lemma}[Max pairwise inner product for isotropic values]
\label{lem:max-pairwise-ip-values}
Assume $\vv_{1},\dots,\vv_{F}$ are i.i.d.\ uniform on $\Sph$ with $d\ge 2$.
Fix $\delta\in(0,1)$ and set
\begin{equation}
L\ \defeq\ \log\Big(\frac{C_{0} F^2}{\delta}\Big)
\label{eq:L-signal-values}
\end{equation}
for a sufficiently large absolute constant $C_{0}$.
Then with probability at least $1-\delta$,
\begin{equation}
\max_{i\neq j}\big|\ip{\vv_i}{\vv_j}\big|
\ \le\
C_{1}\,\sqrt{\frac{L}{d-1}}.
\label{eq:maxpair-ip-values}
\end{equation}
\end{lemma}

\begin{proof}
Fix a pair $(i,j)$ with $i\neq j$.
Condition on $\vv_{i}$ and define $\mathcal{F}_{i}\defeq\sigma(\vv_{i})$.
Then $\vv_{j}$ is still uniform on $\Sph$ and independent of $\mathcal{F}_{i}$.
Applying Lemma~\ref{lem:iso-innerprod} with $\vw=\vv_{i}$ (note $\|\vv_{i}\|_{2}=1$) yields that conditional on $\mathcal{F}_{i}$,
the scalar $\ip{\vv_i}{\vv_j}$ is centered sub-Gaussian with scale $\lesssim 1/\sqrt{d-1}$.
Equivalently, there exist absolute constants $c,C_{0}>0$ such that for all $t\ge 0$,
\begin{equation}
\Pbb\Big(\big|\ip{\vv_i}{\vv_j}\big|\ge t\ \Big|\ \mathcal{F}_{i}\Big)
\ \le\ 2e^{-c(d-1)t^2}.
\label{eq:pair-ip-tail}
\end{equation}
Since the bound from \eqref{eq:pair-ip-tail} does not depend on the realized $\vv_{i}$,
it also holds unconditionally:
\[
\Pbb\Big(\big|\ip{\vv_i}{\vv_j}\big|\ge t\Big)\ \le\ 2e^{-c(d-1)t^2}.
\]

Now take a union bound over all pairs $(i,j)$ with $i\neq j$ (at most $F^2$ pairs):
\[
\Pbb\Big(\max_{i\neq j}\big|\ip{\vv_i}{\vv_j}\big|\ge t\Big)
\ \le\
\sum_{i\neq j}\Pbb\Big(\big|\ip{\vv_i}{\vv_j}\big|\ge t\Big)
\ \le\
2F^2 e^{-c(d-1)t^2}.
\]
Choose
\(
t\defeq C\sqrt{L/(d-1)}
\)
with $C$ large enough so that $2F^2 e^{(-c(d-1)t^2)}\le\delta$.
This yields \eqref{eq:maxpair-ip-values}.
\end{proof}

\subsubsection{Signal Bounds: Isotropic Keys and Bilinear Random Features}
\label{app:Koff-iso-keys}

\begin{lemma}[Uniform off-diagonal bound (isotropic keys, bilinear kernel)]
\label{lem:Koff-iso-keys}
Assume keys are i.i.d.\ uniform on $\Sph$ and the kernel is the bilinear random-feature kernel from \cref{lem:bilinear-sketched-k2} with $m$ features.
Let $\mu\defeq\E[\hat{\matK}_{ij}]=1/d$ for $i\neq j$ and let $\sigma^2\defeq\E[(\hat{\matK}_{ij}-\mu)^2]$ as in
Lemma~\ref{lem:bilinear-mean-var}.
Fix $\delta\in(0,1)$ and define $L\defeq \log(\tfrac{C_{0} F^2}{\delta})$.
Assume that $m \le d^2/L$.
Then with probability at least $1-\delta$ (over keys and features),
\begin{equation}
\max_{i\neq j}\big|\hat{\matK}_{ij}-\mu\big|
\ \le\
C_{1}\Big(\sigma\sqrt{L}+\frac{L^2}{m}\Big),
\label{eq:Koff-centered-max}
\end{equation}
and hence
\begin{equation}
K_{\max}^{\mathrm{off}}=\max_{i\neq j}|\hat{\matK}_{ij}|
\ \le\
\mu\ +\ C_{2}\Big(\sigma\sqrt{L}+\frac{L^2}{m}\Big).
\label{eq:Koff-abs-max}
\end{equation}
\end{lemma}

\begin{proof}
For any fixed pair $(i,j)$ with $i\neq j$, apply Lemma~\ref{lem:chaos-centered-mu} with failure probability $\delta' \coloneqq \delta/F^2$.
With $L' \coloneqq \log(6/\delta') = \log(6F^2/\delta)\le L$ (for $C_{0}$ large enough), the assumption $m\le d^2/L$ implies $m \le d^2/L'$, so Lemma~\ref{lem:chaos-centered-mu} yields
\[
    \big|\hat{\matK}_{ij}-\mu\big|
    \le
    C\Big(\sqrt{\tfrac{L'}{m}} + \tfrac{(L')^2}{m} + \tfrac{\sqrt{L'}}{d}\Big)
\]
with probability at least $1-\delta'$.
A union bound over the at most $F^2$ off-diagonal pairs gives that, with probability at least $1-\delta$,
\[
    \max_{i\neq j}\big|\hat{\matK}_{ij}-\mu\big|
    \le
    C\Big(\sqrt{\tfrac{L}{m}} + \tfrac{L^2}{m} + \tfrac{\sqrt{L}}{d}\Big).
\]
Finally, since $\sigma^2=\Theta(1/d^2+1/m)$ by Lemma~\ref{lem:bilinear-mean-var}, we have $\sqrt{L/m}+\sqrt{L}/d \lesssim \sigma\sqrt{L}$, proving \eqref{eq:Koff-centered-max}.
Equation~\eqref{eq:Koff-abs-max} follows by the triangle inequality.
\end{proof}

\label{app:Kdiag-lower-bilinear}
\begin{lemma}[Uniform diagonal lower bound (bilinear kernel)]
\label{lem:Kdiag-lower-bilinear}
Assume the bilinear random-feature kernel from \cref{lem:bilinear-sketched-k2} with $m$ features, and assume $\|\vk_{i}\|_{2}=1$.
Then there exist absolute constants $c,C_{0}>0$ such that for every $t\in(0,1)$ and every $i\in[F]$,
\begin{equation}
\Pbb\big(\hat{\matK}_{ii}\le 1-t\big)\ \le\ e^{-c m t^2}.
\label{eq:Kdiag-lower-tail}
\end{equation}
Consequently, for any $\delta\in(0,1)$ and $L\defeq \log(\tfrac{C_{0} F^2}{\delta})$, with probability at least $1-\delta$,
\begin{equation}
K_{\min}^{\mathrm{diag}}=\min_{i\in[F]}\hat{\matK}_{ii}
\ \ge\
1\ -\ C_{1}\sqrt{\frac{L}{m}}.
\label{eq:Kdiag-min}
\end{equation}
\end{lemma}

\begin{proof}
Fix $i$ and write
\[
\hat{\matK}_{ii}=\frac1m\sum_{r=1}^m Z_{r},\qquad Z_{r}:=\big((\va_{r}^\top \vk_{i})(\vb_{r}^\top \vk_{i})\big)^2
=(G_rH_{r})^2,
\]
where $G_{r},H_{r}\stackrel{\text{i.i.d.}}{\sim}\mathcal N(0,1)$. Then
$\E [Z_{r}]=1$ and $\Var(Z_{r})=\E[Z_{r}^2]-(\E[Z_{r}])^2=9-1=8$.
Let $Y_{r}:=1-Z_{r}$. Then $\E[Y_{r}]=0$, $\Var(Y_{r})=8$, and since $Z_{r}\ge0$ we have $Y_{r}\le 1$. Therefore, for $t\in(0,1)$,
\[
\Pbb(\hat{\matK}_{ii}\le 1-t)
=\Pbb\!\left(\sum_{r=1}^m Y_{r}\ge mt\right).
\]
By the one-sided Bernstein inequality for independent mean-zero variables bounded above by $1$,
\[
\Pbb\!\left(\sum_{r=1}^m Y_{r}\ge mt\right)
\le
\exp\!\left(-\frac{(mt)^2}{2(8m+\frac{mt}{3})}\right)
=
\exp\!\left(-\frac{m t^2}{2(8+t/3)}\right)
\le
e^{-cmt^2}
\]
for an absolute $c>0$, proving \eqref{eq:Kdiag-lower-tail}.
The uniform bound (\eqref{eq:Kdiag-min}) follows by a union bound over $i\in[F]$.
\end{proof}

\subsubsection{Cross-Talk Bounds: Isotropic Values}

\begin{lemma}[Mean-term concentration $B_{Y}$ (isotropic values)]
\label{lem:BY-iso-values}
Assume $\vv_{1},\dots,\vv_{F}$ are i.i.d.\ uniform on $\Sph$.
Fix $\delta\in(0,1)$ and let $L\defeq \log(\tfrac{C_{0} F^2}{\delta})$.
Then with probability at least $1-\delta$ (over the values),
\begin{equation}
B_{Y}\ \defeq\ \max_{i\in[F]}\max_{j\neq i}\big|\ip{\vone}{\vY^{(ij)}}\big|
\ \le\ C_{1}\sqrt{\frac{FL}{d}}.
\label{eq:BY-iso-values}
\end{equation}
\end{lemma}

\begin{proof}
Fix a pair $(i,j)$ with $j\neq i$ and set $\vw_{ij}\defeq \vv_{i}-\vv_{j}$.
Let
\[
\mathcal{F}_{ij}\ \defeq\ \sigma(\vv_{i},\vv_{j}).
\]
Conditional on $\mathcal{F}_{ij}$, the random variables
\(
\theta_{t}\defeq \ip{\vw_{ij}}{\vv_t}
\)
for $t\notin\{i,j\}$ are independent, centered, and sub-Gaussian with
$\|\theta_{t}\|_{\psi_2}\lesssim \|\vw_{ij}\|_{2}/\sqrt{d-1}$ by Lemma~\ref{lem:iso-innerprod}.
Therefore, conditional on $\mathcal{F}_{ij}$, the sum
\[
\ip{\vone}{\vY^{(ij)}}=\sum_{t\notin\{i,j\}}\theta_{t}
\]
is centered sub-Gaussian with
\(
\big\|\ip{\vone}{\vY^{(ij)}}\big\|_{\psi_2}\lesssim \|\vw_{ij}\|_{2}\sqrt{\tfrac{F-2}{d-1}}
\lesssim \sqrt{\tfrac{F}{d}}.
\)
Hence for all $s\ge 0$,
\[
\Pbb\Big(
\big|\ip{\vone}{\vY^{(ij)}}\big|\ \ge\ C\sqrt{\tfrac{Fs}{d}}
\ \Big|\ \mathcal{F}_{ij}
\Big)
\ \le\ 2e^{-s}.
\]
This conditional tail bound is uniform in $\mathcal{F}_{ij}$, so taking expectations over $\mathcal{F}_{ij}$
yields the same inequality unconditionally.

Now choose $s= L=\log(\tfrac{C_{0} F^2}{\delta})$ so that
$2e^{-s}\le \delta/F^2$.
Union bound over all pairs $(i,j)$ (at most $F(F-1)\le F^2$ choices) gives \eqref{eq:BY-iso-values}.
\end{proof}

\begin{lemma}[Value-difference energy concentration $\widetilde{E}_{v}$ (isotropic values)]
\label{lem:Ev-iso-values}
Assume $\vv_{1},\dots,\vv_{F}$ are i.i.d.\ uniform on $\Sph$.
Fix $\delta\in(0,1)$ and let $L\defeq \log(\tfrac{C_{0} F^2}{\delta})$.
Then with probability at least $1-\delta$ (over the values),
\begin{equation}
\widetilde{E}_{v}\ \defeq\ \max_{i\in[F]}\max_{j\neq i}\sum_{t\notin\{i,j\}}\ip{\vv_{i}-\vv_j}{\vv_t}^2
\ \le\ C_{1}\,\frac{(F-2)+L}{d-1}.
\label{eq:Ev-iso-values}
\end{equation}
\end{lemma}

\begin{proof}
Fix a pair $(i,j)$ with $j\neq i$ and write $\vw_{ij}\defeq \vv_{i}-\vv_{j}$.
Let
\(
\mathcal{F}_{ij}\defeq \sigma(\vv_{i},\vv_{j})
\)
and, for $t\notin\{i,j\}$, define $\theta_{t}\defeq \ip{\vw_{ij}}{\vv_t}$.
Conditional on $\mathcal{F}_{ij}$, the $\theta_{t}$ are independent, centered, sub-Gaussian, therefore the centered squares
\[
Z_{t}\defeq \theta_{t}^2-\E[\theta_{t}^2\mid \mathcal{F}_{ij}]
\]
are independent and sub-exponential with
\(
\|Z_{t}\|_{\psi_1}\lesssim \|\vw_{ij}\|_{2}^2/(d-1)
\)
by Lemma~\ref{lem:iso-innerprod}.

Write
\[
\widetilde{E}_{v}(i,j)\ \defeq\ \sum_{t\notin\{i,j\}}\theta_{t}^2
\ =\ \sum_{t\notin\{i,j\}}\E[\theta_{t}^2\mid\mathcal{F}_{ij}]
\ +\ \sum_{t\notin\{i,j\}}Z_{t}.
\]
Since $\vv_{t}$ is isotropic, $\E[\theta_{t}^2\mid\mathcal{F}_{ij}]=\|\vw_{ij}\|_{2}^2/d$ and hence
$\sum_{t\notin\{i,j\}}\E[\theta_{t}^2\mid\mathcal{F}_{ij}]=(F-2)\|\vw_{ij}\|_{2}^2/d$.
Also $\|\vw_{ij}\|_{2}\le 2$.
A standard Bernstein inequality for sums of independent sub-exponential variables implies that for all $s\ge 0$,
\[
\Pbb\Big(
\sum_{t\notin\{i,j\}}Z_{t}\ \ge\ C\,\frac{\|\vw_{ij}\|_{2}^2}{d-1}\big(\sqrt{(F-2)s}+s\big)
\ \Big|\ \mathcal{F}_{ij}
\Big)
\ \le\ 2e^{-s}.
\]
As in Lemma~\ref{lem:BY-iso-values}, the bound is uniform in $\mathcal{F}_{ij}$, so the same tail holds unconditionally.

Using $\sqrt{(F-2)s}\le \tfrac{(F-2)+s}{2}$ and absorbing constants yields
\[
\Pbb\Big(
\widetilde{E}_{v}(i,j)\ \ge\ C\,\frac{\|\vw_{ij}\|_{2}^2}{d-1}\big((F-2)+s\big)
\Big)
\ \le\ 2e^{-cs}.
\]
Since $\|\vw_{ij}\|_{2}^2\le 4$, choosing $s= L=\log(\tfrac{C_{0} F^2}{\delta})$ makes the right-hand side
$\le \delta/F^2$.
Union bounding over all pairs $(i,j)$ yields \eqref{eq:Ev-iso-values}.
\end{proof}

\begin{lemma}[Coupling concentration $\widetilde{\kappa}$ (isotropic values)]
\label{lem:kappa-iso-values}
Assume $\vv_{1},\dots,\vv_{F}$ are i.i.d.\ uniform on $\Sph$.
Treat the kernel matrix $\hat{\matK}$ as arbitrary/deterministic.
Fix $\delta\in(0,1)$ and let $L\defeq \log(\tfrac{C_{0} F^2}{\delta})$.
Then with probability at least $1-\delta$ (over the values), the coupling statistic $\widetilde{\kappa}$ from \eqref{eq:summary-stats} satisfies
\begin{equation}
\widetilde{\kappa}\ \le\ C_{1}\,\sqrt{\frac{L}{F-2}}.
\label{eq:kappa-iso-values}
\end{equation}
\end{lemma}

\begin{proof}
Fix a pair $(i,j)$ with $j\neq i$.
If $\|\vX^{(ij)}\|_{2}=0$ or $\|\vY^{(ij)}\|_{2}=0$ then $\widetilde{\kappa}^{ij}=0$, so assume both norms are positive and set
\[
\vu^{(ij)}\defeq \vX^{(ij)}/\|\vX^{(ij)}\|_{2}.
\]
Write $\vw_{ij}\defeq \vv_{i}-\vv_{j}$ and let $\mathcal{F}_{ij}\defeq \sigma(\vv_{i},\vv_{j})$.
For $t\notin\{i,j\}$ define $\theta_{t}\defeq \ip{\vw_{ij}}{\vv_t}$. Let
\[
S_{ij}\ \defeq\ \ip{\vu^{(ij)}}{\vY^{(ij)}}\ =\ \sum_{t\notin\{i,j\}}\vu^{(ij)}_{t}\,\theta_{t}.
\]
Then
\[
\widetilde{\kappa}^{ij}=\frac{|\ip{\vX^{(ij)}}{\vY^{(ij)}}|}{\|\vX^{(ij)}\|_{2}\|\vY^{(ij)}\|_2}
=\frac{|\ip{\vu^{(ij)}}{\vY^{(ij)}}|}{\|\vY^{(ij)}\|_2}
=\frac{|S_{ij}|}{\|\vY^{(ij)}\|_2}.
\]

\paragraph{Numerator tail.}
Conditional on $\mathcal{F}_{ij}$, the vectors $(\vv_{t})_{t\notin\{i,j\}}$ are i.i.d.\ uniform on $\Sph$,
so each $\theta_{t}=\ip{\vw_{ij}}{\vv_t}$ is centered sub-Gaussian with
$\|\theta_{t}\|_{\psi_2}\lesssim \|\vw_{ij}\|_{2}/\sqrt{d-1}$ by Lemma~\ref{lem:iso-innerprod}.
Since $\sum_{t\notin\{i,j\}}(\vu^{(ij)}_{t})^2\le 1$, $S_{ij}$ is centered sub-Gaussian with
$\|S_{ij}\|_{\psi_2}\lesssim \|\vw_{ij}\|_{2}/\sqrt{d-1}$. Hence for all $s\ge 0$,
\[
\Pbb\Big(|S_{ij}|\ \ge\ C\,\|\vw_{ij}\|_{2}\sqrt{\tfrac{s}{d-1}}\ \Big|\ \mathcal{F}_{ij}\Big)\ \le\ 2e^{-s}.
\]
The right-hand side does not depend on $\mathcal{F}_{ij}$, so taking expectation over $\mathcal{F}_{ij}$ yields the same bound unconditionally.

\paragraph{Denominator lower tail.}
Conditional on $\mathcal{F}_{ij}$, the centered squares
$Z_{t}\defeq \theta_{t}^2-\E[\theta_{t}^2\mid\mathcal{F}_{ij}]$ are independent, mean-zero, and sub-exponential with
$\|Z_{t}\|_{\psi_1}\lesssim \|\vw_{ij}\|_{2}^2/(d-1)$ (Lemma~\ref{lem:iso-innerprod}).
A Bernstein inequality for sums of independent sub-exponential variables (applied to $-Z_{t}$) yields that for all $s\ge 0$,
\[
\Pbb\Bigg(
\|\vY^{(ij)}\|_{2}^2
\le
\sum_{t\notin\{i,j\}}\E[\theta_{t}^2\mid\mathcal{F}_{ij}]
-
C\,\frac{\|\vw_{ij}\|_{2}^2}{d-1}\big(\sqrt{(F-2)s}+s\big)
\ \Bigg|\ \mathcal{F}_{ij}
\Bigg)
\ \le\ e^{-s}.
\]
Again the bound is uniform in $\mathcal{F}_{ij}$, so it holds unconditionally after taking the expectation.
Since $\vv_{t}$ is isotropic, $\E[\theta_{t}^2\mid\mathcal{F}_{ij}]=\|\vw_{ij}\|_{2}^2/d$, and hence
$\sum_{t\notin\{i,j\}}\E[\theta_{t}^2\mid\mathcal{F}_{ij}]=(F-2)\|\vw_{ij}\|_{2}^2/d$.
Using $1/d\ge 1/(2(d-1))$ for $d\ge 2$, the complement event implies
\[
\|\vY^{(ij)}\|_{2}^2
\ge
\frac{\|\vw_{ij}\|_{2}^2}{d-1}\Big(\frac{F-2}{2}-C\big(\sqrt{(F-2)s}+s\big)\Big).
\]

\paragraph{A high-probability ratio bound.}
On the intersection of the numerator and denominator events, if $F-2\ge C_{0} s$,
then $\|\vY^{(ij)}\|_{2}\ge c\,\|\vw_{ij}\|_{2}\sqrt{(F-2)/(d-1)}$ and hence
\[
\widetilde{\kappa}^{ij}=\frac{|S_{ij}|}{\|\vY^{(ij)}\|_2}\ \le\ C\,\sqrt{\frac{s}{F-2}}.
\]
If instead $F-2< C_{0} s$, then $\sqrt{s/(F-2)}\ge c$ and the same bound holds trivially since $\widetilde{\kappa}^{ij}\le 1$.
Overall, for the fixed $(i,j)$ the bound holds with failure probability at most $3e^{-s}$.

\paragraph{Choose $s$ and union bound.}
Set $s=L=\log(\tfrac{C_{0} F^2}{\delta})$.
Then for a fixed $(i,j)$ the failure probability is $\le 3e^{-L}\le \delta/F^2$ (for $C_{0}$ large enough).
Union bounding over all pairs $(i,j)$ yields \eqref{eq:kappa-iso-values}.
\end{proof}

\begin{lemma}[Centered coupling concentration \(\kappa^\circ\) (isotropic values)]
\label{lem:kappa-centered-iso-values}
Assume \(\vv_1,\dots,\vv_F\) are i.i.d. uniform on \(\Sph\). Let
\(\vX^{\circ(ij)}\in\R^{F-2}\), \(i\neq j\), be deterministic vectors or random
vectors independent of the values. Define
\[
\kappa^{\circ(ij)}
\defeq
\frac{\left|\left\langle \vX^{\circ(ij)},\vY^{(ij)}\right\rangle\right|}
{\|\vX^{\circ(ij)}\|_2\|\vY^{(ij)}\|_2},
\qquad
\kappa^\circ\defeq\max_{i\neq j}\kappa^{\circ(ij)},
\]
with the convention \(\kappa^{\circ(ij)}=0\) if a denominator vanishes. Fix
\(\delta\in(0,1)\), and set
\[
L\defeq \log\left(\frac{C_0F^2}{\delta}\right).
\]
Then, with probability at least \(1-\delta\),
\[
\kappa^\circ\le C\sqrt{\frac{L}{F-2}}.
\]
\end{lemma}

\begin{proof}
Condition on \(\{\vX^{\circ(ij)}:i\neq j\}\). After conditioning, these columns
are deterministic and independent of the values. The proof of
Lemma~\ref{lem:kappa-iso-values} applies unchanged with \(\vX^{(ij)}\) replaced
by \(\vX^{\circ(ij)}\). Removing the conditioning proves the claim.
\end{proof}

\subsubsection{Cross-Talk Bounds: Isotropic Keys and Bilinear Random Features}
\begin{lemma}[Bilinear kernel mean and variance (isotropic keys)]
\label{lem:bilinear-mean-var}
Let $\hat{\matK}$ be the bilinear random-feature kernel from \cref{lem:bilinear-sketched-k2} with $m$ features and
Gaussian weights $(\va_{r},\vb_{r})_{r=1}^m$.
Assume keys $\vk_{1},\dots,\vk_{F}$ are i.i.d.\ uniform on $\Sph$.
Fix $t\neq i$ and write $\rho=\ip{\vk_t}{\vk_i}$.
Then:
\begin{enumerate}[leftmargin=2.2em]
\item (Conditional mean) $\E[\hat{\matK}_{ti}\mid \rho]=\rho^2$.
\item (Unconditional mean) $\mu\defeq \E[\hat{\matK}_{ti}]=\E[\rho^2]=1/d$.
\item (Variance scale) letting $\sigma^2\defeq \E[(\hat{\matK}_{ti}-\mu)^2]$, we have
$\sigma^2=\Theta(\tfrac{1}{d^2}+\tfrac{1}{m})$ %
\end{enumerate}
\end{lemma}

\begin{proof}

Condition on $(\vk_{t},\vk_{i})$ and hence on $\rho$.
Write
\[
\hat{\matK}_{ti}=\frac{1}{m}\sum_{r=1}^m U_{r},
\qquad
U_{r} \defeq (\va_{r}^\top \vk_{t})(\va_{r}^\top \vk_{i})\,(\vb_{r}^\top \vk_{t})(\vb_{r}^\top \vk_{i}),
\]
with $(\va_{r},\vb_{r})$ i.i.d.\ standard Gaussian.
Let $A_{r}\defeq (\va_{r}^\top \vk_{t})(\va_{r}^\top \vk_{i})$ and $B_{r}\defeq (\vb_{r}^\top \vk_{t})(\vb_{r}^\top \vk_{i})$ so that $U_{r}=A_rB_{r}$
and $A_{r}$ is independent of $B_{r}$.

\emph{Mean.}
Since $(\va_{r}^\top \vk_{t},\va_{r}^\top \vk_{i})$ is centered bivariate Gaussian with covariance $\rho$,
$\E[A_{r}\mid \rho]=\rho$ and likewise $\E[B_{r}\mid\rho]=\rho$, hence
$\E[U_{r}\mid\rho]=\rho^2$ and $\E[\hat{\matK}_{ti}\mid\rho]=\rho^2$.
Averaging over isotropic keys yields $\mu=\E[\rho^2]=1/d$.

\emph{Variance scale.}
A Wick/Isserlis calculation gives $\E[A_{r}^2\mid\rho]=1+2\rho^2$ and hence
$\Var(U_{r}\mid\rho)=1+4\rho^2+3\rho^4$.
Therefore
$\Var(\hat{\matK}_{ti}\mid\rho)=\tfrac{1}{m}(1+4\rho^2+3\rho^4)$.
Using the law of total variance gives
\(
\sigma^2=\E[\Var(\hat{\matK}_{ti}\mid\rho)]+\Var(\rho^2)
=\Theta(\tfrac{1}{m})+\Theta(\tfrac{1}{d^2})
\),
where $\Var(\rho^2)=\Theta(1/d^2)$ follows from $\E[\rho^2]=1/d$ and $\E[\rho^4]=3/(d(d+2))$.
\end{proof}

\begin{lemma}[Gaussian-chaos entry concentration, centered at $\rho^2$]\label{lem:chaos-centered-conditional}
Let $\mathbf{x},\mathbf{y} \in \R^d$ be deterministic unit vectors and set $\rho \coloneqq \langle \mathbf{x}, \mathbf{y}\rangle$.
Let $\{\va_{r},\vb_{r}\}_{r=1}^m$ be i.i.d.\ with $\va_{r},\vb_{r} \sim \cN(0,I_{d})$, and define the bilinear random-feature kernel
\[
    \hat K(\mathbf{x},\mathbf{y}) \coloneqq \frac1m \sum_{r=1}^m (\va_{r}^\top \mathbf{x})(\va_{r}^\top \mathbf{y})(\vb_{r}^\top \mathbf{x})(\vb_{r}^\top \mathbf{y}).
\]
Then there exists a universal constant $C>0$ such that, for all $u \ge 0$,
\[
    \Pbb\Big(
        \big|\hat K(\mathbf{x},\mathbf{y}) - \rho^2\big|
        \ge
        C\Big(\sqrt{\tfrac{u}{m}} + \tfrac{u^2}{m}\Big)
    \Big)
    \le 2e^{-u}.
\]
\end{lemma}
\begin{proof}
Fix $\mathbf{x},\mathbf{y}$ and write
\[
    Z_{r} \coloneqq (\va_{r}^\top \mathbf{x})(\va_{r}^\top \mathbf{y})(\vb_{r}^\top \mathbf{x})(\vb_{r}^\top \mathbf{y}) - \rho^2,
    \qquad r\in[m].
\]
Then $\{Z_{r}\}_{r=1}^m$ are i.i.d.\ and mean-zero.
Moreover, $Z_{r}$ is a centered degree-$4$ polynomial in jointly Gaussian random variables with bounded covariance, so $\|Z_{r}\|_{L_2}\le C$; by Gaussian hypercontractivity (Theorem~\ref{thm:gaussian-hypercontractivity} with $k=4$), $\|Z_{r}\|_{L_p}\le (p-1)^2\|Z_{r}\|_{L_2}\le C'p^2$ for all $p\ge2$. A standard moment-to-tail conversion then yields a sub-Weibull$(\tfrac12)$ tail: there exists a universal constant $C_{0}>0$ such that, for all $u\ge 0$,
\[
    \Pbb\Big( |Z_{r}| \ge C_{0}(\sqrt{u} + u^2)\Big) \le 2e^{-u}.
\]
Applying Lemma~\ref{lem:weighted-bern-subweib} with weights $w_{r} \equiv 1/m$ (so that $\|\mathbf{w}\|_{2} = m^{-1/2}$ and $\|\mathbf{w}\|_\infty = m^{-1}$) yields the claim.
\end{proof}

\begin{lemma}[Gaussian-chaos entry concentration, centered at $\mu$]\label{lem:chaos-centered-mu}
Let $\mathbf{x},\mathbf{y} \stackrel{\mathrm{i.i.d.}}{\sim} \mathrm{Unif}(\Sph)$ be independent of $\{\va_{r},\vb_{r}\}_{r=1}^m$, and set $\mu \coloneqq 1/d$ and $\rho\coloneqq \langle \mathbf{x},\mathbf{y}\rangle$.
Fix $\delta\in(0,1)$ and define $L \coloneqq \log(6/\delta)$.
Assume
\begin{equation}\label{eq:m-assumption}
    m \le \frac{d^2}{L}.
\end{equation}
Then there exists a universal constant $C>0$ such that
\[
    \Pbb\Big(
        \big|\hat K(\mathbf{x},\mathbf{y}) - \mu\big|
        \le
        C\Big(
            \sqrt{\tfrac{L}{m}} + \tfrac{L^2}{m} + \sqrt{\tfrac{L}{d^2}}
        \Big)
    \Big)
    \ge 1-\delta.
\]
\end{lemma}
\begin{proof}
By the triangle inequality,
\[
    |\hat K(\mathbf{x},\mathbf{y}) - \mu|
    \le
    |\hat K(\mathbf{x},\mathbf{y}) - \rho^2| + |\rho^2 - \mu|.
\]

\paragraph{Control random-feature fluctuation.}
Condition on $\mathbf{x},\mathbf{y}$ and apply Lemma~\ref{lem:chaos-centered-conditional} with $u=L$ to obtain
\[
    \Pbb\Big(
        |\hat K(\mathbf{x},\mathbf{y}) - \rho^2|
        \le
        C\Big(\sqrt{\tfrac{L}{m}} + \tfrac{L^2}{m}\Big)
        \,\Big|\, \mathbf{x},\mathbf{y}
    \Big)
    \ge 1-\delta/3
\]
for a suitable universal constant $C>0$.

\paragraph{Control geometric fluctuation.}
Recall that if $\mathbf{x},\mathbf{y} \stackrel{\text{i.i.d.}}{\sim} \mathrm{Unif}(\Sph)$, then $\rho \stackrel{d}{=} g_{1}/\|g\|_{2}$ for $g \sim \cN(0,I_{d})$.
Let $X \coloneqq \|g\|_{2}^2 \sim \chi^2_{d}$, so that $\rho^2 = g_1^2/X$.
Since $g_1^2 \sim \chi^2_1$ and $X \sim \chi^2_d$, Theorem~\ref{thm:laurent-massart} implies that, each with probability at least $1-\delta/3$,
\[
    |g_1^2 - 1| \le C(\sqrt{L} + L),
    \qquad
    |X - d| \le C(\sqrt{dL} + L)
\]
for a universal constant $C>0$.
In particular, whenever $d \ge C'L$ for a suitable absolute constant $C'$, the second estimate forces $X \ge d/2$; when $d \lesssim L$ the final bound of this step is vacuous up to constants and can be absorbed by enlarging $C$. So assume $X \ge d/2$. Writing $\rho^2 - \mu = \dfrac{g_1^2 d - X}{Xd}$ and using $X \ge d/2$ together with the triangle inequality $|g_1^2 d - X| \le d\,|g_1^2 - 1| + |X - d|$,
\[
    |\rho^2 - \mu|
    =
    \frac{|g_1^2 d - X|}{Xd}
    \le
    \frac{2}{d}\,|g_1^2 - 1| + \frac{2}{d^2}\,|X - d|.
\]
Substituting the two chi-square estimates and using $d \ge 1$,
\[
    |\rho^2 - \mu|
    \le
    \frac{C}{d}(\sqrt{L} + L) + \frac{C}{d^2}(\sqrt{dL} + L)
    \le
    \frac{C}{d}\big(\sqrt{L} + L\big).
\]
Thus, on the intersection of the two chi-square events (of probability at least $1-2\delta/3$),
\[
    |\rho^2 - \mu|
    \le
    \frac{C}{d}\big(\sqrt{L} + L\big).
\]

\paragraph{Combine and simplify.}
Taking a union bound over both gives, with probability at least $1-\delta$,
\[
    |\hat K(\mathbf{x},\mathbf{y}) - \mu|
    \le
    C\Big(
        \sqrt{\tfrac{L}{m}} + \tfrac{L^2}{m} + \frac{\sqrt{L}}{d} + \frac{L}{d}
    \Big).
\]
Finally, under \eqref{eq:m-assumption}, we have $L/d \le \sqrt{L/m}$, so the $L/d$ term is absorbed, yielding the stated bound.
\end{proof}

\begin{lemma}[Centered column-energy concentration \(E_K^\circ\) (isotropic keys, bilinear kernel; sharpened)]
\label{lem:EK-iso-keys}
Assume keys \(\vk_1,\dots,\vk_F\) are i.i.d. uniform on \(\Sph\), independent of
the bilinear random features from \cref{lem:bilinear-sketched-k2}. Let
\[
\mu\defeq \frac1d,
\qquad
\sigma^2\defeq
\E[(\hat{\matK}_{ti}-\mu)^2]
=
\Theta\left(\frac1{d^2}+\frac1m\right),
\qquad t\neq i.
\]
Fix \(\delta\in(0,1)\), and define
\[
L\defeq \log\left(\frac{C_0F^2}{\delta}\right).
\]
Assume
\[
m\le \frac{d^2}{L},
\qquad
L^3\le c_0\sigma^2m^2.
\]
Then, for \(c_0>0\) sufficiently small and \(C_0\) sufficiently large, with
probability at least \(1-\delta\),
\[
E_K^\circ
\le
C\sigma^2\bigl((F-2)+L\bigr).
\]
Equivalently,
\[
\sqrt{E_K^\circ}
\le
C\sigma\sqrt{(F-2)+L}.
\]
In particular, if \(F-2\ge L\), then
\[
\sqrt{E_K^\circ}
\le
C\sigma\sqrt{F-2}.
\]
\end{lemma}

\begin{proof}
Assume
\(F\ge3\). For each \(i\in[F]\), define
\[
\mathcal C_i
\defeq
\sum_{t\neq i}(\hat{\matK}_{ti}-\mu)^2.
\]
For each \(i\neq j\),
\[
E_K^\circ(i,j)
=
\sum_{t\notin\{i,j\}}(\hat{\matK}_{ti}-\mu)^2
\le
\mathcal C_i.
\]
Thus
\[
E_K^\circ\le \max_i\mathcal C_i.
\]

Fix \(i\), and set \(q_i=\vk_i\). For \(q\in\Sph\), define \(A(q)\) as in
Lemma~\ref{lem:bilinear-column-matrix-regularity}. Then, for every \(x\in\Sph\),
\[
x^\top A(q)x
=
\frac1m\sum_{r=1}^m
(\va_r^\top x)(\va_r^\top q)(\vb_r^\top x)(\vb_r^\top q)
=
\hat K(x,q).
\]
Therefore
\[
\hat{\matK}_{ti}=\vk_t^\top A(q_i)\vk_t,
\qquad t\neq i.
\]
Define
\[
\tau_i\defeq \frac{\operatorname{tr}A(q_i)}{d},
\qquad
B_i\defeq A(q_i)-\tau_iI_d,
\qquad
\beta_i\defeq \tau_i-\mu.
\]
Then \(\operatorname{tr}B_i=0\), and
\[
\hat{\matK}_{ti}-\mu
=
\vk_t^\top B_i\vk_t+\beta_i.
\]
Let
\[
\mathcal G_i
\defeq
\sigma\left(\vk_i,\{(\va_r,\vb_r)\}_{r=1}^m\right).
\]
Conditional on \(\mathcal G_i\), the vectors \(\{\vk_t:t\neq i\}\) are independent
uniform spherical vectors. Define
\[
v_i\defeq \frac{2\|B_i\|_F^2}{d(d+2)}+\beta_i^2,
\qquad
r_i\defeq \frac{\|B_i\|_{\mathrm{op}}}{d}+|\beta_i|.
\]
By Lemma~\ref{lem:bilinear-column-parameters}, with probability at least
\(1-\delta/4\), simultaneously for every \(i\in[F]\),
\[
v_i\le C\sigma^2,
\qquad
r_i^2L^2\le C\sigma^2L.
\]
On this event, Lemma~\ref{lem:spherical-quad-square-sum}, applied conditionally
with \(n=F-1\) and \(s=L\), gives
\[
\Pbb\left(
\mathcal C_i
>
C\left((F-1)\sigma^2+\sigma^2L\right)
\,\middle|\,
\mathcal G_i
\right)
\le e^{-L},
\]
because \(L\ge \log(e(F-1))\) after increasing \(C_0\). Union bounding over
\(i\in[F]\),
\[
Fe^{-L}
=
F\frac{\delta}{C_0F^2}
\le
\delta/4
\]
for \(C_0\) sufficiently large. Combining the conditional square-sum event with
the parameter event yields, with probability at least \(1-\delta\),
\[
\max_i\mathcal C_i
\le
C\sigma^2\bigl((F-1)+L\bigr).
\]
Since \(F\ge3\), \(F-1\le2(F-2)\), so
\[
E_K^\circ
\le
C\sigma^2\bigl((F-2)+L\bigr).
\]
Taking square roots proves
\[
\sqrt{E_K^\circ}
\le
C\sigma\sqrt{(F-2)+L}.
\]
If \(F-2\ge L\), then \((F-2)+L\le2(F-2)\), giving
\[
\sqrt{E_K^\circ}
\le
C\sigma\sqrt{F-2}.
\]
\end{proof}

\begin{lemma}[Concentration of the effective coupling $\kappa_{\mathrm{eff}}^\circ$ (isotropic keys, bilinear kernel)]
\label{lem:keff-iso-keys}
Assume the setting of Theorem~\ref{app:regime2}: isotropic keys, bilinear kernel, and deterministic values/codes.
Let $\mu=1/d$ and $\sigma^2=\Theta(\tfrac{1}{d^2}+\tfrac{1}{m})$.
Fix $\delta\in(0,1)$ and let $L\defeq \log(\tfrac{C_{0} F^2}{\delta})$.
Assume additionally that $m\le d^2/L$.
Define
\[
\kappa_{\mathrm{eff}}^\circ
\ \defeq\
\max_{i\neq j}\frac{|\ip{\vX^{\circ(ij)}}{\vY^{(ij)}}|}{\|\vY^{(ij)}\|_2}
\]
and
\[
L_{v}\ \defeq\ \max_{i\neq j}\frac{\|\vY^{(ij)}\|_{1}^2}{\|\vY^{(ij)}\|_{2}^2}
\]
Then with probability at least $1-\delta$ (over keys and features),
\[
\kappa_{\mathrm{eff}}^\circ
\ \le\
C\left(\sigma\sqrt{L}+\frac{L^2}{m}\right)\sqrt{L_v}.
\]
\end{lemma}

\begin{proof}
Fix a pair $(i,j)$ with $j\neq i$.
Write $\theta_{t}\defeq \hat{\matK}_{ti}-\mu$ for $t\notin\{i,j\}$, so that
$\vX^{\circ(ij)}=(\theta_{t})_{t\notin\{i,j\}}$.
Then
\[
\frac{|\ip{\vX^{\circ(ij)}}{\vY^{(ij)}}|}{\|\vY^{(ij)}\|_2}
\ =\
\frac{\big|\sum_{t\notin\{i,j\}} \theta_{t}\,\vY^{(ij)}_{t}\big|}{\|\vY^{(ij)}\|_2}
\ \le\
\Big(\max_{t\notin\{i,j\}}|\theta_{t}|\Big)\,\frac{\|\vY^{(ij)}\|_1}{\|\vY^{(ij)}\|_2}.
\]
Taking a maximum over $(i,j)$ and using the definition of $L_{v}$ gives
\[
\kappa_{\mathrm{eff}}^\circ
\ \le\
\Big(\max_{p\neq q}|\hat{\matK}_{pq}-\mu|\Big)\,\sqrt{L_v}.
\]
Now apply Lemma~\ref{lem:Koff-iso-keys} with failure probability $\delta$ to bound
\(
\max_{p\neq q}|\hat{\matK}_{pq}-\mu|\le C(\sigma\sqrt{L}+\tfrac{L^2}{m}).
\)
Multiplying by $\sqrt{L_v}$ completes the proof.
\end{proof}

\subsubsection{Additional Concentration Tools}
\label{app:aux-additional-concentration-coupling}

\begin{lemma}[Weighted Bernstein for sub-Weibull$(1/2)$-type tails]
\label{lem:weighted-bern-subweib}
Let $Z_1,\dots,Z_N$ be independent, mean-zero random variables such that
\[
\Pbb\!\left(|Z_r|\ge C_0(\sigma\sqrt{u}+u^2/m)\right)\le 2e^{-u}
\qquad\text{for all }u\ge0,\ r\in[N].
\]
Then there exist absolute constants $C,c>0$ such that for every deterministic
$w\in\R^N$ and every $t\ge0$,
\[
\Pbb\!\left(
\Big|\sum_{r=1}^N w_r Z_r\Big|
\ge
C\Big(\sigma\|w\|_2\sqrt{t}+\frac{t^2}{m}\|w\|_\infty\Big)
\right)
\le 2e^{-ct}.
\]
Consequently, for $u=\log(C_1/\delta)$, with probability at least $1-\delta$,
\[
\Big|\sum_{r=1}^N w_r Z_r\Big|
\le
C_2\Big(\sigma\|w\|_2\sqrt{u}+\frac{u^2}{m}\|w\|_\infty\Big).
\]
\end{lemma}
\begin{proof}
The assumed tail bound is equivalent up to absolute constants to the statement
that each $Z_r$ is sub-Weibull of order $\alpha=\tfrac12$ with generalized
Bernstein--Orlicz parameters $\nu\asymp \sigma$ and $L\asymp (\sigma m)^{-1}$.
Applying Theorem~2.4 of \citep{BongKuchibhotla2023} to the weighted sum
$\sum_{r=1}^N w_r Z_r$ gives
\[
\Pbb\!\left(
\Big|\sum_{r=1}^N w_r Z_r\Big|
\ge
C\Big(\sigma\|w\|_2\sqrt{t}+\sigma L\, t^2 \|w\|_\infty\Big)
\right)
\le 2e^{-ct}.
\]
Since $\sigma L \asymp 1/m$, this becomes
\[
\Pbb\!\left(
\Big|\sum_{r=1}^N w_r Z_r\Big|
\ge
C\Big(\sigma\|w\|_2\sqrt{t}+\frac{t^2}{m}\|w\|_\infty\Big)
\right)
\le 2e^{-ct}.
\]
Setting $t=\log(C_1/\delta)$ yields the stated high-probability bound.
\end{proof}

\begin{theorem}[Gaussian hypercontractivity; Nelson--Gross]
\label{thm:gaussian-hypercontractivity}
Let \(G\sim\mathcal N(0,I_n)\), and let \(P(G)\) be a polynomial of degree at
most \(k\) in the standard Gaussian variables. Then, for every \(q\ge2\),
\[
\|P(G)\|_{L_q}
\le
(q-1)^{k/2}\|P(G)\|_{L_2}.
\]
\end{theorem}

\noindent
This is the standard degree-\(k\) polynomial-chaos corollary of the Gaussian
hypercontractivity theorem; see O'Donnell~\cite[Theorem~11.23]{odonnell2014analysis}.

\begin{theorem}[Sub-Weibull average bound from moment growth]
\label{thm:subweibull-bernstein-average}
Let \(W_1,\dots,W_m\) be independent mean-zero random variables. Suppose that,
for some \(\rho\ge1\) and \(K>0\),
\[
\|W_r\|_{L_p}\le Kp^\rho
\qquad\text{for every }p\ge2\text{ and every }r\in[m].
\]
Then there is a constant \(C_\rho>0\), depending only on \(\rho\), such that for
every \(u\ge1\),
\[
\Pbb\left(
\left|
\frac1m\sum_{r=1}^m W_r
\right|
>
C_\rho K\left(
\sqrt{\frac{u}{m}}+
\frac{u^\rho}{m}
\right)
\right)
\le 2e^{-u}.
\]
In particular, if \(W_1,\dots,W_m\) are independent centered degree-\(k\)
Gaussian chaoses with \(k\ge2\) and
\[
\|W_r\|_{L_2}\le K_0,
\]
then
\[
\Pbb\left(
\left|
\frac1m\sum_{r=1}^m W_r
\right|
>
C_kK_0\left(
\sqrt{\frac{u}{m}}+
\frac{u^{k/2}}{m}
\right)
\right)
\le 2e^{-u}.
\]
\end{theorem}

\begin{proof}
Set \(\alpha\defeq1/\rho\in(0,1]\). By a standard argument, the moment growth
condition implies the Orlicz-norm bound
\[
\|W_r\|_{\psi_\alpha}
\defeq
\inf\left\{c>0:\E\exp\left((|W_r|/c)^\alpha\right)\le2\right\}
\le
C_\rho K
\qquad\text{for every }r\in[m]:
\]
indeed, monotonicity of \(p\mapsto\|W_r\|_{L_p}\) and the assumption give
\(\|W_r\|_{L_{\alpha j}}\le C_\rho Kj^{1/\alpha}\) for every integer \(j\ge1\),
and expanding the exponential and using \(j!\ge(j/e)^j\) shows
\(\E\exp\left((|W_r|/(A_\rho K))^\alpha\right)\le2\) for \(A_\rho\) sufficiently
large.

Now apply the generalized Bernstein--Orlicz inequality ~\cite[Theorem~3.1]{kuchibhotla2022moving} to
\(Z\defeq\frac1m\sum_{r=1}^mW_r\), with weights \(a_r=1/m\) and
\(b_r\defeq a_r\|W_r\|_{\psi_\alpha}\), so that
\(\|b\|_2\le C_\rho K/\sqrt m\) and \(\|b\|_\infty\le C_\rho K/m\).
The defining tail property of their generalized Bernstein--Orlicz norm yields,
for every \(u\ge1\),
\[
\Pbb\left(
|Z|
>
C_\alpha\left(
\|b\|_2\sqrt u
+
\|b\|_\infty u^{1/\alpha}
\right)
\right)
\le
2e^{-u}.
\]
Substituting the bounds on \(\|b\|_2\) and \(\|b\|_\infty\) and using
\(1/\alpha=\rho\) gives the stated tail bound.

Finally, if \(W_r\) is a centered degree-\(k\) Gaussian chaos with
\(\|W_r\|_{L_2}\le K_0\), Theorem~\ref{thm:gaussian-hypercontractivity} gives
\[
\|W_r\|_{L_p}
\le
(p-1)^{k/2}\|W_r\|_{L_2}
\le
C_kK_0p^{k/2},
\qquad p\ge2.
\]
Apply the first part with \(\rho=k/2\) and \(K=C_kK_0\).
\end{proof}
\begin{theorem}[Operator norm from bilinear forms on nets]
\label{thm:operator-norm-net}
Let \(A\in\mathbb R^{n\times m}\), let \(0<\varepsilon<1/2\), and let
\(\mathcal N\subset \mathbb S^{n-1}\), \(\mathcal M\subset\mathbb S^{m-1}\)
be finite \(\varepsilon\)-nets. Then
\[
\max_{x\in\mathcal N,\;y\in\mathcal M}|x^\top Ay|
\le
\|A\|_{\mathrm{op}}
\le
\frac{1}{1-2\varepsilon}
\max_{x\in\mathcal N,\;y\in\mathcal M}|x^\top Ay|.
\]
In particular, for \(\varepsilon=1/4\),
\[
\|A\|_{\mathrm{op}}
\le
2\max_{x\in\mathcal N,\;y\in\mathcal M}|x^\top Ay|.
\]
Moreover, for every \(d\ge 1\), the sphere \(\mathbb S^{d-1}\) admits a
\(1/4\)-net of cardinality at most \(9^d\). Hence one may choose such nets with
\[
|\mathcal N|\le 9^n
\qquad\text{and}\qquad
|\mathcal M|\le 9^m .
\]
\end{theorem}

\noindent
This is Vershynin~\cite[Lemma~4.4.2 and Corollary~4.2.11]{vershynin2026high}.

\begin{theorem}[Laurent--Massart weighted chi-square tail]
\label{thm:laurent-massart}
Let \(Y_1,\dots,Y_D\) be independent standard Gaussian random variables, and let
\(a_1,\dots,a_D\ge0\). Define
\[
Z\defeq \sum_{i=1}^D a_i(Y_i^2-1).
\]
Then, for every \(x>0\),
\[
\Pbb\left(
Z\ge 2\|a\|_2\sqrt x+2\|a\|_\infty x
\right)
\le e^{-x},
\]
and
\[
\Pbb\left(
Z\le -2\|a\|_2\sqrt x
\right)
\le e^{-x}.
\]
In particular, if \(Q\sim\chi^2_D\), then
\[
\Pbb\left(Q-D\ge 2\sqrt{Dx}+2x\right)
\le e^{-x},
\]
and
\[
\Pbb\left(D-Q\ge 2\sqrt{Dx}\right)
\le e^{-x}.
\]
\end{theorem}

\noindent
This is Laurent--Massart~\cite[Lemma~1]{laurent2000adaptive}.

\begin{lemma}[Weighted product-Gaussian chaos]
\label{lem:weighted-product-gaussian}
Let \((g_r,h_r)_{r=1}^m\) be independent pairs of independent standard
Gaussians, and let \(a=(a_1,\dots,a_m)\in\R^m\) be deterministic. Then, for every
\(t\ge0\),
\[
\Pbb\left(
\left|
\sum_{r=1}^m a_rg_rh_r
\right|
>
C\left(\|a\|_2\sqrt t+\|a\|_\infty t\right)
\right)
\le 2e^{-t}.
\]
\end{lemma}

\begin{proof}
For independent \(g,h\sim N(0,1)\), conditioning on \(g\) gives
\[
\E e^{\lambda gh}
=
\E e^{\lambda^2g^2/2}
=
(1-\lambda^2)^{-1/2},
\qquad |\lambda|<1.
\]
Therefore, for \(|\lambda|\le c/\|a\|_\infty\),
\[
\log\E\exp\left(\lambda\sum_{r=1}^m a_rg_rh_r\right)
=
-\frac12\sum_{r=1}^m\log(1-\lambda^2a_r^2)
\le
C\lambda^2\|a\|_2^2.
\]
If \(a=0\), the claim is trivial. Otherwise, Chernoff's bound gives
\[
\Pbb\left(
\sum_{r=1}^m a_rg_rh_r
>
\nu
\right)
\le
\exp\left(-\lambda\nu+C\lambda^2\|a\|_2^2\right)
\]
for every \(0\le\lambda\le c/\|a\|_\infty\). Optimizing with
\[
\lambda
=
 c\min\left\{\frac{\nu}{\|a\|_2^2},\frac1{\|a\|_\infty}\right\}
\]
gives
\[
\Pbb\left(
\sum_{r=1}^m a_rg_rh_r
>
C(\|a\|_2\sqrt t+\|a\|_\infty t)
\right)
\le e^{-t}.
\]
Apply the same argument to \(-\sum_r a_rg_rh_r\) and union bound.
\end{proof}

\begin{lemma}[Square-sum concentration from a two-level tail]
\label{lem:square-sum-two-level}
Let \(X_1,\dots,X_n\) be independent random variables. Assume that, for constants
\(v>0\), \(r>0\), and \(C_0\ge1\),
\[
\E X_i^2\le v
\]
and, for every \(u\ge1\),
\[
\Pbb\left(
|X_i|>C_0(\sqrt{vu}+ru)
\right)
\le 2e^{-u}.
\]
Assume also that
\[
r^2\le C_0v.
\]
Then, for every \(s\ge1\),
\[
\Pbb\left(
\sum_{i=1}^nX_i^2
>
C\left(nv+vs+r^2s^2\right)
\right)
\le Ce^{-s},
\]
where \(C\) depends only on \(C_0\).
\end{lemma}

\begin{proof}
We first reduce the problem to sums of variables with exponential tails. For each
\(i\), let \(R_i\) have the same distribution as \(|X_i|\), with the variables
\(R_1,\dots,R_n\) independent. Define
\[
T_i
\defeq
\inf\left\{u\ge1:
R_i\le C_0(\sqrt{vu}+ru)
\right\}.
\]
Then, by definition,
\[
R_i\le C_0(\sqrt{vT_i}+rT_i).
\]
Moreover, for every \(u\ge1\), if \(T_i>u\), then
\[
R_i>C_0(\sqrt{vu}+ru),
\]
and hence the assumed tail bound gives
\[
\Pbb(T_i>u)\le 2e^{-u}.
\]
Since the desired estimate depends only on the marginal laws and independence, we may
work with this coupled representation and write, after changing the absolute constant,
\[
|X_i|
\le
C(\sqrt{vT_i}+rT_i),
\qquad
\Pbb(T_i>u)\le2e^{-u},\quad u\ge1.
\]

Squaring the domination gives
\[
X_i^2
\le
C\left(vT_i+r^2T_i^2\right).
\]
Hence
\[
\sum_{i=1}^n X_i^2
\le
C\left(
 v\sum_{i=1}^nT_i+r^2\sum_{i=1}^nT_i^2
\right).
\]
It remains to control the two sums involving \(T_i\).

The tail bound \(\Pbb(T_i>u)\le2e^{-u}\) implies, for every \(p\ge2\),
\[
\|T_i\|_{L_p}
\le Cp,
\qquad
\|T_i^2\|_{L_p}
=
\|T_i\|_{L_{2p}}^2
\le Cp^2.
\]
Consequently,
\[
\|T_i-\E T_i\|_{L_p}\le Cp,
\qquad
\|T_i^2-\E T_i^2\|_{L_p}\le Cp^2.
\]
Also \(\E T_i\le C\) and \(\E T_i^2\le C\). Applying
Theorem~\ref{thm:subweibull-bernstein-average} to the centered variables
\(T_i-\E T_i\) with \(\rho=1\) gives
\[
\Pbb\left(
\sum_{i=1}^nT_i>C(n+s)
\right)
\le Ce^{-s}.
\]
Similarly, applying Theorem~\ref{thm:subweibull-bernstein-average} to
\(T_i^2-\E T_i^2\) with \(\rho=2\) gives
\[
\Pbb\left(
\sum_{i=1}^nT_i^2>C(n+s^2)
\right)
\le Ce^{-s}.
\]
On the intersection of these two events,
\[
\sum_{i=1}^n X_i^2
\le
C\left(v(n+s)+r^2(n+s^2)\right).
\]
Since \(r^2\le C_0v\), the term \(r^2n\) is absorbed by \(vn\). Therefore
\[
\sum_{i=1}^n X_i^2
\le
C\left(nv+vs+r^2s^2\right)
\]
with probability at least \(1-Ce^{-s}\).
\end{proof}

\begin{lemma}[Spherical quadratic-form tail]
\label{lem:spherical-quadratic-tail}
Let \(x\sim\mathrm{Unif}(\Sph)\), and let \(B\in\R^{d\times d}\) be
symmetric and trace-free. Then, for every \(u\ge1\),
\[
\Pbb\left(
|x^\top Bx|
>
C\left(
\frac{\|B\|_F}{d}\sqrt u+
\frac{\|B\|_{\mathrm{op}}}{d}u
\right)
\right)
\le
2e^{-u}.
\]
\end{lemma}

\begin{proof}
Write
\[
x=\frac{g}{\|g\|_2},
\qquad
 g\sim\mathcal N(0,I_d).
\]
Diagonalize \(B=Q\operatorname{diag}(\lambda_1,\dots,\lambda_d)Q^\top\). By
rotational invariance,
\[
g^\top Bg
\stackrel{d}{=}
\sum_{i=1}^d\lambda_i g_i^2.
\]
Since \(\operatorname{tr}B=\sum_i\lambda_i=0\),
\[
g^\top Bg
=
\sum_{i=1}^d\lambda_i(g_i^2-1).
\]
Apply Theorem~\ref{thm:laurent-massart} to the positive and negative parts of the
weights \(\lambda_i\). Equivalently, apply it to the two weighted sums associated
with \(\lambda_i^+\) and \(\lambda_i^-\). This gives, for all \(u\ge1\),
\[
\Pbb\left(
|g^\top Bg|
>
C\left(\|B\|_F\sqrt u+\|B\|_{\mathrm{op}}u\right)
\right)
\le
2e^{-u}.
\]
Choose a numerical constant \(a\ge1\), to be fixed below. Applying the above
Gaussian quadratic-form bound with \(au\) in place of \(u\), we get
\[
\Pbb\left(
|g^\top Bg|
>
C_a\left(\|B\|_F\sqrt u+\|B\|_{\mathrm{op}}u\right)
\right)
\le 2e^{-au}.
\]
Also, by the lower-tail part of Theorem~\ref{thm:laurent-massart} applied to
\(\|g\|_2^2\sim\chi_d^2\),
\[
\Pbb\left(\|g\|_2^2<d/2\right)
\le e^{-cd}.
\]
If \(u\le c_1d\), with \(c_1>0\) chosen small enough, then
\(e^{-cd}\le e^{-2u}\). On the event \(\|g\|_2^2\ge d/2\),
\[
|x^\top Bx|
=
\frac{|g^\top Bg|}{\|g\|_2^2}
\le
C_a\left(
\frac{\|B\|_F}{d}\sqrt u+
\frac{\|B\|_{\mathrm{op}}}{d}u
\right).
\]
Thus, for \(u\le c_1d\), the failure probability is at most
\(2e^{-au}+e^{-2u}\). Choosing \(a\) large enough and enlarging the constant in
the threshold gives failure probability at most \(2e^{-u}\).
If \(u>c_1d\), then the same inequality holds deterministically after
enlarging \(C\), because
\[
|x^\top Bx|\le\|B\|_{\mathrm{op}}
\le
C\frac{\|B\|_{\mathrm{op}}}{d}u.
\]
This proves the claim.
\end{proof}

\begin{lemma}[Spherical quadratic square-sum bound]
\label{lem:spherical-quad-square-sum}
Let \(x_1,\dots,x_n\) be i.i.d. uniform on \(\Sph\subset\R^d\). Let
\(B\in\R^{d\times d}\) be deterministic, symmetric, and trace-free, and let
\(\beta\in\R\). Define
\[
X_\ell\defeq x_\ell^\top Bx_\ell+\beta.
\]
Set
\[
v\defeq \frac{2\|B\|_F^2}{d(d+2)}+\beta^2,
\qquad
r\defeq \frac{\|B\|_{\mathrm{op}}}{d}+|\beta|.
\]
Then, for every \(s\ge1\),
\[
\Pbb\left(
\sum_{\ell=1}^nX_\ell^2
>
C\left(nv+vs+r^2s^2\right)
\right)
\le
e^{-s}.
\]
\end{lemma}

\begin{proof}
For \(x\sim\mathrm{Unif}(\Sph)\), isotropy gives
\[
\E[x^\top Bx]=\frac{\operatorname{tr}B}{d}=0.
\]
The fourth-moment identity for the sphere gives
\[
\E[(x^\top Bx)^2]
=
\frac{(\operatorname{tr}B)^2+2\|B\|_F^2}{d(d+2)}
=
\frac{2\|B\|_F^2}{d(d+2)}.
\]
Therefore
\[
\E X_\ell^2=v.
\]
Moreover,
\[
\left(\frac{\|B\|_{\mathrm{op}}}{d}\right)^2
\le
\frac{\|B\|_F^2}{d^2}
\le
C\frac{2\|B\|_F^2}{d(d+2)}
\le
Cv,
\]
and \(\beta^2\le v\), so
\[
r^2\le Cv.
\]
By Lemma~\ref{lem:spherical-quadratic-tail}, for every \(u\ge1\),
\[
\Pbb\left(
|x^\top Bx|
>
C\left(
\frac{\|B\|_F}{d}\sqrt u+
\frac{\|B\|_{\mathrm{op}}}{d}u
\right)
\right)
\le 2e^{-u}.
\]
Since \(\|B\|_F/d\le C\sqrt v\), \(\|B\|_{\mathrm{op}}/d\le r\), and
\(|\beta|\le\sqrt v\), this implies
\[
\Pbb\left(
|X_\ell|>C(\sqrt{vu}+ru)
\right)
\le 2e^{-u}.
\]
Lemma~\ref{lem:square-sum-two-level} applies because \(r^2\le Cv\). It gives
\[
\Pbb\left(
\sum_{\ell=1}^nX_\ell^2
>
C\left(nv+vs+r^2s^2\right)
\right)
\le Ce^{-s}.
\]
Replacing \(s\) in the equation above by \(s+c_0\), for a sufficiently large
absolute constant \(c_0\), and enlarging \(C\), changes the right-hand side to
\(e^{-s}\) and leaves the threshold in the same form. This proves the
claim.
\end{proof}

\begin{lemma}[Scalar weight regularity for one bilinear column]
\label{lem:scalar-weight-regularity}
Let \(\alpha_1,\dots,\alpha_m,\beta_1,\dots,\beta_m\) be independent standard
Gaussians, and define
\[
\gamma_r=\alpha_r\beta_r,
\qquad
\eta_r=\alpha_r\beta_r^2,
\qquad
\zeta_r=\alpha_r^2\beta_r.
\]
If \(m\ge C_*L^3\), then, with probability at least \(1-Ce^{-L}\),
\[
\left|
\frac1m\sum_{r=1}^m(\alpha_r^2\beta_r^2-1)
\right|
\le
C\left(\sqrt{\frac Lm}+\frac{L^2}{m}\right),
\]
\[
\sum_{r=1}^m\gamma_r^2\le Cm,
\qquad
\sum_{r=1}^m\eta_r^2\le Cm,
\qquad
\sum_{r=1}^m\zeta_r^2\le Cm,
\]
and
\[
\|\gamma\|_\infty^2L\le Cm.
\]
\end{lemma}

\begin{proof}
The variable \(\alpha^2\beta^2-1\) is a centered degree-four Gaussian polynomial
with bounded \(L_2\)-norm. By Theorem~\ref{thm:gaussian-hypercontractivity},
\[
\|\alpha^2\beta^2-1\|_{L_p}
\le
Cp^2,
\qquad p\ge2.
\]
Applying Theorem~\ref{thm:subweibull-bernstein-average} with \(\rho=2\) gives
\[
\left|
\frac1m\sum_{r=1}^m(\alpha_r^2\beta_r^2-1)
\right|
\le
C\left(\sqrt{\frac Lm}+\frac{L^2}{m}\right)
\]
with probability at least \(1-2e^{-L}\).

Next,
\[
\E\gamma_r^2=1,
\qquad
\E\eta_r^2=3,
\qquad
\E\zeta_r^2=3.
\]
The centered variables
\[
\gamma_r^2-1,
\qquad
\eta_r^2-3,
\qquad
\zeta_r^2-3
\]
are centered Gaussian polynomials of degrees \(4,6,6\), respectively, with bounded
\(L_2\)-norms. Applying Theorem~\ref{thm:subweibull-bernstein-average} with
\(\rho=2\) for \(\gamma_r^2-1\) and \(\rho=3\) for the degree-six terms gives
\[
\frac1m\sum_{r=1}^m\gamma_r^2\le C,
\qquad
\frac1m\sum_{r=1}^m\eta_r^2\le C,
\qquad
\frac1m\sum_{r=1}^m\zeta_r^2\le C
\]
with probability at least \(1-Ce^{-L}\), provided \(m\ge C_*L^3\).

It remains to control \(\|\gamma\|_\infty\). Since
\[
2|\alpha\beta|\le \alpha^2+\beta^2,
\]
and \(\alpha^2+\beta^2\sim\chi^2_2\), there is an absolute constant \(c>0\) such
that
\[
\Pbb(|\alpha\beta|>u)\le 2e^{-cu}.
\]
Thus
\[
\Pbb\left(\|\gamma\|_\infty^2L>Cm\right)
\le
2m\exp\left(-c\sqrt{\frac{Cm}{L}}\right).
\]
Since \(m\ge C_*L^3\), write \(m=L^3y\) with \(y\ge C_*\). Then
\[
\sqrt{\frac{m}{L}}=L\sqrt y,
\]
whereas
\[
L+\log(2m)+1
=
L+\log(2L^3y)+1
\le
C(L+\log y)
\le
CL\sqrt y.
\]
Choosing the absolute constants \(C_*\) and \(C\) sufficiently large therefore ensures
\[
c\sqrt{\frac{Cm}{L}}
\ge
L+\log(2m)+1.
\]
Therefore
\[
2m\exp\left(-c\sqrt{\frac{Cm}{L}}\right)
\le e^{-L}.
\]
A union bound over the above events completes the proof.
\end{proof}

\begin{lemma}[Weighted Gaussian product bounds]
\label{lem:weighted-gaussian-product}
Let \(U,V\in\R^{n\times m}\) have independent standard Gaussian entries, and let
\(\Gamma=\operatorname{diag}(\gamma_1,\dots,\gamma_m)\) be deterministic. Then,
for every \(s\ge1\), with probability at least \(1-4e^{-s}\),
\[
\|U\Gamma V^\top\|_{\mathrm{op}}
\le
C\left(
\|\gamma\|_2\sqrt{n+s}+\|\gamma\|_\infty(n+s)
\right),
\]
and
\[
\left|\operatorname{tr}(U\Gamma V^\top)\right|
\le
C\left(\sqrt n\,\|\gamma\|_2\sqrt s+\|\gamma\|_\infty s\right).
\]
Furthermore, if
\[
\|\gamma\|_2^2\le C_0m,
\qquad
\|\gamma\|_\infty^2s\le C_0m,
\qquad
n\ge s,
\]
then, with probability at least \(1-6e^{-s}\),
\[
\|U\Gamma V^\top\|_F^2\le Cn^2m.
\]
\end{lemma}

\begin{proof}
For fixed \(x,y\in\mathbb S^{n-1}\),
\[
x^\top U\Gamma V^\top y
=
\sum_{r=1}^m \gamma_r (u_r^\top x)(v_r^\top y),
\]
where \(u_r,v_r\) are the columns of \(U,V\). The variables
\((u_r^\top x)(v_r^\top y)\) are independent products of independent standard
Gaussians. Lemma~\ref{lem:weighted-product-gaussian} gives, for every \(t\ge0\),
\[
\Pbb\left(
|x^\top U\Gamma V^\top y|
>
C(\|\gamma\|_2\sqrt t+\|\gamma\|_\infty t)
\right)
\le
2e^{-t}.
\]
Take \(1/4\)-nets \(\mathcal N,\mathcal M\) of \(\mathbb S^{n-1}\) with cardinalities
at most \(9^n\). Set \(t=s+2n\log 9\), union bound over
\(\mathcal N\times\mathcal M\), and apply Theorem~\ref{thm:operator-norm-net}.
This proves the operator bound.

For the trace,
\[
\operatorname{tr}(U\Gamma V^\top)
=
\sum_{\ell=1}^n\sum_{r=1}^m\gamma_rU_{\ell r}V_{\ell r}.
\]
This is again a weighted sum of products of independent standard Gaussians, with
weights \(\gamma_r\) repeated \(n\) times. Lemma~\ref{lem:weighted-product-gaussian}
therefore gives the stated trace bound.

For the Frobenius bound, first note
\[
\|U\Gamma\|_F^2
=
\sum_{r=1}^m\gamma_r^2\|u_r\|_2^2.
\]
This is a weighted chi-square variable with mean \(n\|\gamma\|_2^2\le Cnm\). By
Theorem~\ref{thm:laurent-massart}, using
\(\|\gamma\|_\infty^2s\le C_0m\), we get
\[
\|U\Gamma\|_F^2\le Cnm
\]
with probability at least \(1-e^{-s}\).

Condition on \(U,\Gamma\), and set
\[
R\defeq \Gamma U^\top U\Gamma.
\]
Then
\[
\|U\Gamma V^\top\|_F^2
=
\operatorname{tr}(VRV^\top).
\]
Writing the rows of \(V\) as \(g_1,\dots,g_n\in\R^m\),
\[
\operatorname{tr}(VRV^\top)
=
\sum_{\ell=1}^n g_\ell^\top Rg_\ell.
\]
Diagonalize \(R\), and apply Theorem~\ref{thm:laurent-massart}. With conditional
probability at least \(1-e^{-s}\),
\[
\operatorname{tr}(VRV^\top)
\le
n\operatorname{tr}R
+2\sqrt{n\operatorname{tr}(R^2)s}
+2\|R\|_{\mathrm{op}}s.
\]
On \(\{\|U\Gamma\|_F^2\le Cnm\}\),
\[
\operatorname{tr}R=\|U\Gamma\|_F^2\le Cnm,
\]
\[
\|R\|_{\mathrm{op}}\le \operatorname{tr}R\le Cnm,
\]
and
\[
\operatorname{tr}(R^2)
\le
\|R\|_{\mathrm{op}}\operatorname{tr}R
\le
Cn^2m^2.
\]
Since \(n\ge s\), all three terms are bounded by \(Cn^2m\). Thus
\[
\|U\Gamma V^\top\|_F^2\le Cn^2m.
\]
Union bounding the events completes the proof.
\end{proof}

\begin{lemma}[Bilinear one-column matrix regularity]
\label{lem:bilinear-column-matrix-regularity}
Fix \(\delta\in(0,1)\), and define
\[
L\defeq \log\left(\frac{C_0F^2}{\delta}\right).
\]
Assume
\[
m\le \frac{d^2}{L},
\qquad
m\ge C_*L^3.
\]
Let \(\va_1,\vb_1,\dots,\va_m,\vb_m\in\R^d\) be i.i.d.\ standard Gaussian
vectors---the bilinear random features of \cref{lem:bilinear-sketched-k2}.
For \(q\in\Sph\), define
\[
A(q)
\defeq
\frac1{2m}\sum_{r=1}^m
(\va_r^\top q)(\vb_r^\top q)(\va_r\vb_r^\top+\vb_r\va_r^\top),
\]
and
\[
N(q)\defeq A(q)-qq^\top.
\]
Let \(\vk_1,\dots,\vk_F\) be i.i.d. uniform on \(\Sph\), independent of the
features. Then, with probability at least \(1-\delta/4\), simultaneously for all
\(i\in[F]\),
\[
\|N(\vk_i)\|_F^2\le C\frac{d^2}{m},
\]
\[
\|N(\vk_i)\|_{\mathrm{op}}
\le
C\left(\sqrt{\frac dm}+\frac{d}{\sqrt{mL}}\right),
\]
and
\[
|\operatorname{tr}N(\vk_i)|
\le
C\sqrt{\frac{dL}{m}}.
\]
\end{lemma}

\begin{proof}
It suffices to prove the result for a fixed deterministic \(q\in\Sph\) with
failure probability at most \(Ce^{-L}\). Conditional on the keys, the vectors
\(\vk_i\) are deterministic and independent of the feature randomness. A union
bound over \(i\in[F]\) gives failure probability at most
\[
CFe^{-L}=CF\frac{\delta}{C_0F^2}\le \delta/4
\]
for \(C_0\) sufficiently large.

Fix \(q\). By rotational invariance, assume \(q=e_1\). Write
\[
\va_r=(\alpha_r,u_r),
\qquad
\vb_r=(\beta_r,v_r),
\]
where \(\alpha_r,\beta_r\sim N(0,1)\), \(u_r,v_r\sim\mathcal N(0,I_{d-1})\), and
all variables are independent across \(r\). Set \(n=d-1\). The \(r\)-th summand is
\[
S_r
=
\frac12\alpha_r\beta_r(\va_r\vb_r^\top+\vb_r\va_r^\top),
\]
and \(\E S_r=e_1e_1^\top\). In block form,
\[
N(e_1)
=
\begin{pmatrix}
N_{11} & N_{1\perp}^\top\\
N_{1\perp} & N_{\perp\perp}
\end{pmatrix},
\]
where
\[
N_{11}=\frac1m\sum_{r=1}^m(\alpha_r^2\beta_r^2-1),
\]
\[
N_{1\perp}=\frac1{2m}\sum_{r=1}^m(\eta_ru_r+\zeta_rv_r),
\]
and, with \(U,V\in\R^{n\times m}\) having columns \(u_r,v_r\),
\(\Gamma=\operatorname{diag}(\gamma_1,\dots,\gamma_m)\),
\[
N_{\perp\perp}
=
\frac1{2m}(U\Gamma V^\top+V\Gamma U^\top),
\]
with
\[
\gamma_r=\alpha_r\beta_r,
\qquad
\eta_r=\alpha_r\beta_r^2,
\qquad
\zeta_r=\alpha_r^2\beta_r.
\]

Let \(\mathcal G\) be the good event from Lemma~\ref{lem:scalar-weight-regularity}.
On \(\mathcal G\),
\[
|N_{11}|\le C\left(\sqrt{\frac Lm}+\frac{L^2}{m}\right),
\]
\[
\|\gamma\|_2^2\le Cm,
\qquad
\|\gamma\|_\infty^2L\le Cm,
\qquad
\sum_r\eta_r^2\le Cm,
\qquad
\sum_r\zeta_r^2\le Cm.
\]
Also \(\Pbb(\mathcal G^c)\le Ce^{-L}\).

Conditional on the scalar weights,
\[
S_\eta\defeq\sum_r\eta_ru_r
\sim
\mathcal N\left(0,\left(\sum_r\eta_r^2\right)I_n\right),
\qquad
S_\zeta\defeq\sum_r\zeta_rv_r
\sim
\mathcal N\left(0,\left(\sum_r\zeta_r^2\right)I_n\right).
\]
On \(\mathcal G\), both covariance scalars are at most \(Cm\). Therefore
\(\|S_\eta\|_2^2/(Cm)\) and \(\|S_\zeta\|_2^2/(Cm)\) are stochastically dominated,
up to an absolute constant, by \(\chi_n^2\). By the chi-square consequence of
Theorem~\ref{thm:laurent-massart}, with conditional failure probability at most
\(2e^{-L}\),
\[
\|S_\eta\|_2
\le
C\sqrt m(\sqrt n+\sqrt L),
\qquad
\|S_\zeta\|_2
\le
C\sqrt m(\sqrt n+\sqrt L).
\]
Hence
\[
\|N_{1\perp}\|_2
\le
C\sqrt{\frac{n+L}{m}}.
\]
Since \(m\ge C_*L^3\) and \(m\le d^2/L\),
\[
d^2\ge mL\ge C_*L^4,
\]
so \(d\ge cL^2\) and \(n+L\le Cd\). Therefore
\[
\|N_{1\perp}\|_2\le C\sqrt{\frac dm}.
\]

By Lemma~\ref{lem:weighted-gaussian-product}, on \(\mathcal G\), with conditional
probability at least \(1-4e^{-L}\),
\[
\|U\Gamma V^\top\|_{\mathrm{op}}
\le
C\left(\sqrt m\sqrt{n+L}+\sqrt{\frac mL}(n+L)\right).
\]
The same bound holds for \(V\Gamma U^\top\). Thus
\[
\|N_{\perp\perp}\|_{\mathrm{op}}
\le
C\left(\sqrt{\frac{n+L}{m}}+\frac{n+L}{\sqrt{mL}}\right)
\le
C\left(\sqrt{\frac dm}+\frac d{\sqrt{mL}}\right).
\]
Combining block estimates gives
\[
\|N(e_1)\|_{\mathrm{op}}
\le
C\left(\sqrt{\frac dm}+\frac d{\sqrt{mL}}\right).
\]

For the Frobenius bound, Lemma~\ref{lem:weighted-gaussian-product} gives
\[
\|U\Gamma V^\top\|_F^2\le Cn^2m,
\]
because \(\|\gamma\|_2^2\le Cm\), \(\|\gamma\|_\infty^2L\le Cm\), and
\(n=d-1\ge L\). Thus
\[
\|N_{\perp\perp}\|_F^2\le C\frac{d^2}{m}.
\]
The scalar and vector blocks satisfy
\[
|N_{11}|^2\le C\frac{d^2}{m},
\qquad
\|N_{1\perp}\|_2^2\le C\frac dm\le C\frac{d^2}{m}.
\]
Therefore
\[
\|N(e_1)\|_F^2\le C\frac{d^2}{m}.
\]

Finally,
\[
\operatorname{tr}N(e_1)
=
N_{11}+\operatorname{tr}N_{\perp\perp}
=
N_{11}+\frac1m\operatorname{tr}(U\Gamma V^\top).
\]
By Lemma~\ref{lem:weighted-gaussian-product},
\[
\frac1m|\operatorname{tr}(U\Gamma V^\top)|
\le
C\left(\sqrt{\frac{dL}{m}}+\sqrt{\frac Lm}\right)
\le
C\sqrt{\frac{dL}{m}}.
\]
The scalar term \(N_{11}\) is absorbed by the same bound. Hence
\[
|\operatorname{tr}N(e_1)|
\le
C\sqrt{\frac{dL}{m}}.
\]
This proves the fixed-direction estimate. Rotational invariance and the initial
union bound over the keys complete the proof.
\end{proof}

\begin{lemma}[Bilinear one-column quadratic-form parameters]
\label{lem:bilinear-column-parameters}
Assume the setting of Lemma~\ref{lem:bilinear-column-matrix-regularity}. Suppose
also that
\[
m\le \frac{d^2}{L},
\qquad
L^3\le c_0\sigma^2m^2,
\]
where
\[
\sigma^2=\Theta\left(\frac1{d^2}+\frac1m\right).
\]
For \(q\in\Sph\), define
\[
\tau(q)\defeq \frac{\operatorname{tr}A(q)}{d},
\qquad
B(q)\defeq A(q)-\tau(q)I_d,
\qquad
\beta(q)\defeq \tau(q)-\frac1d.
\]
Also define
\[
v(q)\defeq \frac{2\|B(q)\|_F^2}{d(d+2)}+\beta(q)^2,
\qquad
r(q)\defeq \frac{\|B(q)\|_{\mathrm{op}}}{d}+|\beta(q)|.
\]
If \(c_0>0\) is sufficiently small and \(C_0\) sufficiently large, then with
probability at least \(1-\delta/4\), simultaneously for every \(i\in[F]\),
\[
v(\vk_i)\le C\sigma^2,
\qquad
r(\vk_i)^2L^2\le C\sigma^2L.
\]
\end{lemma}

\begin{proof}
Since \(m\le d^2/L\),
\[
\frac1{d^2}\le \frac1{mL}\le \frac1m.
\]
Thus
\[
\sigma^2=\Theta\left(\frac1{d^2}+\frac1m\right)\asymp \frac1m.
\]
The assumption \(L^3\le c_0\sigma^2m^2\) implies \(m\ge C_*L^3\), after choosing
\(c_0>0\) sufficiently small. Hence
Lemma~\ref{lem:bilinear-column-matrix-regularity} applies. Also
\[
d^2\ge mL\ge C_*L^4,
\]
so \(d\ge cL^2\).

Work on the event from Lemma~\ref{lem:bilinear-column-matrix-regularity}. Fix
\(i\), write \(q=\vk_i\), \(A=A(q)\), and \(N=N(q)=A-qq^\top\). Since
\[
A=qq^\top+N,
\]
we have
\[
\operatorname{tr}A=1+\operatorname{tr}N,
\qquad
\beta(q)=\frac{\operatorname{tr}N}{d}.
\]
Using the trace bound,
\[
|\beta(q)|\le C\sqrt{\frac{L}{dm}}.
\]
Since \(d\ge L\),
\[
\beta(q)^2\le C\frac{L}{dm}\le C\frac1m\le C\sigma^2.
\]

Next,
\[
B(q)
=
\left(qq^\top-\frac1dI_d\right)
+
\left(N-\frac{\operatorname{tr}N}{d}I_d\right).
\]
The trace-removal map is an orthogonal projection in Frobenius norm, and
\[
\left\|qq^\top-\frac1dI_d\right\|_F^2
=
1-\frac1d
\le1.
\]
Therefore
\[
\|B(q)\|_F^2
\le
C\left(1+\|N\|_F^2\right)
\le
C\left(1+\frac{d^2}{m}\right).
\]
Hence
\[
\frac{2\|B(q)\|_F^2}{d(d+2)}
\le
C\left(\frac1{d^2}+\frac1m\right)
\le
C\sigma^2.
\]
Together with \(\beta(q)^2\le C\sigma^2\), this proves \(v(q)\le C\sigma^2\).

For \(r(q)\), use
\[
\|B(q)\|_{\mathrm{op}}
\le
1+\|N\|_{\mathrm{op}}+\frac{|\operatorname{tr}N|}{d}.
\]
Therefore
\[
\frac{\|B(q)\|_{\mathrm{op}}}{d}
\le
C\left(
\frac1d+
\frac1{\sqrt{dm}}+
\frac1{\sqrt{mL}}+
\sqrt{\frac{L}{d^3m}}
\right).
\]
Since \(d\ge L^2\), all terms except \(1/d\) are dominated by
\(C/\sqrt{mL}\). Also
\[
|\beta(q)|
\le
C\sqrt{\frac{L}{dm}}
\le
C\frac1{\sqrt{mL}}.
\]
Thus
\[
r(q)
\le
C\left(\frac1d+\frac1{\sqrt{mL}}\right).
\]
Squaring and multiplying by \(L^2\),
\[
r(q)^2L^2
\le
C\left(\frac{L^2}{d^2}+\frac{L}{m}\right).
\]
Since \(m\le d^2/L\),
\[
\frac{L^2}{d^2}\le \frac{L}{m}.
\]
Therefore
\[
r(q)^2L^2\le C\frac{L}{m}\le C\sigma^2L.
\]
The estimates hold simultaneously for all \(q=\vk_i\).
\end{proof}

%% file: appendix/theory/03_margin_bounds/02a_welch_upper_bound_iso.tex
\subsubsection{A Welch/Frobenius Upper Bound in the Isotropic Key/Value Regime}
\phantomsection\label{app:theory:margin_bounds:rkrv:welch}
\label{app:welch-upper-bound-iso}

In the main text, \Cref{thm:bilinear-scaling} provides a lower bound on the minimum margin achieved by our construction in the isotropic key/value regime.
We now show that this lower bound is asymptotically tight, up to constants and logarithmic factors, by proving a matching \emph{upper bound} on the best possible margin of any admissible rank-limited kernel memory in this regime. The key idea is that the arbitrary-keys / isotropic-values analysis in \Cref{app:theory:margin_bounds:akrv} is controlled by the column energy
\[
E_{\mathrm{col}}(i) := \sum_{t\neq i} \hat{\matK}_{ti}^2.
\]
Thus, to upper bound the achievable margin, it suffices to show that for any rank-$m$ kernel, some column must have nontrivial off-diagonal squared mass.
We achieve this using a Welch / Frobenius argument: intuitively, low rank forces a non-negligible amount of kernel mass off the diagonal, which in turn yields an unavoidable cross-talk floor.

We begin by formalizing the class of kernels under consideration.

\begin{definition}[Rank-$m$ kernel]
\label{def:rank-m-kernel}
Let $\vk_{1},\dots,\vk_{F} \in \mathbb{R}^d$ be the stored keys. We say that
$\hat{\matK} \in \mathbb{R}^{F\times F}$ is a \emph{rank-$m$ kernel} if there exists a feature
map $\phi(\cdot)$ with $\phi(\vk_{i})\in\mathbb{R}^m$ such that
\[
\hat{\matK}_{ij} = \langle \phi(\vk_{i}), \phi(\vk_{j})\rangle
\qquad\text{for all } i,j\in[F].
\]
Equivalently, $\hat{\matK} \succeq 0$ and $\operatorname{rank}(\hat{\matK})\le m$.
\end{definition}

The following lemma is the kernel-side ingredient. It lower bounds the maximum column energy
of any PSD rank-$m$ kernel in terms of its diagonal.

\begin{lemma}[Welch/Frobenius lower bound for rank-$m$ kernels]
\label{lem:welch-frobenius-general}
Let $\hat{\matK} \in \mathbb{R}^{F\times F}$ be a rank-$m$ kernel in the sense of
Definition~\ref{def:rank-m-kernel}. Define
\[
E_{\mathrm{col}}(i) := \sum_{t\neq i}\hat{\matK}_{ti}^2,
\qquad
E_{\mathrm{col}} := \max_{i\in[F]} E_{\mathrm{col}}(i).
\]
Then
\[
\|\hat{\matK}\|_{F}^2 \ge \frac{\operatorname{tr}(\hat{\matK})^2}{m},
\]
and consequently
\[
E_{\mathrm{col}}
\;\ge\;
\frac{1}{F}\left(
\frac{\operatorname{tr}(\hat{\matK})^2}{m}
-
\sum_{i=1}^F \hat{\matK}_{ii}^2
\right).
\]
In particular, if
\[
\hat{\matK}_{ii}\in[1-\varepsilon_{\mathrm{diag}},\,1+\varepsilon_{\mathrm{diag}}]
\qquad\text{for all } i\in[F],
\]
then
\[
E_{\mathrm{col}}
\;\ge\;
\frac{F(1-\varepsilon_{\mathrm{diag}})^2}{m}
-
(1+\varepsilon_{\mathrm{diag}})^2.
\]
\end{lemma}

\begin{proof}
Since $\hat{\matK} \succeq 0$ and $\operatorname{rank}(\hat{\matK})\le m$, let
$\lambda_{1},\dots,\lambda_{r}$ be the nonzero eigenvalues of $\hat{\matK}$, where $r\le m$. Then
\[
\operatorname{tr}(\hat{\matK})=\sum_{a=1}^r \lambda_{a},
\qquad
\|\hat{\matK}\|_{F}^2=\sum_{a=1}^r \lambda_{a}^2.
\]
Applying Cauchy--Schwarz to $(\lambda_{1},\dots,\lambda_{r})$ gives
\[
\Big(\sum_{a=1}^r \lambda_{a}\Big)^2
\le r\sum_{a=1}^r \lambda_{a}^2
\le m\sum_{a=1}^r \lambda_{a}^2,
\]
hence
\[
\|\hat{\matK}\|_{F}^2 \ge \frac{\operatorname{tr}(\hat{\matK})^2}{m}.
\]

Next expand the Frobenius norm entrywise:
\[
\|\hat{\matK}\|_{F}^2
=
\sum_{i=1}^F \hat{\matK}_{ii}^2
+
\sum_{i=1}^F\sum_{t\neq i}\hat{\matK}_{ti}^2
=
\sum_{i=1}^F \hat{\matK}_{ii}^2
+
\sum_{i=1}^F E_{\mathrm{col}}(i).
\]
Therefore
\[
\sum_{i=1}^F E_{\mathrm{col}}(i)
\ge
\frac{\operatorname{tr}(\hat{\matK})^2}{m}
-
\sum_{i=1}^F \hat{\matK}_{ii}^2.
\]
Since the maximum is at least the average,
\[
E_{\mathrm{col}}
=
\max_{i} E_{\mathrm{col}}(i)
\ge
\frac{1}{F}\sum_{i=1}^F E_{\mathrm{col}}(i)
\ge
\frac{1}{F}\left(
\frac{\operatorname{tr}(\hat{\matK})^2}{m}
-
\sum_{i=1}^F \hat{\matK}_{ii}^2
\right).
\]

If moreover $\hat{\matK}_{ii}\in[1-\varepsilon_{\mathrm{diag}},1+\varepsilon_{\mathrm{diag}}]$ for all $i$,
then
\[
\operatorname{tr}(\hat{\matK})=\sum_{i=1}^F \hat{\matK}_{ii} \ge F(1-\varepsilon_{\mathrm{diag}})
\]
and
\[
\sum_{i=1}^F \hat{\matK}_{ii}^2 \le F(1+\varepsilon_{\mathrm{diag}})^2.
\]
Substituting these into the previous display yields
\[
E_{\mathrm{col}}
\ge
\frac{1}{F}\left(
\frac{F^2(1-\varepsilon_{\mathrm{diag}})^2}{m}
-
F(1+\varepsilon_{\mathrm{diag}})^2
\right)
=
\frac{F(1-\varepsilon_{\mathrm{diag}})^2}{m}
-
(1+\varepsilon_{\mathrm{diag}})^2.
\]
\end{proof}

Regarding the assumption
\[
\hat{\matK}_{ii}\in[1-\varepsilon_{\mathrm{diag}},1+\varepsilon_{\mathrm{diag}}],
\]
note that for the exact quadratic kernel on unit-norm isotropic keys, one has
\[
K_{2}(\vk_{i},\vk_{i})=\langle \vk_{i},\vk_{i}\rangle^2 = 1
\]
exactly.
For the sketched / bilinear random-feature kernel used in our construction, the diagonal is generally not exactly $1$, but it instead concentrates near $1$ with high probability by \Cref{lem:Kdiag-lower-bilinear}.

Combining the kernel-side lower bound with the isotropic-value competitor argument used in the
proof of \Cref{rem:welch-benchmark} yields the desired asymptotic optimality statement.

\begin{corollary}[Rank-$m$ kernels are asymptotically unbeatable in the isotropic key/value regime]
\label{cor:rank-m-unbeatable-iso}
Assume the isotropic key/value setting, and let $\hat{\matK}$ be any rank-$m$ kernel independent of
the values, with
\[
\hat{\matK}_{ii}\in[1-\varepsilon_{\mathrm{diag}},\,1+\varepsilon_{\mathrm{diag}}]
\qquad\text{for all } i\in[F].
\]
Then the isotropic-value competitor argument used in the proof of \Cref{rem:welch-benchmark} yields
\[
\gamma_{\min}
\;\le\;
(1+\varepsilon_{\mathrm{diag}})
-
\Omega\!\left(
\sqrt{\frac{E_{\mathrm{col}}\log F}{d}}
\right)
+
\mathcal{O}\!\left(
\sqrt{\frac{\log F}{d}}
\right)
+
\mathcal{O}\!\left(
\sqrt{\frac{E_{\mathrm{col}}\log(1/\delta)}{d}}
\right)
+
\mathcal{O}\!\left(
\sqrt{\frac{E_{\mathrm{col}}}{d}}
\right).
\]
Combining with Lemma~\ref{lem:welch-frobenius-general} yields
\[
\gamma_{\min}
\;\le\;
(1+\varepsilon_{\mathrm{diag}})
-
\Omega\!\left(
\sqrt{
\frac{\big(\frac{F(1-\varepsilon_{\mathrm{diag}})^2}{m}-(1+\varepsilon_{\mathrm{diag}})^2\big)\log F}{d}
}
\right)
+ \text{lower-order terms}.
\]
In particular, if $\varepsilon_{\mathrm{diag}}=o(1)$ and $F/m\to\infty$, then up to constants and
logarithmic factors,
\[
\gamma_{\min}
\;\le\;
1 - \Omega\!\left(\sqrt{\frac{(F/m)\log F}{d}}\right).
\]
Equivalently, no admissible rank-$m$ kernel can asymptotically beat the
$\sqrt{(F/m)\log F / d}$ margin floor in the isotropic key/value regime, up to constants and logs.
\end{corollary}

\begin{proof}
Apply the isotropic-value competitor argument used in the proof of \Cref{rem:welch-benchmark} to
the column $i^\star$ with $E_{\mathrm{col}}(i^\star)=E_{\mathrm{col}}$. This gives
\[
\gamma_{\min}
\le
\hat{\matK}_{i^\star i^\star}
-
\Omega\!\left(
\sqrt{\frac{E_{\mathrm{col}}\log F}{d}}
\right)
+
\mathcal{O}\!\left(
|\hat{\matK}_{i^\star i^\star}|\sqrt{\frac{\log F}{d}}
\right)
+
\mathcal{O}\!\left(
\sqrt{\frac{E_{\mathrm{col}}\log(1/\delta)}{d}}
\right)
+
\mathcal{O}\!\left(
\sqrt{\frac{E_{\mathrm{col}}}{d}}
\right),
\]
after absorbing harmless constant-factor differences into the big-$\mathcal{O}$/big-$\Omega$
notation.

Using $\hat{\matK}_{i^\star i^\star}\le 1+\varepsilon_{\mathrm{diag}}$ gives the first bound.
The second bound results from substituting the lower bound on $E_{\mathrm{col}}$ from~\Cref{lem:welch-frobenius-general}.

Finally, if $\varepsilon_{\mathrm{diag}}=o(1)$, then
\[
\frac{F(1-\varepsilon_{\mathrm{diag}})^2}{m}
-
(1+\varepsilon_{\mathrm{diag}})^2
=
\Theta\!\left(\frac{F}{m}\right)
\]
whenever $F/m\to\infty$, so the dominant negative term scales as
\[
\sqrt{\frac{(F/m)\log F}{d}},
\]
which proves the last claim.
\end{proof}

%% file: appendix/theory/04_transformers_perturbation.tex
\subsubsection{Noisy Margin}
\paragraph{Clean and Noisy Margins}
Given per-item queries $\vq_{1},\dots,\vq_{F}$, define 
\[
\gamma_{\min}(\vq)\ \defeq\ \min_{i\in[F]}\ \min_{j\neq i}\ \gamma_{ij}(\vq_{i}).
\]
Notably, let the noisy margin, for noisy queries $\tilde{\vk}$, be
\[
\gamma_{\min}(\tilde{\vk})\ \defeq\ \min_{i\in[F]}\ \min_{j\neq i}\ \gamma_{ij}(\tilde{\vk}_{i}).
\]
and the clean margin for noiseless queries $\vk$ be
\[
\gamma_{\min}(\vk)\ \defeq\ \min_{i\in[F]}\ \min_{j\neq i}\ \gamma_{ij}(\vk_{i}).
\]

\paragraph{Noisy queries.}
We query with $\widetilde{\vk}_{i}$ (instead of $\vk_{i}$), assuming
\begin{equation}
\norm{\widetilde{\vk}_{i}-\vk_i}_{2}\ \le\ \epsilon
\qquad \forall i\in[F].
\label{eq:noise-assump}
\end{equation}

\paragraph{Lipschitz stability in the query argument.}
Assume that for some $L_{k}$ and all queries of interest,
\begin{equation}
\big|K(\vk_{t},\vq)-K(\vk_{t},\vq')\big|
\ \le\
L_{k}\,\norm{\vq-\vq'}_{2}
\qquad \forall t\in[F].
\label{eq:Lipk}
\end{equation}
(Here $\vq,\vq'$ range over a set containing $\{\vk_{i}\}\cup\{\widetilde{\vk}_{i}\}$.)

\begin{theorem}[Noised margin bound with Lipschitz stability (isotropic values)]
\label{thm:noised-lip-isovalues}
Assume noisy queries (\cref{eq:noise-assump}), lipschitz stability on the kernel (\cref{eq:Lipk}), and isotropic values.
Fix $\delta\in(0,1)$ and set $L= \log(\tfrac{C_0F^2}{\delta})$.
Then with probability at least $1-\delta$, simultaneously for all $i\in[F]$ and all $j\neq i$,
\begin{equation}
\gamma_{ij}(\widetilde{\vk}_{i})
\ \ge\
\gamma_{ij}(\vk_{i})
\ -\ 2(1+\mu_{v})\,L_{k}\,\epsilon
\ -\ 2\sqrt{2}\,L_{k}\,\epsilon\,\sqrt{\frac{FL}{d}}.
\label{eq:thm1-main}
\end{equation}
Moreover, using the standard isotropic-values coherence bound for $\mu_{v}$ from \cref{lem:max-pairwise-ip-values},
the same event implies the simplified form
\begin{equation}
\gamma_{ij}(\widetilde{\vk}_{i})
\ \ge\
\gamma_{ij}(\vk_{i})
\ -\ 2\,L_{k}\,\epsilon
\ -\ C_{1}\,L_{k}\,\epsilon\,\sqrt{\frac{FL}{d}},
\label{eq:thm1-absorbed}
\end{equation}
for an absolute constant $C_{1}>0$. Further, in the common regime $FL \ge d$ (so that $\sqrt{FL/d}\ge 1$), we have
\begin{equation}
\gamma_{ij}(\widetilde{\vk}_{i})
\ \ge\
\gamma_{ij}(\vk_{i})
\ -\ C_{\mathrm{2}}\,L_{k}\,\epsilon\,\sqrt{\frac{FL}{d}},
\label{eq:thm1-pairwise-simplified}
\end{equation}
for an absolute constant $C_{\mathrm{2}}>0$.
\end{theorem}

\begin{proof}
Fix $i\in[F]$ and $j\neq i$, and define
\[
\Delta_{t,i}\ \defeq\ K(\vk_{t},\widetilde{\vk}_{i})-K(\vk_{t},\vk_{i}).
\]
Subtracting the noisy and clean margin gives the exact expansion
\begin{equation}
\gamma_{ij}(\widetilde{\vk}_{i})-\gamma_{ij}(\vk_{i})\ =\ \sum_{t=1}^F \Delta_{t,i}\,\ip{\vv_{i}-\vv_j}{\vv_t}.
\label{eq:margin-diff-expand}
\end{equation}

\paragraph{Bound $\Delta_{t,i}$ by Lipschitzness.}
By \cref{eq:Lipk} and \cref{eq:noise-assump}, $|\Delta_{t,i}|\le L_{k}\epsilon$ for all $t$.

\paragraph{Handle $t=i$ and $t=j$.}
Using $\big|\ip{\vv_{i}-\vv_j}{\vv_i}\big|\le 1+\mu_{v}$ and $\big|\ip{\vv_{i}-\vv_j}{\vv_j}\big|\le 1+\mu_{v}$,
\[
\Delta_{i,i}\ip{\vv_{i}-\vv_j}{\vv_i}+\Delta_{j,i}\ip{\vv_{i}-\vv_j}{\vv_j}
\ \ge\ -\,2(1+\mu_{v})\,L_{k}\epsilon.
\]

\paragraph{Concentrate the sum over $t\notin\{i,j\}$.}
Let
\[
S_{ij}\ \defeq\ \sum_{t\notin\{i,j\}} \Delta_{t,i}\,\ip{\vv_{i}-\vv_j}{\vv_t}.
\]
Condition on $(\vv_{i},\vv_{j})$. Then $(\vv_{t})_{t\notin\{i,j\}}$ are independent isotropic vectors, so by the
\emph{isotropic inner-product sub-Gaussianity lemma} (\cref{lem:iso-innerprod}),
each $\ip{\vv_{i}-\vv_j}{\vv_t}$ is mean-zero and sub-Gaussian with parameter
$\lesssim \|\vv_{i}-\vv_{j}\|_{2}/\sqrt{d}\le 2/\sqrt{d}$.
By standard sub-Gaussian closure under weighted sums (see, e.g., the proof of \cref{lem:kappa-iso-values}),
$S_{ij}$ is sub-Gaussian with variance proxy at most
\[
\frac{4}{d}\sum_{t\notin\{i,j\}}\Delta_{t,i}^2
\ \le\ \frac{4}{d}\cdot F\,(L_{k}\epsilon)^2.
\]
Hence the usual one-sided sub-Gaussian tail bound yields
\[
\Pbb\!\left(S_{ij}\le -u\right)
\ \le\ \exp\!\left(-\frac{u^2\,d}{8FL_{k}^2\epsilon^2}\right)
\qquad \forall u\ge 0.
\]
With $u\defeq 2\sqrt{2}\,L_{k}\epsilon\sqrt{\tfrac{FL}{d}}$, the right-hand side equals $e^{-L}=\delta/(C_0F^2)$.
A union bound over all ordered pairs $(i,j)$ gives that with probability at least $1-\delta/2$,
simultaneously for all $i$ and $j\neq i$,
\[
S_{ij}\ \ge\ -2\sqrt{2}\,L_{k}\epsilon\sqrt{\frac{FL}{d}}.
\]

\paragraph{Combine.}
Plugging the bounds for the $t=i,j$ terms and for $S_{ij}$ into \cref{eq:margin-diff-expand} yields \cref{eq:thm1-main}.

\paragraph{Absorb the $\mu_{v}$ term.}
By the standard coherence bound for isotropic values (\cref{lem:max-pairwise-ip-values}),
with probability at least $1-\delta/2$ we have $\mu_{v}\le C\sqrt{\tfrac{L}{d-1}}$ for an absolute $C$.
Since $\sqrt{\tfrac{L}{d-1}}\le \sqrt{\tfrac{FL}{d}}$, the term $2\mu_{v} L_{k}\epsilon$ can be absorbed into the
$\sqrt{\tfrac{FL}{d}}$ term, giving \cref{eq:thm1-absorbed} (after adjusting $C_{1}$).
Finally, in the regime $FL\ge d$, we have $2L_{k}\epsilon\le 2L_{k}\epsilon\sqrt{\tfrac{FL}{d}}$, so the
$2L_{k}\epsilon$ term can also be absorbed into the $\sqrt{\tfrac{FL}{d}}$ term, yielding \cref{eq:thm1-pairwise-simplified}.
\end{proof}

\begin{corollary}[Noised min-margin bound for bilinear-MLP stored-query retrieval (isotropic values)]
\label{cor:noised-bilinear}
Consider the isotropic-keys/isotropic-values bilinear-MLP setting of \cref{app:combined-regime3},
and assume the kernel satisfies \cref{eq:Lipk} with Lipschitz constant $L_{\mathrm{bil}}$.
Fix $\delta\in(0,1)$ and set $L=\log(\tfrac{C_{0} F^2}{\delta})$ for an absolute $C$.
Assume also the common regime $FL\ge d$.
Then with probability at least $1-\delta$,
\begin{equation}
\gamma_{\min}(\widetilde{\vk})
\ \ge\
\gamma_{\min}(\vk)
\ -\ C_{\mathrm{2}}\,L_{\mathrm{bil}}\,\epsilon\,\sqrt{\frac{FL}{d}}.
\label{eq:thm1-simplified}
\end{equation}
\end{corollary}

\begin{proof}
Apply \cref{thm:noised-lip-isovalues} in the regime $FL\ge d$ using \cref{eq:thm1-pairwise-simplified} with $L_{k}=L_{\mathrm{bil}}$.
Since the bound holds \emph{simultaneously} for all $(i,j)$, we may take the minimum over $i\in[F]$ and $j\neq i$:
\[
\gamma_{\min}(\widetilde{\vk})
=\min_{i}\min_{j\neq i}\gamma_{ij}(\widetilde{\vk}_{i})
\ \ge\
\min_{i}\min_{j\neq i}\gamma_{ij}(\vk_{i})\ -\ C_{\mathrm{2}}L_{\mathrm{bil}}\epsilon\sqrt{\frac{FL}{d}}
=\gamma_{\min}(\vk)\ -\ C_{\mathrm{2}}L_{\mathrm{bil}}\epsilon\sqrt{\frac{FL}{d}}.
\]
\end{proof}

\begin{corollary}[Condition on $\epsilon$ for $\gamma_{\min}(\widetilde{\vk})>0$]
\label{cor:epspos-new}
Under the assumptions of \cref{cor:noised-bilinear}, assume the clean margin is strictly positive:
\begin{equation}
\gamma_{\min}(\vk)\ >\ 0.
\label{eq:clean-positive-assump}
\end{equation}
A sufficient condition for $\gamma_{\min}(\widetilde{\vk})>0$ is
\begin{equation}
\epsilon
\ <\
\frac{\gamma_{\min}(\vk)}
{C_{\mathrm{2}}\,L_{\mathrm{bil}}}\,\sqrt{\frac{d}{FL}}.
\label{eq:epspos-new}
\end{equation}
\end{corollary}

\begin{proof}
By \cref{eq:thm1-simplified},
\[
\gamma_{\min}(\widetilde{\vk})
\ \ge\
\gamma_{\min}(\vk)\ -\ C_{\mathrm{2}}L_{\mathrm{bil}}\epsilon\sqrt{\frac{FL}{d}}.
\]
If \cref{eq:epspos-new} holds, the subtracted term is $<\gamma_{\min}(\vk)$, so the right-hand side is positive.
\end{proof}

\begin{corollary}[Noise-robust fact-storage capacity remains information-theoretically optimal]
\label{cor:capacity-new}
Fix $\delta\in(0,1)$ and set $L\defeq \log(\tfrac{C_{0} F^2}{\delta})$ for an absolute $C_{0}$.
Consider the iso--iso bilinear-MLP model with dimension $d$.
There exist universal constants $c,C,c_{0}>0$ such that if
\begin{equation}
d\ \ge\ C L,
\qquad
m\ \ge\ C L^3,
\qquad
F\ \le\ c\,\min\Big\{\frac{d^3}{L},\frac{md}{L}\Big\},
\label{eq:capacity-regime}
\end{equation}
then $\gamma_{\min}(\vk)\ge c_{0}$ with probability at least $1-\delta$ (by the combined margin bound of \cref{app:combined-regime3}).
If additionally
\begin{equation}
\epsilon
\ \le\
\frac{c_0}{2\,C_{\mathrm{2}}\,L_{\mathrm{bil}}}\,\sqrt{\frac{d}{F L}},
\label{eq:capacity-noise}
\end{equation}
then $\gamma_{\min}(\widetilde{\vk})\ge c_{0}/2$ with probability at least $1-\delta$.

In the sketch-limited regime $F\lesssim md/L$, the parameter count $W\asymp md$ thus satisfies
$W\asymp F L\asymp F\log(F/\delta)$, i.e.\ the (noise-robust) fact-storage scaling is information-theoretically optimal
up to constants (subject to \cref{eq:capacity-noise}).
\end{corollary}

\begin{proof}
Under \cref{eq:capacity-regime}, the combined margin bound of \cref{app:combined-regime3} yields
$\gamma_{\min}(\vk)\ge c_{0}$ with probability at least $1-\delta$.
Then apply \cref{eq:thm1-simplified}:
\[
\gamma_{\min}(\widetilde{\vk})
\ \ge\
c_{0}\ -\ C_{\mathrm{2}}L_{\mathrm{bil}}\epsilon\sqrt{\frac{F L}{d}}.
\]
If \cref{eq:capacity-noise} holds, the subtracted term is at most $c_{0}/2$, giving $\gamma_{\min}(\widetilde{\vk})\ge c_{0}/2$.

Finally, when $F\lesssim md/L$, we have $md\asymp F L\asymp F\log(F/\delta)$.
\end{proof}

%% file: appendix_short/related_work.tex
\section{Additional Related Work}
\label{app:related_work}

\paragraph{Empirical studies: probing and editing LLM knowledge.}
\citet{geva2021transformerfeedforwardlayerskeyvalue,geva2022transformerfeedforwardlayersbuild} observed that knowledge is often stored within MLPs via key--value mappings, motivating a line of work that attempts to reverse engineer the facts encoded in MLPs~\citep{dai2022knowledgeneuronspretrainedtransformers,nanda2023factfinding} and to edit them~\citep{dai2022knowledgeneuronspretrainedtransformers,meng2023locatingeditingfactualassociations,memit,model_edit_scaling,gu2024modeleditingharmsgeneral,fang2025alphaeditnullspaceconstrainedknowledge,sun2025mitigatingnegativeinterferencemultilingual}. These studies provide strong empirical evidence that MLPs act as a locus of factual storage in large language models.

\paragraph{Empirical studies: scaling factual knowledge.}
A related empirical line of work formalizes factual knowledge as associative recall over key--value stores and studies its scaling behavior~\citep{elhage2022toymodelssuperposition,allenzhu2024physicslanguagemodels33,zucchet2025languagemodelslearnfacts}. These works consistently find that trained models store facts at the asymptotically optimal rate implied by \Cref{thm: info_bounds_capacity-const}~\citep{allenzhu2024physicslanguagemodels33,zucchet2025languagemodelslearnfacts,morris2025languagemodelsmemorize}, which motivates the search for explicit constructions with comparable parameter efficiency.